\documentclass[final,12pt]{alt2024} 

\usepackage{amsmath,amsfonts,bm}
\usepackage{hyperref}
\usepackage{url}
\usepackage[T1]{fontenc}
\usepackage{algorithm, algpseudocode}
\usepackage{graphicx}
\usepackage{enumitem}
\usepackage{wrapfig}
\usepackage{xfrac}
\usepackage{longtable}
\usepackage{array}

\def\eqref#1{equation~\ref{#1}}

\newcommand{\paren}[1]{\left(#1\right)}

\newcommand{\norm}[1]{\left\|#1\right\|}

\newcommand{\inner}[2]{\left\langle#1, #2\right\rangle}

\renewcommand{\vec}[1]{\texttt{vec}\paren{#1}}

\newcommand{\R}{\mathbb{R}}
\newcommand{\N}{\mathbb{N}}

\def\calL{{\mathcal{L}}}

\def\bfA{{\mathbf{A}}}

\def\bfF{{\mathbf{F}}}

\def\bfI{{\mathbf{I}}}

\def\bfO{{\mathbf{O}}}

\def\bfU{{\mathbf{U}}}
\def\bfV{{\mathbf{V}}}
\def\bfW{{\mathbf{W}}}
\def\bfX{{\mathbf{X}}}
\def\bfY{{\mathbf{Y}}}

\def\bfa{{\mathbf{a}}}
\def\bfb{{\mathbf{b}}}

\def\bfs{{\mathbf{s}}}

\def\bfu{{\mathbf{u}}}
\def\bfv{{\mathbf{v}}}
\def\bfw{{\mathbf{w}}}
\def\bfx{{\mathbf{x}}}
\def\bfy{{\mathbf{y}}}
\def\bfz{{\mathbf{z}}}

\DeclareMathOperator*{\argmin}{arg\,min}
\newcommand{\dx}[1]{\mathcal{D}_{\bfx,#1}}
\newcommand{\du}[1]{\mathcal{D}_{\bfu,#1}}

\definecolor{myblue2}{HTML}{4682B4}

\newtheorem{defin}{Definition}

\newtheorem{asump}{Assumption}

\newtheorem{cond}{Condition}

\title[Acceleration of Neural Network Training]{Provable Accelerated Convergence of Nesterov's Momentum \\ for Deep ReLU Neural Networks}
\usepackage{times}

\altauthor{%
 \Name{Fangshuo Liao} \Email{Fangshuo.Liao@rice.edu}\\
 \addr Rice University
 \AND
 \Name{Anastasios Kyrillidis} \Email{anastasios@rice.edu}\\
 \addr Rice University
}

\begin{document}

\maketitle

\begin{abstract}%
  Current state-of-the-art analyses on the convergence of gradient descent for training neural networks focus on characterizing properties of the loss landscape, such as the Polyak-\L ojaciewicz (PL) condition and the restricted strong convexity. 
  While gradient descent converges linearly under such conditions, it remains an open question whether Nesterov's momentum enjoys accelerated convergence under similar settings and assumptions. 
  In this work, we consider a new class of objective functions, where only a subset of the parameters satisfies strong convexity, and show Nesterov's momentum achieves acceleration in theory for this objective class. 
  We provide two realizations of the problem class, one of which is deep ReLU networks, which constitutes this work as the first that proves an accelerated convergence rate for non-trivial neural network architectures.
\end{abstract}


\section{Introduction}

Training neural networks with gradient-based methods has shown surprising empirical success \citep{lecun1998gradient,lecun2015deep,zhang2017understanding,Goodfellow-et-al-2016}; yet, it has been a mystery why such a simple algorithm can consistently find a good minimum for these highly non-convex objectives \citep{zhang2018learning, li2020learning, yun2018small,auer1995exponentially,safran2018spurious}. As a consequence of this mysterious phenomenon, an equally, if not more, intriguing question is \textit{why momentum methods, which are designed originally for accelerating the minimization of convex objectives, can achieve faster convergence speed when applied to complicated non-convex objectives, as that of neural network training.}

The advances of the Neural Tangent Kernel (NTK) \citep{jacot2020neural} promoted the theoretical understanding of neural network training. The use of NTK shows that when the width of the neural network approaches infinity, the training process can be treated as a kernel machine. 
Inspired by the NTK analysis, a large body of work has focused on showing the convergence of gradient descent for various neural network architectures under finite over-parameterization requirements \citep{du2018gradient,du2019gradient,zhu2019aconvergence,zou2019an,zhao2019learning,awasthi2021a,ling2023global,allenzhu2019convergence, song2020quadratic,su2019learning}. Yet, this line of analysis is hard to extend to deeper and more complicated architectures and requires a significantly larger over-parameterization than what is used in practice.

To resolve this limitation, later work started to build connections between the understanding of neural network training and the widely studied optimization theory. In particular, a recent line of work characterizes the loss landscape of neural networks using the local Polyak-\L ojaciewicz (PL) condition \citep{song2021subquadratic, liu2020Loss, nguyen2021ontheproof, ling2023global}. 
Based on the well-established theory of how gradient descent converges under the PL condition \citep{karimi2020linear}, this line of work decouples the neural network structure from the dynamics of the loss function along the optimization trajectory. 
This way, these works could perform a more fine-grained analysis of the relationship between regulatory conditions, such as the PL condition, and the neural network structure. 
Such analysis not only resulted in further relaxed over-parameterization requirements \citep{song2021subquadratic,nguyen2021ontheproof,liu2020Loss} but was also shown to be easily extended to deep architectures \citep{ling2023global}, suggesting that it is more suitable in practice.

In contrast to the fast-growing research devoted to (stochastic) gradient descent, there is limited theoretical work on the convergence of momentum methods in deep learning. 
The acceleration of both the Heavy Ball method and Nesterov's momentum is shown only for shallow ReLU networks \citep{wang2021modular,liu2022provable} and deep linear networks \citep{wang2021modular,liu2022convergence}. It remains an open question to prove the acceleration for neural network training in a scenario closer to what is used in practice in terms of both the over-parameterization requirement and the depth and architecture of the neural network. As a result, we are interested in finding a regulatory condition for neural networks that enables the accelerated convergence for momentum methods. 

Showing acceleration under only the PL condition has been a long-standing difficulty. 
For the Heavy Ball method, \cite{danilova2018nonmonotone} established a linear convergence rate under the PL condition, but no acceleration is shown without assuming strong convexity. 
\cite{wang2022provable} proved an accelerated convergence rate; yet, the authors assume the $\lambda^{\star}$-\textsc{average out} condition, which cannot be easily justified for complicated objectives like neural networks. 
To our knowledge, the convergence proof for Nesterov's momentum under the PL condition in non-convex settings is currently missing. 
In the continuous limit, acceleration is proved in a limited scenario \citep{apodopoulos2022convergence}, which does not easily extend to the discrete case \citep{shi2022understanding}. 
Finally, \cite{yue2022lower} shows that gradient descent already achieves an optimal convergence rate for functions satisfying smoothness and the PL condition. 
This suggests that we need to leverage properties beyond the PL condition to prove the acceleration of momentum methods in a broader class of neural networks.

Based on prior work \citep{liu2020Loss} that shows over-parameterized systems are essentially non-convex in any neighborhood of the global minimum, we aim at developing a relaxation to the (strong) convexity in the non-convex setting that enables the momentum methods to achieve acceleration. 
In particular, we consider the minimization of a new class of objectives:
\begin{equation}
    \label{eq:partition_model}
    \vspace{-0.1cm}
    \min_{\bfx\in\R^{d_1}, \bfu\in\R^{d_2}}f(\bfx, \bfu),
    \vspace{-0.15cm}
\end{equation}
where $f$ satisfies the strong convexity with respect to $\bfx$, among other assumptions (c.f., Assumption \ref{asump:strong_cvx}-\ref{asump:universal_opt}). 
Intuitively, our construction assumes that the parameter space can be partitioned into two sets, and only one of the two sets enjoys rich properties, such as strong convexity. 
In this paper, we focus on Nesterov's momentum since it has been shown in a recent work that the Heavy Ball method cannot achieve acceleration even for smooth and strongly convex functions \citep{goujaud2023provable}. Indeed, in previous empirical works, Nesterov's momentum is not only shown to achieve acceleration in neural network training \citep{sutskever2013}, but also demonstrate better performance under large-scale testing than the Heavy Ball method \citep{dahl2023benchmarking}.

\textbf{Our contribution.} Our paper starts with an investigation of the properties of the problem class in (\ref{eq:partition_model}) that satisfies Assumption \ref{asump:strong_cvx}-\ref{asump:universal_opt}. In particular, we show that this set of assumptions is stronger than the PL condition but weaker than strong convexity, and, as a consequence, gradient descent converges linearly with rate $1 - \Theta\paren{\sfrac{1}{\kappa}}$ under these assumptions. Next, we prove that Nesterov's momentum enjoys an accelerated linear convergence with a convergence rate $1 - \Theta\paren{\sfrac{1}{\sqrt{\kappa}}}$. Under Assumption \ref{asump:strong_cvx}-\ref{asump:universal_opt}, our result holds even when $f$ is non-convex and non-smooth:
\begin{theorem}[Informal statement of Theorem \ref{theo:nesterov_conv}]
    Let $f\paren{\bfx,\bfu}:\R^{d_1}\times \R^{d_2}\rightarrow\R$ be $L_1$-smooth and $\mu$-strongly convex with respect to $\bfx$ for all $\bfu\in\R^{d_2}$, and let $\kappa = \sfrac{L_1}{\mu}$. If $f\paren{\bfx,\bfu}$ also satisfies Assumption \ref{asump:g_smooth}-\ref{asump:universal_opt} with sufficiently small $G_1, G_2$ and sufficiently large $R_{\bfx}, R_{\bfu}$, then the sequence $\left\{\paren{\bfx_k,\bfu_k}\right\}_{k=0}^\infty$ generated by Nesterov's momentum satisfies:
    \vspace{-0.2cm}
    \[
        f\paren{\bfx_k,\bfu_k}\leq 2\paren{1 - \Theta\paren{\tfrac{1}{\sqrt{\kappa}}}}^k\paren{f(\bfx_0,\bfu_0)-f^*}.
    \]
    \vspace{-0.7cm}
\end{theorem}
Next, we provide two realizations of our problem class. In Section \ref{sec:addition_model}, we first consider fitting an additive model under the MSE loss. We prove the acceleration of Nesterov's momentum as long as the non-convex component of the additive model is small enough to guarantee the Lipschitz-type assumptions. Next, we turn to deep ReLU network training in Section \ref{sec:dnn}. We show that when the width of the neural network trained with $n$ samples is $\Omega\paren{n^4}$, under proper initialization, Nesterov's momentum converges to zero training loss with rate $1 - \Theta\paren{\sfrac{1}{\sqrt{\kappa}}}$. To the best of our knowledge, \textit{this is the first result that establishes accelerated convergence of deep ReLU networks}:
\begin{theorem}[Informal statement of Theorem \ref{theo:nn_nesterov_conv}]
    Given a dataset with $n$ samples and $d_0$ features, we let $\bfF$ be a deep ReLU neural network with width $\Omega\paren{n^4d_0^2}$ and let $\mathcal{L}_k\in\R$ be the MSE loss value at iteration $k$ generated by training $\bfF$ with Nesterov's momentum. Then, for all $k\geq 0$, we have that:
    \vspace{-0.2cm}
    \[
        \mathcal{L}_k \leq 2\paren{1 - \Theta\paren{\tfrac{1}{\sqrt{\kappa}}}}^k\mathcal{L}_0.
    \]
    \vspace{-0.7cm}
\end{theorem}

\subsection{Related Works}
\textbf{Convergence in neural network training.} 
The NTK-based analysis builds upon the idea that when the width approaches infinity, training neural networks behaves like training a kernel machine. Various techniques are developed to control the error when the width becomes finite. In particular, \citep{du2018gradient} tracks the change of activation patterns in ReLU-based neural networks and often requires a large over-parameterization. Later works improve the over-parameterization requirement by leveraging matrix concentration inequalities \citep{song2020quadratic}, performing fine-grained analysis on the change of Jacobians \citep{oymak2019moderate}, analyzing the functional approximation property \citep{su2019learning}, and building their analysis upon the separability assumption of the data in the reproducing Hilbert space of the neural network \citep{ji2020polylogarithmic}. Going beyond two-layer neural networks, \citet{allenzhu2019convergence, du2019gradient} analyze the convergence of gradient descent on deep neural networks under a large over-parameterization. 
In the meantime, the analysis was also extended to other training algorithms and settings, such as stochastic gradient descent \citep{oymak2019moderate,ji2020polylogarithmic,xu2021onepass,zou2018stochastic}, drop-out \citep{liao2022on,mianjy2020on}, federated training \citep{huang2021flntk}, and adversarial training \citep{li2022federated}.

A noticeable line of work focuses on establishing that the PL condition is satisfied by neural networks, where the coefficient of the PL condition is based on the eigenvalue of the NTK matrix. \cite{nguyen2021ontheproof} shows the PL condition is satisfied by deep ReLU neural networks by considering the dominance of the gradient with respect to the weight in the last layer. \cite{liu2020Loss} proves the PL condition by upper bounding the Hessian for deep neural networks with smooth activation functions. \cite{song2021subquadratic} further reduces the over-parameterization while maintaining the PL condition via the expansion of the activation function with the Hermite polynomials. Lastly, \cite{banerjee2023restricted} establishes the restricted strong convexity of neural networks within a sequence of ball-shaped regions centered around the weights per iteration; yet, the coefficient of the strong convexity is not explicitly characterized in theory.

It should be noted that the above work relies on the over-parameterization of the neural network, which, in many cases, leads the training dynamic to be trapped in the so-called kernel regime \citep{woodworth2020kernel,yehudai2022power,yang2021tensor}. While crucial to guarantee a favorable loss landscape \citep{safran2018spurious}, it is also shown that even mild over-parameterization leads to an exponentially slower convergence rate \citep{xu2023overparameterization} and cannot explain the behavior of learning a single neuron \citep{yehudai2022power}. However, the above work focuses solely on the training with gradient descent. While our analysis is based on the over-parameterization assumption, it is, to the best of our knowledge, the first to show the convergence of Nesterov's momentum on deep neural networks and opens up the possibility of studying Nesterov's momentum on neural networks in a more realistic scenario.

\noindent\textbf{Convergence of Nesterov's Momentum.} 
The original proof of Nesterov's momentum \citep{nesterov2018lectures} builds upon the idea of estimating sequences for both convex smooth objectives and strongly convex smooth objectives. 
Later work in \citep{bansal2019potential} provides an alternate proof within the same setting by constructing a Lyapunov function. 
In the non-convex setting, a large body of works focuses on variants of Nesterov's momentum that lead to a guaranteed convergence by employing techniques such as negative curvature exploitation \citep{carmon2017accelerated}, cubic regularization \citep{carmon2020first}, and restarting schemes \citep{li2022restarted}. For neural networks, \citep{liu2022provable,liu2022convergence} are the only works that study the convergence of Nesterov's momentum. However, considering the over-parameterization requirement, the objective is similar to a quadratic function. Deviating from Nesterov's momentum, \cite{wang2021modular} studies the convergence of the Heavy-ball method under similar over-parameterization requirements. 
A recent work \citep{wu2023meanfield} proves the convergence of the Heavy-ball method under the mean-field limit; such a limit is not the focus of our study in this paper. Lastly, \cite{jelassi2022understanding} shows that momentum-based methods improve the generalization ability of neural networks. However, there is no explicit convergence guarantee for the training loss.
\section{Problem Setup and Assumptions}
\label{sec:preliminary}
\noindent\textbf{Notations} Standard lower-case letters (e.g. $a$) denote scalars, bold lower-case letters (e.g. $\bfa$) denote vectors, and bold upper-case letters (e.g. $\bfA$) denote matrices. For a vector $\bfa$, we use $a_i$ to denote its $i$-th entry and $\norm{\bfa}_2$ its $\ell_2$-norm. For a matrix $\bfA$, we use $a_{ij}$ to denote its $(i,j)$-th entry and $\norm{\bfA}_F$ its Frobenius norm. we use $(\bfa_1, \bfa_2)$ to denote the concatenation of $\bfa_1, \bfa_2$. For a matrix $\bfA$ with columns $\bfa_1,\dots, \bfa_n$, we use $\texttt{V}(\bfA) = (\bfa_1,\dots,\bfa_n)$ to denote the vectorized form of $\bfA$.

Optimization literature often focuses on the constraint-free minimization of a function $\hat{f}:\R^d\rightarrow \R$. In this scenario, Nesterov's momentum with step size $\eta$ and momentum parameter $\beta$ for minimizing $\hat{f}\paren{\bfw}$ bears the form, as in \citet{bansal2019potential} and (2.2.22) of \cite{nesterov2018lectures}\footnote{Despite a different choice of step size and momentum parameter.}
\begin{equation}
    \label{eq:plain_nesterov}
    \bfw_{k+1} = \bar{\bfw}_k - \eta\nabla\hat{f}\paren{\bar{\bfw}};\quad \bar{\bfw}_{k+1} = \bfw_{k+1} + \beta\paren{\bfw_{k+1} - \bfw_k}
\end{equation}
In this paper, we reformulate this problem using the following definition.
\vspace{-0.2cm}
\begin{defin}
     A function $f:\R^{d_1}\times \R^{d_2}\rightarrow \R$ is called a partitioned equivalence of $\hat{f}:\R^d\rightarrow \R$, if $i)$ $d_1 + d_2 = d$, and $ii)$ there exists a permutation function $\pi: \mathbb{R}^{d} \rightarrow \mathbb{R}^d$ over the parameters of $\hat{f}$, such that 
    $\hat{f}(\bfw) = f(\bfx, \bfu)$ if and only if $\pi(\bfw) = (\bfx, \bfu)$. We say that $(\bfx, \bfu)$ is a partition of $\bfw$.
\end{defin}
\vspace{-0.2cm}
Despite the difference in the representation of their parameters, $\hat{f}$ and $f$ share the same properties, and any algorithm for $\hat{f}$ would produce the same result for $f$. Therefore, we turn our focus from the minimization problem of $\hat{f}$ to the minimization problem in (\ref{eq:partition_model}). 
We should clarify that when we study the property of $f$ as an attempt to study the property of $\hat{f}$, \textit{we only need to assume the existence of such a partitioned equivalence}, instead of requiring an efficient algorithm to identify this equivalence explicitly.
We further assume that $f$ is a composition of a loss function $g:\R^{\hat{d}}\rightarrow \R$ and a possibly non-smooth and non-convex model function $h:\R^{d_1}\times \R^{d_2}\rightarrow \R^{\hat{d}}$, for some dimension $\hat{d} \in \mathbb{Z}_+$; i.e., $f(\bfx, \bfu) = g(h(\bfx, \bfu))$. With this construction of functional composition, we can assume only a partial smoothness on $f$ together with the smoothness of $g$, instead of the full smoothness property of $f$.
We obey the following notation with respect to gradients of $f$:
\begin{equation}
    \begin{gathered}
    \nabla_1f(\bfx, \bfu) = \tfrac{\partial f(\bfx, \bfu)}{\partial\bfx};\; \nabla_2f(\bfx, \bfu) = \tfrac{\partial f(\bfx, \bfu)}{\partial\bfu};\;\nabla f(\bfx, \bfu)  = \paren{\nabla_1f(\bfx, \bfu), \nabla_2f(\bfx, \bfu)}.
    \end{gathered}
\end{equation}
We will consider Nesterov's momentum with constant step size $\eta$ and momentum parameter $\beta$:
\begin{align}
    \label{eq:nesterov}
    \begin{split}
    \paren{\bfx_{k+1}, \bfu_{k+1}} &= \paren{\bfy_k, \bfv_k} - \eta\nabla f(\bfy_k, \bfv_k)\\
    \paren{\bfy_{k+1}, \bfv_{k+1}} &= \paren{\bfx_{k+1}, \bfu_{k+1}} + \beta\paren{\paren{\bfx_{k+1}, \bfu_{k+1}} - \paren{\bfx_k, \bfu_k}}    
    \end{split}
\end{align}
with $\bfy_0 = \bfx_0$ and $\bfv_0 = \bfu_0$. The algorithm formulation in (\ref{eq:nesterov}) is mathematically equivalent to (\ref{eq:plain_nesterov}) for optimizing $\hat{f}(\bfx)$. Therefore, the execution of (\ref{eq:nesterov}) is completely agnostic to the parameter partition. To state our assumptions, let $\mathcal{B}^{(1)}_{R_{\bfx}}$ and $\mathcal{B}^{(2)}_{R_{\bfu}}$ denote the balls centered as $\bfx_0$ and $\bfu_0$:
\vspace{-0.2cm}
\[
    \mathcal{B}^{(1)}_{R_{\bfx}} = \{\bfx\in\R^{d_1}:\norm{\bfx-\bfx_0}_2\leq R_{\bfx}\};\quad \mathcal{B}^{(2)}_{R_{\bfu}} = \{\bfu\in\R^{d_2}:\norm{\bfu-\bfu_0}_2\leq R_{\bfu}\}.
\]
Next, we state the assumptions on the general class of objectives we consider.
\begin{asump}
    \label{asump:strong_cvx}
    $f$ is $\mu$-strongly convex with $\mu > 0$ with respect to the first part of its parameters:
    \vspace{-0.2cm}
    \[
        f(\bfy, \bfu)\geq f(\bfx, \bfu) + \inner{\nabla_1f(\bfx, \bfu)}{\bfy - \bfx} + \frac{\mu}{2}\norm{\bfy - \bfx}_2^2,\quad\forall \bfx,\bfy\in\R^{d_1};\; \bfu\in \mathcal{B}^{(2)}_{R_{\bfu}}.
    \]
    \vspace{-0.7cm}
\end{asump}
\vspace{-0.2cm}
\begin{asump}
    \label{asump:f_smooth}
    $f$ is $L_1$-smooth with respect to the first part of its parameters:
    \vspace{-0.2cm}
    \[
        f(\bfy, \bfu) \leq f(\bfx, \bfu) + \inner{\nabla_1f(\bfx,\bfu)}{\bfy - \bfx} + \frac{L_1}{2}\norm{\bfy - \bfx}_2^2,\quad\forall \bfx,\bfy\in\R^{d_1};\; \bfu\in\mathcal{B}^{(2)}_{R_{\bfu}}.
    \]
    \vspace{-0.5cm}
\end{asump}
\vspace{-0.2cm}
Based on Assumption \ref{asump:strong_cvx} and \ref{asump:f_smooth}, we define the condition number of $f$.
\begin{defin}
    (Condition Number)
    The condition number $\kappa$ of $f$ is given by $\kappa = \sfrac{L_1}{\mu}$.
\end{defin}

\begin{asump}
    \label{asump:g_smooth}
    $g$ satisfies $\min_{\bfs\in\R^{\hat{d}}}g(\bfs) = \min_{\bfx\in\R^{d_1},\bfu\in\R^{d_2}}f(\bfx,\bfu)$, and is $L_2$-smooth:
    \vspace{-0.2cm}
    \[
        g(\bfs_1) \leq g(\bfs_2) + \inner{\nabla g(\bfs_1)}{\bfs_2 - \bfs_1} + \frac{L_2}{2}\norm{\bfs_2 - \bfs_1}_2^2,\quad\forall \bfs_1, \bfs_2\in\R^{\hat{d}}.
    \]
    \vspace{-0.5cm}
\end{asump}
Assumptions \ref{asump:strong_cvx} and \ref{asump:f_smooth} are relaxed versions of the smoothness and strong convexity. Instead of assuming that the objective is smooth and strongly convex over all parameters, we only assume such property to hold with respect to a subset of the parameters while the rest lie near initialization. Assumption \ref{asump:g_smooth} is standard in prior literature \citep{liu2020Loss, song2021subquadratic} of neural network training, and holds for loss functions such as the MSE loss and the logistic loss.
\begin{asump}
    \label{asump:h_lip}
    $h$ satisfies $G_1$-Lipschitzness with respect to the second part of its parameters:
    \[
        \norm{h(\bfx, \bfu) - h(\bfx, \bfv)}_2\leq G_1\norm{\bfu - \bfv}_2,\quad\forall \bfx\in\mathcal{B}^{(1)}_{R_{\bfx}};\;\bfu,\bfv\in\mathcal{B}^{(2)}_{R_{\bfu}}.
    \]
\end{asump}

\begin{asump}
    \label{asump:grad_lip}
    The gradient of $f$ with respect to the first part of its parameter, $\nabla_1f(\bfx, \bfu)$, satisfies $G_2$-Lipschitzness with respect to the second part of its parameters:
    \[
        \norm{\nabla_1f(\bfx, \bfu) - \nabla_1f(\bfx, \bfv)}_2 \leq G_2\norm{\bfu - \bfv}_2,\quad\forall \bfx\in\mathcal{B}^{(1)}_{R_{\bfx}};\;\bfu,\bfv\in\mathcal{B}^{(2)}_{R_{\bfu}}.
    \]
\end{asump}

\begin{asump}
    \label{asump:universal_opt}
    Minimum values of $f$ restricted to the optimization over $\bfx$ equal the global minimum value:
    \vspace{-0.2cm}
    \[
        \min_{\bfx\in\R^{d_1}}f(\bfx, \bfu) = f^\star:=\min_{\bfx\in\R^{d_1},\bfu\in\R^{d_2}}f(\bfx,\bfu);\quad \forall \bfu\in\mathcal{B}^{(2)}_{R_{\bfu}}.
    \]
    \vspace{-0.5cm}
\end{asump}
Since we do not assume $f$ to be convex or smooth with respect to $\bfu$, we cannot guarantee that the updates in (\ref{eq:nesterov}) on $\bfu$ will make positive progress towards finding the global minimum. Nevertheless, the updates on $\bfu$ are unavoidable since the execution of (\ref{eq:nesterov}) is agnostic to the parameter partition.
Therefore, we treat the change in the second part of the parameters as errors induced by the updates. 
Assumptions \ref{asump:h_lip} and \ref{asump:grad_lip} are made to control the effect on the change of the model output $h(\bfx, \bfu)$ and the gradient with respect to $\bfx$ caused by the change of $\bfu$. Moreover, without Assumption \ref{asump:universal_opt}, it is possible that the change of $\bfu$ will lead the optimization trajectory to some local minimum of $\bfu$, such that the global minimum value cannot be achieved even when $\bfx$ is fully optimized. 
We show that Assumptions \ref{asump:strong_cvx}-\ref{asump:universal_opt} are satisfied by a smooth and strongly convex function:
\begin{theorem}
    \label{theo:strong_cvx_smooth_suffice}
    Let $\tilde{f}$ be $\tilde{\mu}$-strongly convex and $\tilde{L}$-smooth. Then $\tilde{f}$ satisfies Assumptions \ref{asump:strong_cvx}-\ref{asump:universal_opt} with:
    \vspace{-0.2cm}
    \[
        R_{\bfx} = R_{\bfu} = \infty;\;\mu = \tilde{\mu};\;L_1 = L_2 = \tilde{L};\; G_1 = G_2 = 0.
    \]
\end{theorem}
\vspace{-0.2cm}
Theorem \ref{theo:strong_cvx_smooth_suffice} shows that the combination of Assumptions \ref{asump:strong_cvx}-\ref{asump:universal_opt} is no stronger than the assumption that the objective is smooth and strongly convex. Therefore, the minimization of the class of functions satisfying Assumptions \ref{asump:strong_cvx}-\ref{asump:universal_opt} does not have a better lower complexity bound than the class of smooth and strong convex functions. That is, the best convergence rate we can achieve is $1 - \Theta\paren{\sfrac{1}{\sqrt{\kappa}}}$.

\section{Accelerated Convergence under Partial Strong Convexity}
\subsection{Warmup: Convergence of Gradient Descent}
\label{sec:gd_conv}
The previous section shows that $f$ satisfying Assumption \ref{asump:strong_cvx}-\ref{asump:universal_opt} is weaker than the combination of smoothness and strong convexity. Before diving into the convergence of gradient descent, we first show that Assumptions \ref{asump:strong_cvx},\ref{asump:universal_opt} imply the PL condition:
\begin{lemma}
    \label{lem:PL_condition}
    Suppose that Assumption \ref{asump:strong_cvx}, \ref{asump:universal_opt} hold. Then, for all $\bfx\in\R^d$ and $\bfu\in\mathcal{B}^{(2)}_{R_{\bfu}}$, we have:
    \vspace{-0.2cm}
    \[
        \norm{\nabla f(\bfx, \bfu)}_2^2\geq \norm{\nabla_1f(\bfx,\bfu)}_2^2 \geq 2\mu\paren{f(\bfx,\bfu) - f^\star}.
    \]
\end{lemma}
\vspace{-0.2cm}
Recall that, due to the minimum assumption made on the relationship between $f(\bfx,\bfu)$ and $\bfu$, we treat the change of $\bfu$ during the iterates as an error. Thus, we need the following lemma, which bounds how much $f$ is affected by the change of $\bfu$.
\begin{lemma}
    \label{lem:df_du}
    Let Assumptions \ref{asump:g_smooth}, \ref{asump:h_lip} hold. For any $\hat{\mathcal{Q}} > 0$ and $\bfx\in\mathcal{B}^{(1)}_{R_{\bfx}}, \bfu,\bfv\in\mathcal{B}^{(2)}_{R_{\bfu}}$, we have:
    \vspace{-0.2cm}
    \[
        f(\bfx,\bfu)-f(\bfx,\bfv)\leq \hat{\mathcal{Q}}^{-1}L_2\paren{f(\bfx,\bfv)-f^\star} + \tfrac{G_1^2}{2}\paren{L_2 + \hat{\mathcal{Q}}}\norm{\bfu - \bfv}_2^2.
    \]
\end{lemma}
\vspace{-0.2cm}
With the help of Lemmas \ref{lem:PL_condition} and \ref{lem:df_du}, we can show the linear convergence of gradient descent:
\begin{theorem}
    \label{theo:gd_conv}
    Suppose that Assumptions \ref{asump:strong_cvx}-\ref{asump:h_lip} and \ref{asump:universal_opt} hold with $G_1^4 \leq \frac{\mu^2}{8L_2^2}$ and
    \vspace{-0.2cm}
    \[
        R_{\bfx} \geq 16\eta\kappa\sqrt{L_1}\paren{f(\bfx_0,\bfu_0) - f^\star}^{\frac{1}{2}};\; \quad R_{\bfu} \geq 16\eta\kappa G_1\sqrt{L_2}\paren{f(\bfx_0,\bfu_0) - f^\star}^{\frac{1}{2}}.
    \]
    Then there exists constant $c > 0$ such that gradient descent with $\eta = \tfrac{c}{L_1}$ converges according to:
    \[
        \vspace{-0.2cm} f(\bfx_k,\bfu_k) - f^\star \leq \paren{1 - \tfrac{c}{4\kappa}}^k\paren{f(\bfx_0,\bfu_0) - f^\star}.
    \]
\end{theorem}
\vspace{-0.0cm}
I.e., Theorem \ref{theo:gd_conv} shows that gradient descent applied to $f$ converges linearly with a rate of $1 - \Theta(\sfrac{1}{\kappa})$ within our settings. The proofs for Lemma \ref{lem:PL_condition} and \ref{lem:df_du}, and Theorem \ref{theo:gd_conv} are deferred to Appendix \ref{sec:proof_gd_conv}.

\subsection{Acceleration of Nesterov's Momentum}
\label{sec:nesterov_conv}
We will now study the convergence property of Nesterov's momentum in (\ref{eq:nesterov}) under only Assumptions \ref{asump:strong_cvx}-\ref{asump:universal_opt}. Our result shows an accelerated convergence rate compared with gradient descent.
\begin{theorem}
    \label{theo:nesterov_conv}
    Let Assumptions (\ref{asump:strong_cvx})-(\ref{asump:universal_opt}) hold. Consider Nesterov's momentum given by (\ref{eq:nesterov})
    with initialization $\{\bfx_0, \bfu_0\} = \{\bfy_0, \bfv_0\}$. There exists absolute constants $c, C_1, C_2 > 0$, such that, if $\mu, L_1, L_2, G_1, G_2$ and $R_{\bfx}, R_{\bfu}$ satisfy:
    \begin{equation}
        \label{eq:G1_G4_req}
        \begin{aligned}
            & G_1^4 \leq \tfrac{C_1\mu^2}{L_2(L_2 + 1)^2}\paren{\tfrac{1-\beta}{1+\beta}}^3;\quad G_1^2G_2^2\leq \tfrac{C_2\mu^3}{L_2(L_2+1)\sqrt{\kappa}}\paren{\tfrac{1-\beta}{1+\beta}}^2; \\
            & R_{\bfx} \geq \frac{36}{c}\sqrt{\kappa}\paren{\tfrac{\eta(L_2 + 1)}{1 - \beta}}^{\frac{1}{2}}(f(\bfx_0, \bfu_0) - f^\star)^{\frac{1}{2}}; \\
            & R_{\bfu} \geq \frac{36}{c}\sqrt{\kappa}\paren{\tfrac{\eta G_1^2L_2(L_2 + 1)(1+\beta)^3}{\mu\beta(1-\beta)^3}}^{\frac{1}{2}}(f(\bfx_0, \bfu_0) - f^\star)^{\frac{1}{2}},
        \end{aligned}
    \end{equation}
    and, if we choose $\eta = \sfrac{c}{L_1}$, $\beta = \sfrac{(4\sqrt{\kappa} - \sqrt{c})}{(4\sqrt{\kappa} + 7\sqrt{c})}$, then $\bfx_k,\bfy_k\in\mathcal{B}^{(1)}_{R_{\bfx}}$ and $\bfu_k,\bfv_k\in\mathcal{B}^{(2)}_{R_{\bfu}}$ for all $k\in\N$, and Nesterov's recursion converges according to:
    \begin{equation}
        \label{eq:nesterov_conv}
        f(\bfx_k, \bfu_k) - f^\star \leq 2\paren{1 - \tfrac{c}{4\sqrt{\kappa}}}^k(f(
        \bfx_0, \bfu_0) - f^\star).
    \end{equation}
\end{theorem}
Theorem \ref{theo:nesterov_conv} shows that, under Assumptions \ref{asump:strong_cvx}-\ref{asump:universal_opt} with a sufficiently small $G_1$ and $G_2$ as in (\ref{eq:G1_G4_req}), Nesterov's iteration enjoys an accelerated convergence, as in (\ref{eq:nesterov_conv}). Moreover, the iterates of Nesterov's momentum $\{(\bfx_k,\bfy_k)\}_{k=1}^\infty$ and $\{(\bfu_k,\bfv_k)\}_{k=1}^\infty$ stay in a ball around initialization with radius in (\ref{eq:G1_G4_req}). 

To better interpret our result, we first focus on (\ref{eq:G1_G4_req}). By our choice of $\beta$, we have that $1-\beta = \Theta\paren{\sfrac{1}{\sqrt{\kappa}}}$ and $1 + \beta = \Theta\paren{1}$. Therefore, the requirement of $G_1, G_2$ in  (\ref{eq:G1_G4_req}) can be simplified to $G_1^4\leq O\paren{\sfrac{\mu^{\sfrac{7}{2}}}{L_1^{\sfrac{3}{2}}L_2^3}}$ and $G_1^2G_2^2 \leq O\paren{\sfrac{\mu^{\sfrac{9}{2}}}{L_1^{\sfrac{3}{2}}L_2^2}}$. This simplified condition implies that we need a smaller $G_1$ and $G_2$ if $\mu$ is small and $L_1$ and $L_2$ are large. For the requirement on $R_{\bfx}$ and $R_{\bfu}$ in (\ref{eq:G1_G4_req}), we can simplify with $\eta = O\paren{\sfrac{1}{L_1}}$ and $\beta = \Theta\paren{1}$. In this way, $R_{\bfx}$ and $R_{\bfu}$ reduce to $\Omega\paren{\sfrac{L_1^{\sfrac{1}{4}}L_2^{\sfrac{1}{2}}}{\mu^{\sfrac{3}{4}}}}\cdot(f(\bfx_0, \bfu_0) - f^\star)^{\frac{1}{2}}$ and $\Omega\paren{\sfrac{G_1L_1^{\sfrac{3}{4}}L_2}{\mu^{\sfrac{7}{4}}}}\cdot(f(\bfx_0, \bfu_0) - f^\star)^{\frac{1}{2}}$, respectively. Both quantities grow with a larger $L_1$ and $L_2$ and a smaller $\mu$. Noticeably $R_{\bfu}$ also scales with $G_1$. Focusing on the convergence property in (\ref{eq:nesterov_conv}), we can conclude that Nesterov's momentum achieves an accelerated convergence rate of $1 - \Theta\paren{\sfrac{1}{\sqrt{\kappa}}}$ compared with the $1 - \Theta\paren{\sfrac{1}{\kappa}}$ rate in Theorem \ref{theo:gd_conv}. In more detail, we discuss the proof of Theorem \ref{theo:nesterov_conv} in the sections below.

\subsection{Technical Difficulty}
\label{sec:difficulty}
Similar to the previous work on showing the convergence of Nesterov's momentum \citep{bansal2019potential,acceleration2021}, the core of our proof is the construction of a Lyapunov function that upper bounds the optimality gap $f(\bfx_k,\bfu_k) - f^*$ at each step $k$, and enjoys a linear convergence. However, the construction of this Lyapunov function faces the following difficulty.

\medskip
\noindent\textbf{Difficulty 1.} \textit{Most previous analyses of Nesterov's momentum use the global minimum as a reference point to construct the Lyapunov function; see \citet{bansal2019potential, acceleration2021}. In the original proof of Nesterov, the construction of the estimating sequence also assumes the existence of a unique global minimum \citep{nesterov2018lectures}. However, in our scenario, the objective function is non-convex. It allows the existence of multiple global minima, which prevents us from directly applying the Lyapunov function or estimating sequence, as in previous works.}

\medskip
\noindent While the non-convexity of $f$ introduces the possibility of multiple global minima, Assumption \ref{asump:strong_cvx} implies that, with a fixed $\bfu$, there exists a unique $\bfx^{\star}(\bfu)$ that minimizes $f(\bfx, \bfu)$. Moreover, Assumption \ref{asump:universal_opt} implies that, for all $\bfu\in\mathcal{B}^{(2)}_{R_{\bfu}}$, the local minimum $(\bfx^{\star}(\bfu), \bfu)$ is also a global minimum. Thus, we resolve the difficulty by using the $\bfx^{\star}(\bfu_k)$ as the reference point for the Lyapunov function at the $k$th iteration and ensure the stability of the Lyapunov function by bounding the change of $\bfx^{\star}(\bfu_k)$. The following lemma gives a characterization of this property.
\begin{lemma}
    \label{lem:minimum_moving}
    Let $\bfx^{\star}(\bfu) = \argmin_{\bfx\in\R^{d_1}}f(\bfx,\bfu)$. Suppose Assumptions \ref{asump:strong_cvx} and \ref{asump:grad_lip} hold. Then, we have:
    \[
        \norm{\bfx^{\star}(\bfu_1) - \bfx^{\star}(\bfu_2)}_2\leq \tfrac{G_2}{\mu}\norm{\bfu_1 - \bfu_2}_2,\quad\forall \bfu_1,\bfu_2\in\mathcal{B}^{(2)}_{R_{\bfu}}.
    \]
\end{lemma}
Lemma \ref{lem:minimum_moving} indicates that, if we view $\bfx^{\star}(\bfu)$ as a function of $\bfu$, then this function is $\frac{G_2}{\mu}$-Lipschitz. For a fixed $\bfu$, given the nice properties on $\bfx$, the iterates of Nesterov's momentum will guide $\bfx$ to the minimum, based on the current $\bfu$. Lemma \ref{lem:minimum_moving} guarantees that the progress toward the minimum induced by $\bfu_1$ does not deviate much from the progress toward the minimum induced by $\bfu_2$. However, to apply Lemma \ref{lem:minimum_moving}, we must control the change of $\bfu_k$ between iterations. Indeed, this bound is also necessary to apply the smoothness-like condition in Lemma \ref{lem:df_du}. Unlike gradient descent, $\left\{\bfu_k\right\}_{k=1}^\infty$ generated by Nesterov's momentum introduces the following difficulty.

\medskip
\noindent\textbf{Difficulty 2.} \textit{Unrolling the Nesterov's momentum iterates shows that $\bfu_{k+1}- \bfu_k$ is a linear combination of previous gradients. Under the assumption of smoothness, the norm of the gradient is bounded by a factor times the optimality gap at the current point, namely}:\footnote{\text{For the proof of this property, please see Lemma \ref{lem:du_grad_bound}}}
\[
    \norm{\nabla_2f\paren{\bfx,\bfu}}_2^2\leq 2G_1L_2\paren{f\paren{\bfx,\bfu} - f^*}.
\]
\textit{In the case of gradient descent, the applied gradients are evaluated at steps $\paren{\bfx_k,\bfu_k}$, and thus $\norm{\nabla_2f\paren{\bfx_k,\bfu_k}}_2$ can be controlled since $\paren{f\paren{\bfx,\bfu} - f^*}$ can be shown to enjoy a linear convergence by an induction-based argument \citep{du2018gradient,nguyen2021ontheproof}. However, in the case of Nesterov's momentum, we cannot directly utilize this relationship since the applied gradient is evaluated at the intermediate step $(\bfy_k,\bfv_k)$, and while we know that the optimality gap at $(\bfx_k,\bfu_k)$ converges linearly, we have very little knowledge about the optimality gap at $(\bfy_k,\bfv_k)$.}

\medskip
\noindent To tackle this difficulty, our analysis starts with a careful bound on $\norm{\bfx_{k+1} - \bfx_k}_2^2$ by utilizing the convexity with respect to $\bfx$ to characterize the inner product $\inner{\nabla_1f(\bfy_k,\bfv_k)}{\bfx_k-\bfx_{k-1}}$. After that, we bound $\norm{\nabla_1f(\bfy_k,\bfv_k)}_2^2$ using a combination of $\norm{\bfx_{k+1} - \bfx_k}_2^2$ and $\norm{\bfx_{k} - \bfx_{k-1}}_2^2$. Lastly, we relate $\norm{\nabla_2f(\bfy_k,\bfv_k)}_2^2$ to $\norm{\nabla_1f(\bfy_k,\bfv_k)}_2^2$ using the following gradient dominance property.
\begin{lemma}
    \label{lem:grad_dominance}
    Suppose that Assumptions \ref{asump:strong_cvx}, \ref{asump:g_smooth}, \ref{asump:h_lip}, and \ref{asump:universal_opt} hold. Then, we have:
    \[
        \norm{\nabla_2f(\bfx,\bfu)}_2^2\leq \tfrac{G_1^2L_2}{\mu}\norm{\nabla_1f(\bfx,\bfu)}_2^2, \quad\forall \bfx\in\mathcal{B}^{(1)}_{R_{\bfx}};\;\bfu\in\mathcal{B}^{(2)}_{R_{\bfu}}.
    \]
\end{lemma}
Lemma \ref{lem:grad_dominance} establishes the upper bound of $\norm{\nabla_2f(\bfx,\bfu)}_2^2$ using $\norm{\nabla_1f(\bfx,\bfu)}_2^2$. Direct application of this result will contribute to the bound on $\norm{\bfu_{k+1} - \bfu_k}_2$. Intuitively, this lemma also implies that the effect of gradient update on $\bfu$ is less significant than that on $\bfx$.

\subsection{Proof Overview} 
Our analysis is based on the Lyapunov function proof given by \citet{bansal2019potential} for proving the accelerated convergence rate of Nesterov's momentum in (\ref{eq:plain_nesterov}) in minimizing a strongly convex and smooth objective $\hat{f}(\bfw)$. In particular, \citet{bansal2019potential}\footnote{after rephrasing in the notations used in our scenario} uses the following Lyapunov function,
\begin{equation}
    \label{eq:lyapunov_sc}
    \hat{\phi}_k = f\paren{\bfw_k} - f^* + \frac{\mu}{2}\norm{\hat{\bfz}_k - \bfw^\star}_2
\end{equation}
where $\bfw^\star$ is the global minimum, and $\hat{\bfz}_k$ can be considered as a mixing of the sequences $\left\{\bfw_k\right\}_{k=1}^{\infty}$ and $\left\{\bar{\bfw}_k\right\}_{k=1}^{\infty}$ in (\ref{eq:plain_nesterov}). In particular, the second term in (\ref{eq:lyapunov_sc}) computes the distance between the mixed variable $\hat{\bfz}_k$ and the reference point $\bfw^\star$ and is added to $\hat{\phi}_k$ to guarantee a linear convergence on $\hat{\phi}_k$.
Following our discussion in the previous section, we construct our Lyapunov function using $\bfx^{\star}\paren{\bfu} = \argmin_{\bfx\in\R^{d_1}}f(\bfx,\bfu)$ as a reference point. For the simplicity of notations, we define $\bfx^{\star}_k = \bfx^{\star}\paren{\bfu_k}$. Similar to \citet{bansal2019potential}, we let $\bfz_k$ be the linear combination of $\bfy_k$ and $\bfx_k$, and choose a proper scaling factor $\mathcal{Q}_1$ for the distance between $\bfz_k$ and the previous reference point $\bfx_{k-1}^{\star}$. For some properly choose $\gamma$ and $\lambda$, we define
\[
    \bfz_k = \frac{1 - \beta\lambda}{\beta\lambda}(\bfy_k-\bfx_k) + \bfy_k;\;\;\mathcal{Q}_1 = \frac{\lambda^2}{2\eta(1 + \gamma)^5}
\]
Compared with the proof of Nesterov's momentum on smooth and strongly convex functions, the proof in our setting has to accommodate the errors caused by the change of $\bfu$. Our complicated scaling in the form of $\bfz_k$ and $\mathcal{Q}_1$ is to make sure the errors caused by $\bfu$ can be properly canceled out by the positive progress made by updating $\bfx$. Setting $\bfy_{-1} = \bfy_0$ and $\bfv_{-1} = \bfv_0$, we consider the following Lyapunov function:
\[
    \phi_k = f(\bfx_k,\bfu_k) - f^\star + \mathcal{Q}_1\norm{\bfz_k-\bfx_{k-1}^{\star}}_2^2 + \frac{\eta}{8}\norm{\nabla_1f(\bfy_{k-1},\bfv_{k-1})}_2^2
\]
The last term in the expression of $\phi_k$ also eliminates the errors from updating $\bfu$. Our proof recursively establishes the following three properties:
\begin{align}
    \label{eq:lyapunov_step}
    & \paren{1 - \tfrac{c}{2\sqrt{\kappa}}}^{-1}\phi_{k+1} - \phi_k \leq \frac{c}{4\sqrt{\kappa}}\paren{1 - \tfrac{c}{4\sqrt{\kappa}}}^k\phi_0; \\
    \label{eq:u_iter_bound}
    \norm{\bfx_k - \bfx_{k-1}}_2^2 & \leq \mathcal{Q}_2\paren{1 - \tfrac{c}{4\sqrt{\kappa}}}^k\phi_0;\quad\norm{\bfu_k - \bfu_{k-1}}_2^2 \leq G_1^2\mathcal{Q}_3\paren{1 - \tfrac{c}{4\sqrt{\kappa}}}^k\phi_0.
\end{align}
Intuitively, (\ref{eq:lyapunov_step}) implies an accelerated linear convergence of $f(\bfx_{k'},\bfu_{k'}) - f^\star$ up to $k'\leq k+1$, which further implies the bound on $\norm{\bfx_{k'} - \bfx_{k'-1}}_2$ and $\norm{\bfu_{k'} - \bfu_{k'-1}}_2$ as in (\ref{eq:u_iter_bound}) up to $k'\leq k+1$. In turn, (\ref{eq:u_iter_bound}) will guarantee that $\bfx_k\in\mathcal{B}_{R_{\bfx}},\bfu\in\mathcal{B}_{R_{\bfu}}$, and control the error caused by updating $\bfu$. These two conditions combined will imply that (\ref{eq:lyapunov_step}) holds. This idea is further detailed below.

By our choice of $\phi_k$, we must have that $f(\bfx_k,\bfu_k) - f^\star\leq \phi_k$. Unrolling (\ref{eq:lyapunov_step}) implies that:
\begin{equation}
    \label{eq:lyapunov_conv}
    f(\bfx_k,\bfu_k) - f^\star \leq \phi_{k} \leq \paren{1 - \tfrac{c}{4\sqrt{\kappa}}}^{k}\phi_0,
\end{equation}
Combined with the bound that $\phi_0 \leq 2\paren{f(\bfx_0,\bfu_0)-f^{\star}}$\footnote{For the detail please see Lemma \ref{lem:phi0_upper_bound}}, (\ref{eq:lyapunov_step}) further implies the convergence in (\ref{eq:nesterov_conv}). The following lemma shows that, if (\ref{eq:nesterov_conv}) holds for all $k\leq \hat{k}$, we can guarantee (\ref{eq:u_iter_bound}) with $k = \hat{k}+1$.
\begin{lemma}
    \label{lem:xu_iter_bound}
    Let the Assumptions of Theorem \ref{theo:nesterov_conv} hold. If (\ref{eq:lyapunov_conv}) holds for all $k\leq \hat{k}$, then (\ref{eq:u_iter_bound}) holds for $k = \hat{k} + 1$ with $\mathcal{Q}_2 = \frac{6\eta(L_2+1)}{1-\beta}$ and $\mathcal{Q}_3 = \frac{6\eta L_2(L_2 + 1)(1+\beta)^3}{\mu\beta(1-\beta)^3}$.
\end{lemma}
The proof of Lemma \ref{lem:xu_iter_bound} utilizes how we resolve Difficulty 2 in the previous section. Lemma \ref{lem:xu_iter_bound} implies that up to iteration $\hat{k} + 1$, $\bfx_k$ and $\bfu_k$ stay in $\mathcal{B}^{(1)}_{R_{\bfx}}$ and $\mathcal{B}^{(2)}_{R_{\bfu}}$ with the lower bound of $R_{\bfx}$ and $R_{\bfu}$ assumed in Theorem \ref{theo:nesterov_conv}. This allows us to apply Assumptions \ref{asump:strong_cvx}-\ref{asump:universal_opt} in showing (\ref{eq:lyapunov_step}) for iteration $\hat{k} + 1$. Moreover, the bound on $\norm{\bfu_k - \bfu_{k-1}}_2$ connects the following lemma to (\ref{eq:lyapunov_step}).
\begin{lemma}
    \label{lem:lyapunov_conv_raw}
    Let the Assumptions of Theorem (\ref{theo:nesterov_conv}) hold. Then, we have:
    \[
        \paren{1 - \tfrac{c}{2\sqrt{\kappa}}}^{-1}\phi_{k+1} - \phi_k \leq c\beta^2\sqrt{\kappa}\paren{G_1^2L_2 + \tfrac{8\mathcal{Q}_1G_1^2G_2^2}{\mu^2}}\norm{\bfu_k - \bfu_{k-1}}_2^2
    \]
\end{lemma}
The proof of Lemma \ref{lem:lyapunov_conv_raw} has a similar idea to \citet{bansal2019potential} when showing that the Lyapunov function converges. The additional difficulty of Lemma \ref{lem:lyapunov_conv_raw} lies in carefully controlling the errors caused by the change of $\bfu_k$. Plugging the form of $\mathcal{Q}_3$ into the upper bound in (\ref{eq:u_iter_bound}) and then plugging the resulting upper bound into Lemma \ref{lem:lyapunov_conv_raw} yields
\[
    \paren{1 - \tfrac{c}{2\sqrt{\kappa}}}^{-1}\phi_{k+1} - \phi_k \leq \frac{6c\sqrt{\kappa}\eta L_2(L_2 + 1)(1+\beta)^3}{\mu(1-\beta)^3}\paren{G_1^4L_2 + \tfrac{8\mathcal{Q}_1G_1^4G_2^2}{\mu^2}}\paren{1 - \tfrac{c}{4\sqrt{\kappa}}}^k\phi_0
\]
The equation above directly implies (\ref{eq:lyapunov_step}) after we plug in the upper bound on $G_1$ and $G_2$ assumed in Theorem \ref{theo:nesterov_conv}. This establishes the inductive step and thus finishes the proof.
\begin{wrapfigure}{R}{0.45\textwidth}
    \vspace{-1cm}
    \begin{minipage}{0.45\textwidth}
        \centering
        \vspace{1cm}
        \subfigure[$\sigma_{\max}(\bfA_2) = 0.01$]{\includegraphics[width=\textwidth]{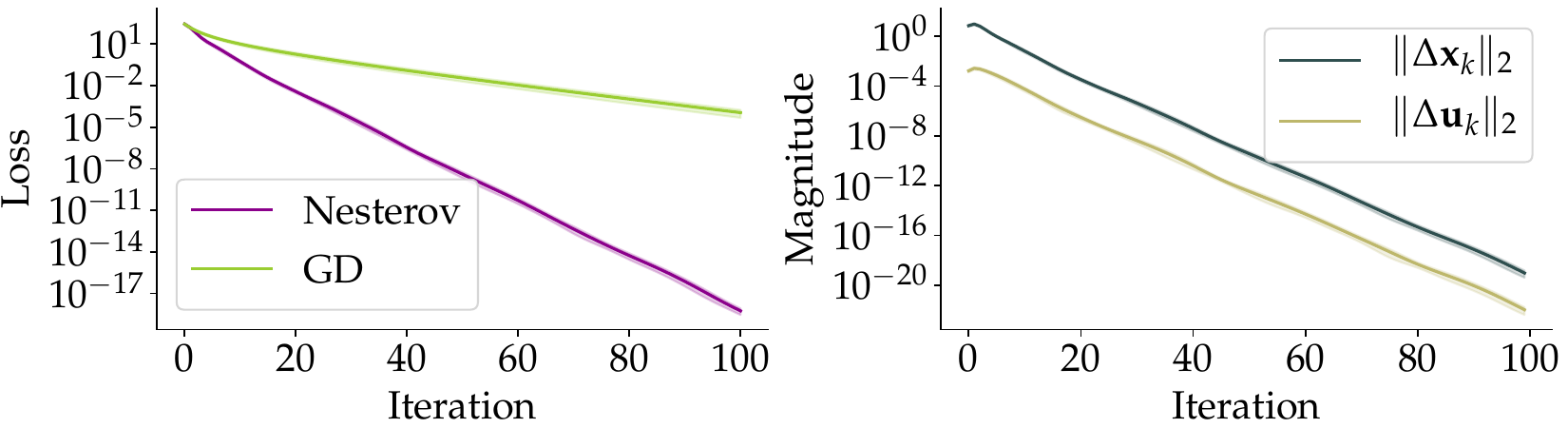}}
        \subfigure[$\sigma_{\max}(\bfA_2) = 0.3$]{\includegraphics[width=\textwidth]{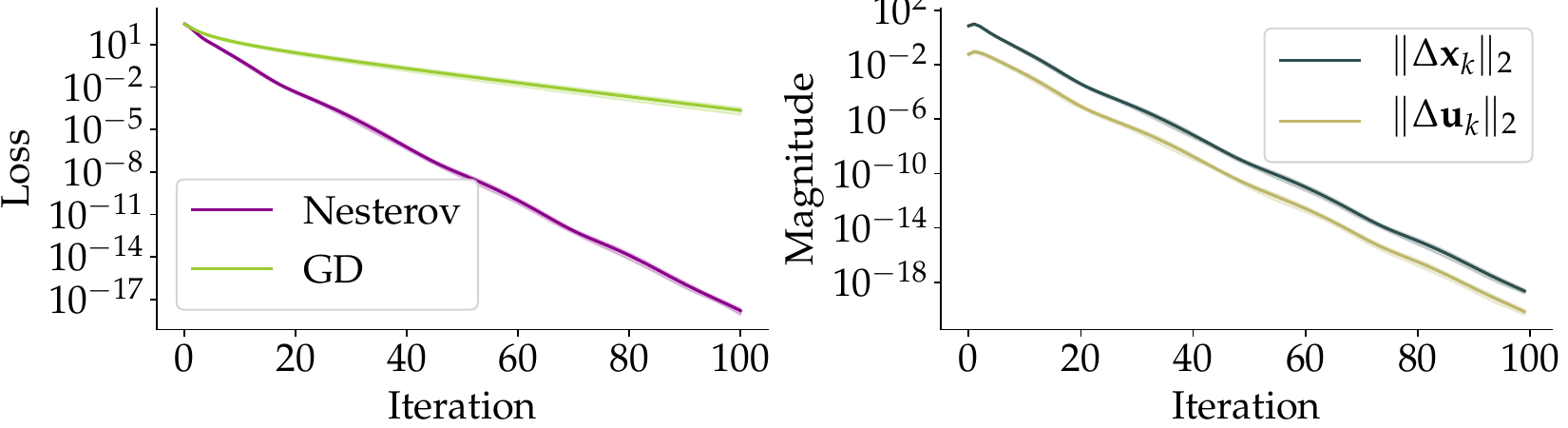}}
        \subfigure[$\sigma_{\max}(\bfA_2) = 5$]{\includegraphics[width=\textwidth]{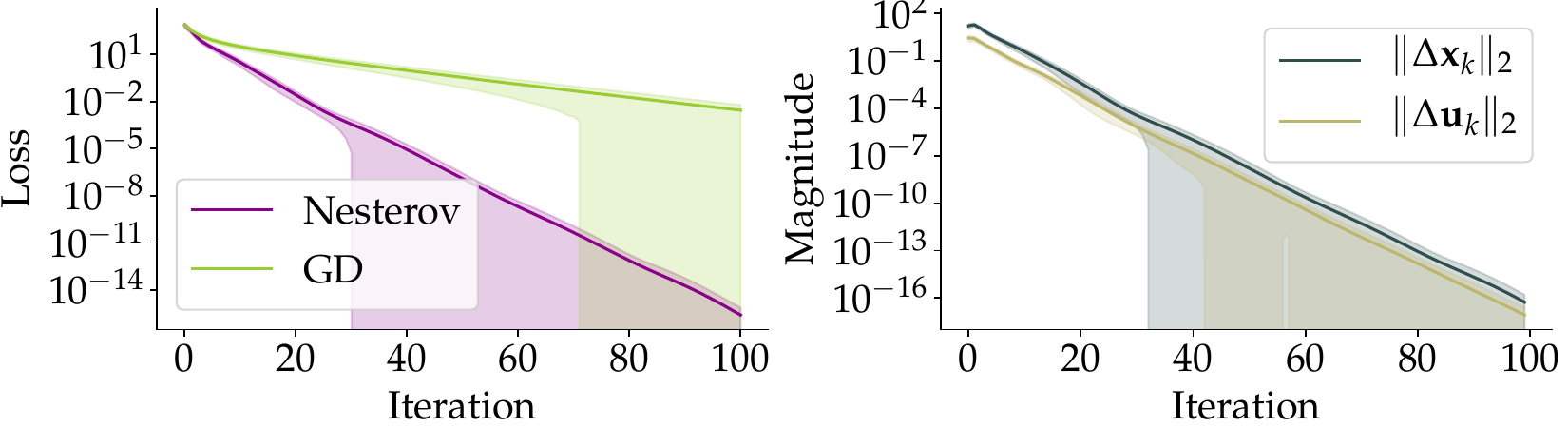}}
    \end{minipage}
    \caption{Experiment of learning additive model with gradient descent and Nesterov's momentum.}
    \label{fig:exp}
    \vspace{-2cm}
\end{wrapfigure}
\section{Realization of Assumption \ref{asump:strong_cvx}-\ref{asump:universal_opt}}
We consider two realizations of problems that satisfy Assumptions \ref{asump:strong_cvx}-\ref{asump:universal_opt}. By enforcing the requirement in Theorem \ref{theo:nesterov_conv} on the two models, we show that although nonconvex and possibly non-smooth, the two models enjoy accelerated convergence when trained with Nesterov's momentum.
\subsection{Additive model}
\label{sec:addition_model}
Given matrices $\bfA_1\in\R^{m\times m},\bfA_2\in\R^{m\times d}$ and a non-linear function $\sigma:\R^{m}\rightarrow \R^m$, we consider $h(\bfx, \bfu) = \bfA_1\bfx + \sigma(\bfA_2\bfu)$ as the summation of a linear model and a non-linear model. If we train $h(\bfx,\bfu)$ over the loss $g(\bfs) = \frac{1}{2}\norm{\bfs - \bfb}_2^2$ for some label $\bfb\in\R^m$, the objective can be written as
\begin{equation}
    \label{eq:addit}
    \begin{aligned}
    f(\bfx, \bfu) & = g(\bfA_1\bfx + \sigma\paren{\bfA_2\bfu})\\& = \frac{1}{2}\norm{\bfA_1\bfx + \sigma\paren{\bfA_2\bfu} - \bfb}_2^2.
    \end{aligned}
\end{equation}
Due to the non-linearity of $\sigma$, $f(\bfx, \bfu)$ is generally non-convex. If we further choose $\sigma$ to be some non-smooth function such as ReLU, i.e., $\sigma(\bfx)_i = \max\{0, x_i\}$, the objective can also be non-smooth. Yet, assuming that $\sigma$ is Lipschitz, we can show that $f(\bfx, \bfu)$ satisfies Assumptions \ref{asump:strong_cvx}-\ref{asump:universal_opt}.
\begin{lemma}
    \label{lem:additive_sat_asump}
    Let $\sigma$ be $B$-Lipschitz and $\sigma_{\min}(\bfA_1) > 0$. Then $f(\bfx,\bfu)$ satisfies Assumptions \ref{asump:strong_cvx}-\ref{asump:universal_opt} with 
    \begin{equation}
        \begin{aligned}
           R_{\bfx} = & R_{\bfu} = \infty;\; \mu = \sigma_{\min}\paren{\bfA_1}^2;\; L_1 = \sigma_{\max}\paren{\bfA_1}^2;\; L_2 = 1\\
            G_1 & = B\sigma_{\max}\paren{\bfA_2};\; G_2 = B\sigma_{\max}\paren{\bfA_1}\sigma_{\max}\paren{\bfA_2}.
        \end{aligned}
    \end{equation}
\end{lemma}
Notice that while $\mu$ and $L_1$ depend entirely on the property of $\bfA_1$, both $G_1$ and $G_2$ can be made smaller by choosing $\bfA_2$ with a small enough $\sigma_{\max}(\bfA_2)$. Intuitively, this means that $G_1$ and $G_2$ can be controlled, as long as the component that introduces the non-convexity and non-smoothness $\sigma\paren{\bfA_2\bfu}$ is small enough. Therefore, we can apply Theorem \ref{theo:nesterov_conv} to the minimization of (\ref{eq:addit}).

\begin{theorem}
    \label{theo:addit_conv}
    Consider the problem in (\ref{eq:addit}) and assume using Nesterov's momentum to minimize $f(\bfx, \bfu)$, where  $\sigma$ is a $B$-Lipschitz function. 
    Let $\kappa = \sfrac{\sigma_{\max}\paren{\bfA_1}^2}{\sigma_{\min}\paren{\bfA_1}^2}$, and suppose:
    \begin{equation}
        \label{eq:addit_conv_req}
        \sigma_{\min}\paren{\bfA_1} \geq \tilde{C}\sigma_{\max}\paren{\bfA_2}B\kappa^{0.75},
    \end{equation}
    for large enough constant $\tilde{C} > 0$.
    Then there exists constant $c > 0$ such that if we choose $\eta = \sfrac{c}{\sigma_{\max}\paren{\bfA_1}^2}$ and $\beta = \sfrac{(4\sqrt{\kappa} - \sqrt{c})}{(4\sqrt{\kappa} + 7\sqrt{c})}$, Nesterov's momentum in (\ref{eq:nesterov}) converges according to:
    \[
        f(\bfx_k,\bfu_k)\leq 2\paren{1 - \frac{c}{4\sqrt{\kappa}}}^kf(\bfx_0, \bfu_0).
    \]
\end{theorem}
Notice that the requirement in (\ref{eq:addit_conv_req}) favors a larger $\sigma_{\min}(\bfA_1)$ and smaller $\sigma_{\max}(\bfA_2), B$ and $\kappa$. 
Using this example, we empirically verify the theoretical result of Theorem \ref{theo:nesterov_conv}, as shown in Figure \ref{fig:exp}. The three rows correspond to the cases of $\sigma_{\max}\paren{\bfA_2} \in\{0.01, 0.3, 5\}$, respectively. The plot lines denote the average loss/distance among ten trials, while the shaded region denotes the standard deviation. Observing the plots in the left column, in all three cases, Nesterov's momentum achieves a faster convergence compared with gradient descent, while the case with the largest $\sigma_{\max}\paren{\bfA_2}$ introduces the largest variance between results of the ten trials. Recall that $\sigma_{\max}\paren{\bfA_2}$ controls the magnitude of $G_1$ and $G_2$. Thus, this phenomenon shows that when $G_1$ and $G_2$ become larger, the theoretical guarantee in Theorem \ref{theo:nesterov_conv} begins to break down, and the result depends more on the initialization. Figures in the right column plots the evolution of $\norm{\bfx_{k}-\bfx_{k-1}}_2^2$ and $\norm{\bfu_{k}-\bfu_{k-1}}_2^2$. All three figures show that the two quantity decrease linearly. This phenomenon supports the linear decay of the two quantities, as shown in Lemma \ref{lem:xu_iter_bound}. Moreover, as $\sigma_{\max}\paren{\bfA_2}$ increases, the relative magnitude of $\norm{\bfu_{k}-\bfu_{k-1}}_2^2$ to $\norm{\bfx_{k}-\bfx_{k-1}}_2^2$ also increase (the line of ``$\|\Delta \bfu_k\|_2$'' get closer to ``$\|\Delta \bfx_k\|_2$''). This supports that $\norm{\bfu_{k}-\bfu_{k-1}}_2^2$ scales with $G_1$, as shown in Lemma \ref{lem:xu_iter_bound}.

\subsection{Deep ReLU Neural Networks}
\label{sec:dnn}
Consider the $\Lambda$-layer ReLU neural network with layer widths $\{d_\ell\}_{\ell=0}^{\Lambda}$. Denoting the number of training samples with $n$, we consider the input and label of the training data given by $\bfX\in\R^{n\times d_0}$ and $\bfY\in\R^{n\times d_{\Lambda}}$. Let the weight matrix in the $\ell$-th layer be $\bfW_{\ell}$. We use $\bm{\theta} = \{\bfW_{\ell}\}_{\ell=1}^{\Lambda}$ to denote the collection of all weights and $\sigma(\bfA)_{ij} = \max\{0,a_{ij}\}$ to denote the ReLU function. Then, the output of each layer is given by:
\begin{equation}
    \label{eq:relu_net }
    \bfF_{\ell}\paren{\bm{\theta}} = \begin{cases}
    \bfX, & \text{if }\ell = 0;\\
    \sigma\paren{\bfF_{\ell-1}\paren{\bm{\theta}}\bfW_{\ell}}, & \text{if }\ell\in[\Lambda-1];\\
    \bfF_{\Lambda-1}\paren{\bm{\theta}}\bfW_{\Lambda}, & \text{if }\ell=\Lambda.
    \end{cases}
\end{equation}
We consider the training of $\bfF_{\Lambda}\paren{\bm{\theta}}$ over the MSE loss, as in $\calL(\bm{\theta}) = \frac{1}{2}\norm{\bfF_{\Lambda}\paren{\bm{\theta}} - \bm{Y}}_F^2$.
We can interpret the scenario using our partition model in (\ref{eq:partition_model}). Let $g(\bfs) = \frac{1}{2}\norm{\bfs - \texttt{V}\paren{\bfY}}_2^2$. If we partition the parameter $\bm{\theta}$ into $\bfx = \texttt{V}\paren{\bfW_{\Lambda}}$ and $\bfu = \left(\texttt{V}\paren{\bfW_1},\dots,\texttt{V}\paren{\bfW_{\Lambda-1}}\right)$, then we can write:
\begin{equation}
    \label{eq:nn_obj}
    h(\bfx, \bfu) = \texttt{V}\paren{\bfF_{\Lambda}\paren{\bm{\theta}}};\quad f(\bfx,\bfu) = \frac{1}{2}\norm{\texttt{V}\paren{\bfF_{\Lambda}\paren{\bm{\theta}}} - \texttt{V}\paren{\bfY}}_2^2 = \frac{1}{2}\norm{\bfF_{\Lambda}\paren{\bm{\theta}} - \bfY}_F^2.
\end{equation}
For some given $\bfx$ and $\bfu$, we let $\bfW_{\Lambda}(\bfx)$ be the matrix such that $\bfx = \texttt{V}\paren{\bfW_{\Lambda}(\bfx)}$; similarly, we let $\bfW_{\ell}(\bfu)$ with $\ell\in[L-1]$ be the matrices such that $\bfu = \left(\texttt{V}\paren{\bfW_1(\bfu)},\dots,\texttt{V}\paren{\bfW_{\Lambda-1}(\bfu)}\right)$. Denote $\lambda_{\Lambda} = \sup_{\bfx\in\mathcal{B}^{(1)}_{R_{\bfx}}}\sigma_{\max}\paren{\bfW_{\Lambda}(\bfx)}$ and $\lambda_{\ell} = \sup_{\bfu\in\mathcal{B}^{(2)}_{R_{\bfx}}}\sigma_{\max}\paren{\bfW_{\ell}(\bfu)}$ for $\ell\in[\Lambda-1]$. Moreover, denote $\lambda_{i\rightarrow j} = \prod_{\ell=i}^j\lambda_{\ell}$. Then we can show that $f(\bfx,\bfu)$ defined in (\ref{eq:nn_obj}) satisfies Assumptions \ref{asump:strong_cvx}-\ref{asump:universal_opt}.
\begin{lemma}
    \label{lem:nn_sat_asump}
    Let $\bm{\theta}(0)$ be the initialization of the ReLU network in (\ref{eq:relu_net }) and $\alpha_0 = \sigma_{\min}\paren{\bfF_{\Lambda-1}\paren{\bm{\theta}(0)}}$. Assume that each $\bfW_{\ell}$ is initialized such that $\norm{\bfW_{\ell}(0)}_2\leq \frac{\lambda_{\ell}}{2}$ and $\alpha_0 > 0$. Then, $f(\bfx, \bfu)$ satisfies Assumptions \ref{asump:strong_cvx}-\ref{asump:universal_opt} with:
    \begin{gather*}
    R_{\bfx} = \frac{\lambda_{\Lambda}}{2};\quad R_{\bfu} = \frac{1}{2}\paren{\min_{\ell\in[\Lambda - 1]}\lambda_{\ell}}\min\left\{1, \frac{\alpha_0}{2\sqrt{\Lambda}\norm{\bfX}_F\lambda_{1\rightarrow\Lambda - 1}}\right\}^2;\quad\mu = \frac{\alpha_0^2}{2}\\  L_1 = \norm{\bfX}_F^2\lambda_{1\rightarrow\Lambda - 1}^2;\quad L_2 = 1;\quad G_1 = \paren{\lambda_{\Lambda} + R_{\bfu}}\sqrt{\Lambda}\norm{\bfX}_F\lambda_{1\rightarrow\Lambda-1}\paren{\min_{\ell\in[\Lambda-1]}\lambda_{\ell}}^{-1}\\
    G_2 = \paren{\paren{2\lambda_{\Lambda} + R_{\bfu}}\norm{\bfX}_F\lambda_{1\rightarrow\Lambda - 1} + \norm{\bfY}_F}\sqrt{\Lambda}\norm{\bfX}_F\lambda_{1\rightarrow\Lambda - 1}\paren{\min_{\ell\in[\Lambda-1]}\lambda_{\ell}}^{-1}
\end{gather*}
\end{lemma}
As shown in Lemma 3.3 by \citet{nguyen2021ontheproof}, we can guarantee that $\alpha_0 > 0$ with sufficient over-parameterization. To show the acceleration of Nesterov's momentum when training (\ref{eq:relu_net }), we need to $i)$ guarantee that the condition of $R_{\bfx}$ and $R_{\bfu}$ in (\ref{eq:G1_G4_req}) satisfies the upper bound in Lemma \ref{lem:nn_sat_asump}, and $ii)$ the quantities $\mu, L_1, L_2, G_1$ and $G_2$ defined in Lemma \ref{lem:nn_sat_asump} satisfy the requirement in (\ref{eq:G1_G4_req}). Enforcing the two conditions with sufficient over-parameterization gives us the following theorem.
\begin{theorem}
    \label{theo:nn_nesterov_conv}
    Consider training the ReLU neural network in (\ref{eq:relu_net }) using the MSE loss, or equivalently, minimizing $f(\bfx, \bfu)$ defined in (\ref{eq:nn_obj}) with Nesterov's momentum with $\eta = \sfrac{c}{L_1}$ and $\beta = \frac{4\sqrt{\kappa} - \sqrt{c}}{4\sqrt{\kappa} + 7\sqrt{c}}$, where $\kappa = \tfrac{2L_1}{\alpha_0^2}$ and $\alpha_0, L_1$ defined in Lemma \ref{lem:nn_sat_asump}. If the width of the network satisfies:
    \begin{equation}
        \label{eq:overparam_req}
        d_{\ell} = \Theta\paren{m}\quad \forall \ell \in[\Lambda-2];\quad d_{\Lambda-1} = \Omega\paren{n^{4.5}\max\left\{n,d^2\right\}}, 
    \end{equation}
    for some $m\geq \max\{d_0, d_{\Lambda}\}$, and we initialize the weights according to:
    \[
        \left[\bfW_{\ell}(0)\right]_{ij} \sim\mathcal{N}\paren{0, d_{\ell-1}^{-1}},\quad \forall \ell\in[\Lambda-1];\quad \left[\bfW_{\Lambda}(0)\right]_{ij} \sim\mathcal{N}\paren{0, d_{\Lambda-1}^{-\frac{3}{2}}}.
    \]
    Then, with high probability over the initialization, there exists an absolute constant $c>0$ such that:
    \begin{equation}
        \label{eq:nn_nest_conv}
        \calL\paren{\bm{\theta}(k)} \leq 2\paren{1 - \tfrac{c}{4\sqrt{\kappa}}}^k\calL(\bm{\theta}(0)).
    \end{equation}
\end{theorem}
As in prior work \citep{nguyen2021ontheproof}, we treat the depth of the neural network $\Lambda$ to be a constant when computing the over-parameterization requirement. Next, we compare our result with Theorem 2.2 and Corollary 3.2 of \citet{nguyen2021ontheproof}. To deal with the additional complexity of Nesterov's momentum as introduced in Section \ref{sec:difficulty}, our over-parameterization is slightly larger than the over-parameterization of $d_{\Lambda-1} = \Theta\paren{n^3m^3}$ in Corollary 3.2 of \citet{nguyen2021ontheproof}. Moreover, in Theorem 2.2 of \citet{nguyen2021ontheproof}, since their choice of $\eta$ is also $O\paren{\sfrac{1}{L_1}}$, the convergence rate they achieve for gradient descent is $1 - \Theta\paren{\sfrac{1}{\kappa}}$. Compared with this rate, Theorem \ref{theo:nn_nesterov_conv} achieves a faster convergence of $1 - \Theta\paren{\sfrac{1}{\sqrt{\kappa}}}$. This shows that Nesterov's momentum enjoys acceleration when training deep ReLU neural networks.

\section{Conclusion and Broader Impact}
We consider the minimization of a new class of objective functions, namely the partition model, where the function is smooth and strongly convex with respect to only a subset of its parameters. This class of objectives is more general than the class of smooth and strongly convex functions. We prove the convergence of gradient descent and Nesterov's momentum on this class of objectives under certain assumptions and show that Nesterov's momentum achieves an accelerated convergence rate of $1 - \Theta\paren{\sfrac{1}{\sqrt{\kappa}}}$ compared to the $1 - \Theta\paren{\sfrac{1}{\kappa}}$ convergence rate of gradient descent. Moreover, we considered the training of the additive model and deep ReLU networks as two realizations of the partition model. We showed the acceleration of Nesterov's momentum on these two realizations.

Future works can focus on three aspects. First, one can consider the case where Assumption \ref{asump:universal_opt} does not hold, and study whether Nesterov's momentum can still converge to up to some error with acceleration under a milder condition of this assumption. Second, one can extend the analysis to different neural network architectures by investigating whether Assumptions \ref{asump:strong_cvx}-\ref{asump:universal_opt} hold on CNNs and ResNets. Lastly, since the weight selection process is extensively studied by literature on neural network pruning, one can study whether neural network pruning keeps weights with good optimization properties and potentially connects our result with the theory of pruning methods such as the Lottery Ticket Hypothesis.

\acks{This work is supported by NSF CMMI no. 2037545 and NSF CAREER award no. 2145629, Welch Foundation Grant \#A22-0307, a Microsoft Research Award, and an Amazon Research Award.}

\bibliography{alt2024}

\begin{thebibliography}{58}
\providecommand{\natexlab}[1]{#1}
\providecommand{\url}[1]{\texttt{#1}}
\expandafter\ifx\csname urlstyle\endcsname\relax
  \providecommand{\doi}[1]{doi: #1}\else
  \providecommand{\doi}{doi: \begingroup \urlstyle{rm}\Url}\fi

\bibitem[Allen-Zhu et~al.(2019{\natexlab{a}})Allen-Zhu, Li, and
  Song]{allenzhu2019convergence}
Zeyuan Allen-Zhu, Yuanzhi Li, and Zhao Song.
\newblock On the convergence rate of training recurrent neural networks,
  2019{\natexlab{a}}.

\bibitem[Allen-Zhu et~al.(2019{\natexlab{b}})Allen-Zhu, Li, and
  Song]{zhu2019aconvergence}
Zeyuan Allen-Zhu, Yuanzhi Li, and Zhao Song.
\newblock A convergence theory for deep learning via over-parameterization.
\newblock In Kamalika Chaudhuri and Ruslan Salakhutdinov, editors,
  \emph{Proceedings of the 36th International Conference on Machine Learning},
  volume~97 of \emph{Proceedings of Machine Learning Research}, pages 242--252.
  PMLR, 09--15 Jun 2019{\natexlab{b}}.
\newblock URL \url{https://proceedings.mlr.press/v97/allen-zhu19a.html}.

\bibitem[Apidopoulos et~al.(2022)Apidopoulos, Ginatta, and
  Villa]{apodopoulos2022convergence}
Vassilis Apidopoulos, Nicolò Ginatta, and Silvia Villa.
\newblock Convergence rates for the heavy-ball continuous dynamics for
  non-convex optimization, under polyak–Łojasiewicz condition.
\newblock \emph{Journal of Global Optimization}, 84, 05 2022.
\newblock \doi{10.1007/s10898-022-01164-w}.

\bibitem[Auer et~al.(1995)Auer, Herbster, and Warmuth]{auer1995exponentially}
Peter Auer, Mark Herbster, and Manfred K.~K Warmuth.
\newblock Exponentially many local minima for single neurons.
\newblock In D.~Touretzky, M.C. Mozer, and M.~Hasselmo, editors, \emph{Advances
  in Neural Information Processing Systems}, volume~8. MIT Press, 1995.
\newblock URL
  \url{https://proceedings.neurips.cc/paper_files/paper/1995/file/3806734b256c27e41ec2c6bffa26d9e7-Paper.pdf}.

\bibitem[Awasthi et~al.(2021)Awasthi, Das, and Gollapudi]{awasthi2021a}
Pranjal Awasthi, Abhimanyu Das, and Sreenivas Gollapudi.
\newblock A convergence analysis of gradient descent on graph neural networks.
\newblock In M.~Ranzato, A.~Beygelzimer, Y.~Dauphin, P.S. Liang, and J.~Wortman
  Vaughan, editors, \emph{Advances in Neural Information Processing Systems},
  volume~34, pages 20385--20397. Curran Associates, Inc., 2021.
\newblock URL
  \url{https://proceedings.neurips.cc/paper_files/paper/2021/file/aaf2979785deb27864047e0ea40ef1b7-Paper.pdf}.

\bibitem[Banerjee et~al.(2023)Banerjee, Cisneros-Velarde, Zhu, and
  Belkin]{banerjee2023restricted}
Arindam Banerjee, Pedro Cisneros-Velarde, Libin Zhu, and Misha Belkin.
\newblock Restricted strong convexity of deep learning models with smooth
  activations.
\newblock In \emph{The Eleventh International Conference on Learning
  Representations}, 2023.
\newblock URL \url{https://openreview.net/forum?id=PINRbk7h01}.

\bibitem[Bansal and Gupta(2019)]{bansal2019potential}
Nikhil Bansal and Anupam Gupta.
\newblock Potential-function proofs for gradient methods.
\newblock \emph{Theory of Computing}, 15\penalty0 (4):\penalty0 1--32, 2019.
\newblock \doi{10.4086/toc.2019.v015a004}.
\newblock URL \url{https://theoryofcomputing.org/articles/v015a004}.

\bibitem[Carmon and Duchi(2020)]{carmon2020first}
Yair Carmon and John~C. Duchi.
\newblock First-order methods for nonconvex quadratic minimization.
\newblock \emph{{SIAM} Review}, 62\penalty0 (2):\penalty0 395--436, jan 2020.
\newblock \doi{10.1137/20m1321759}.
\newblock URL \url{https://doi.org/10.1137\%2F20m1321759}.

\bibitem[Carmon et~al.(2017)Carmon, Duchi, Hinder, and
  Sidford]{carmon2017accelerated}
Yair Carmon, John~C. Duchi, Oliver Hinder, and Aaron Sidford.
\newblock Accelerated methods for non-convex optimization, 2017.

\bibitem[Dahl et~al.(2023)Dahl, Schneider, Nado, Agarwal, Sastry, Hennig,
  Medapati, Eschenhagen, Kasimbeg, Suo, Bae, Gilmer, Peirson, Khan, Anil,
  Rabbat, Krishnan, Snider, Amid, Chen, Maddison, Vasudev, Badura, Garg, and
  Mattson]{dahl2023benchmarking}
George~E. Dahl, Frank Schneider, Zachary Nado, Naman Agarwal,
  Chandramouli~Shama Sastry, Philipp Hennig, Sourabh Medapati, Runa
  Eschenhagen, Priya Kasimbeg, Daniel Suo, Juhan Bae, Justin Gilmer, Abel~L.
  Peirson, Bilal Khan, Rohan Anil, Mike Rabbat, Shankar Krishnan, Daniel
  Snider, Ehsan Amid, Kongtao Chen, Chris~J. Maddison, Rakshith Vasudev, Michal
  Badura, Ankush Garg, and Peter Mattson.
\newblock Benchmarking neural network training algorithms, 2023.

\bibitem[Danilova et~al.(2018)Danilova, Kulakova, and
  Polyak]{danilova2018nonmonotone}
Marina Danilova, Anastasiya Kulakova, and Boris Polyak.
\newblock Non-monotone behavior of the heavy ball method, 2018.

\bibitem[Du et~al.(2019{\natexlab{a}})Du, Lee, Li, Wang, and
  Zhai]{du2019gradient}
Simon Du, Jason Lee, Haochuan Li, Liwei Wang, and Xiyu Zhai.
\newblock Gradient descent finds global minima of deep neural networks.
\newblock In Kamalika Chaudhuri and Ruslan Salakhutdinov, editors,
  \emph{Proceedings of the 36th International Conference on Machine Learning},
  volume~97 of \emph{Proceedings of Machine Learning Research}, pages
  1675--1685. PMLR, 09--15 Jun 2019{\natexlab{a}}.
\newblock URL \url{https://proceedings.mlr.press/v97/du19c.html}.

\bibitem[Du et~al.(2019{\natexlab{b}})Du, Zhai, Poczos, and
  Singh]{du2018gradient}
Simon~S. Du, Xiyu Zhai, Barnabas Poczos, and Aarti Singh.
\newblock Gradient descent provably optimizes over-parameterized neural
  networks.
\newblock In \emph{International Conference on Learning Representations},
  2019{\natexlab{b}}.
\newblock URL \url{https://openreview.net/forum?id=S1eK3i09YQ}.

\bibitem[d’Aspremont et~al.(2021)d’Aspremont, Scieur, and
  Taylor]{acceleration2021}
Alexandre d’Aspremont, Damien Scieur, and Adrien Taylor.
\newblock Acceleration methods.
\newblock \emph{Foundations and Trends® in Optimization}, 5\penalty0
  (1-2):\penalty0 1--245, 2021.
\newblock ISSN 2167-3888.
\newblock \doi{10.1561/2400000036}.
\newblock URL \url{http://dx.doi.org/10.1561/2400000036}.

\bibitem[Goodfellow et~al.(2016)Goodfellow, Bengio, and
  Courville]{Goodfellow-et-al-2016}
Ian Goodfellow, Yoshua Bengio, and Aaron Courville.
\newblock \emph{Deep Learning}.
\newblock MIT Press, 2016.
\newblock \url{http://www.deeplearningbook.org}.

\bibitem[Goujaud et~al.(2023)Goujaud, Taylor, and
  Dieuleveut]{goujaud2023provable}
Baptiste Goujaud, Adrien Taylor, and Aymeric Dieuleveut.
\newblock Provable non-accelerations of the heavy-ball method, 2023.

\bibitem[Huang et~al.(2021)Huang, Li, Song, and Yang]{huang2021flntk}
Baihe Huang, Xiaoxiao Li, Zhao Song, and Xin Yang.
\newblock Fl-ntk: A neural tangent kernel-based framework for federated
  learning analysis.
\newblock In \emph{International Conference on Machine Learning}, 2021.

\bibitem[Jacot et~al.(2020)Jacot, Gabriel, and Hongler]{jacot2020neural}
Arthur Jacot, Franck Gabriel, and Clément Hongler.
\newblock Neural tangent kernel: Convergence and generalization in neural
  networks, 2020.

\bibitem[Jelassi and Li(2022)]{jelassi2022understanding}
Samy Jelassi and Yuanzhi Li.
\newblock Towards understanding how momentum improves generalization in deep
  learning, 2022.

\bibitem[Ji and Telgarsky(2020)]{ji2020polylogarithmic}
Ziwei Ji and Matus Telgarsky.
\newblock Polylogarithmic width suffices for gradient descent to achieve
  arbitrarily small test error with shallow relu networks.
\newblock In \emph{International Conference on Learning Representations}, 2020.
\newblock URL \url{https://openreview.net/forum?id=HygegyrYwH}.

\bibitem[Karimi et~al.(2020)Karimi, Nutini, and Schmidt]{karimi2020linear}
Hamed Karimi, Julie Nutini, and Mark Schmidt.
\newblock Linear convergence of gradient and proximal-gradient methods under
  the polyak-\l{}ojasiewicz condition, 2020.

\bibitem[Lecun et~al.(1998)Lecun, Bottou, Bengio, and
  Haffner]{lecun1998gradient}
Yann Lecun, Leon Bottou, Y.~Bengio, and Patrick Haffner.
\newblock Gradient-based learning applied to document recognition.
\newblock \emph{Proceedings of the IEEE}, 86:\penalty0 2278 -- 2324, 12 1998.
\newblock \doi{10.1109/5.726791}.

\bibitem[LeCun et~al.(2015)LeCun, Bengio, and Hinton]{lecun2015deep}
Yann LeCun, Y.~Bengio, and Geoffrey Hinton.
\newblock Deep learning.
\newblock \emph{Nature}, 521:\penalty0 436--44, 05 2015.
\newblock \doi{10.1038/nature14539}.

\bibitem[Li and Lin(2022)]{li2022restarted}
Huan Li and Zhouchen Lin.
\newblock Restarted nonconvex accelerated gradient descent: No more
  polylogarithmic factor in the $o(\epsilon^{-7/4})$ complexity.
\newblock In Kamalika Chaudhuri, Stefanie Jegelka, Le~Song, Csaba Szepesvari,
  Gang Niu, and Sivan Sabato, editors, \emph{Proceedings of the 39th
  International Conference on Machine Learning}, volume 162 of
  \emph{Proceedings of Machine Learning Research}, pages 12901--12916. PMLR,
  17--23 Jul 2022.
\newblock URL \url{https://proceedings.mlr.press/v162/li22o.html}.

\bibitem[Li et~al.(2022)Li, Song, and Yang]{li2022federated}
Xiaoxiao Li, Zhao Song, and Jiaming Yang.
\newblock Federated adversarial learning: A framework with convergence
  analysis, 2022.

\bibitem[Li et~al.(2020)Li, Ma, and Zhang]{li2020learning}
Yuanzhi Li, Tengyu Ma, and Hongyang~R. Zhang.
\newblock Learning over-parametrized two-layer relu neural networks beyond ntk,
  2020.

\bibitem[Liao and Kyrillidis(2022)]{liao2022on}
Fangshuo Liao and Anastasios Kyrillidis.
\newblock On the convergence of shallow neural network training with randomly
  masked neurons.
\newblock \emph{Transactions on Machine Learning Research}, 2022.
\newblock ISSN 2835-8856.
\newblock URL \url{https://openreview.net/forum?id=e7mYYMSyZH}.

\bibitem[Ling et~al.(2023)Ling, Xie, Wang, Zhang, and Lin]{ling2023global}
Zenan Ling, Xingyu Xie, Qiuhao Wang, Zongpeng Zhang, and Zhouchen Lin.
\newblock Global convergence of over-parameterized deep equilibrium models,
  2023.

\bibitem[Liu et~al.(2020)Liu, Zhu, and Belkin]{liu2020Loss}
Chaoyue Liu, Libin Zhu, and Mikhail Belkin.
\newblock Loss landscapes and optimization in over-parameterized non-linear
  systems and neural networks.
\newblock 2020.

\bibitem[Liu et~al.(2022{\natexlab{a}})Liu, Pan, and Tao]{liu2022provable}
Xin Liu, Zhisong Pan, and Wei Tao.
\newblock Provable convergence of nesterov’s accelerated gradient method for
  over-parameterized neural networks.
\newblock \emph{Knowledge-Based Systems}, 251:\penalty0 109277,
  2022{\natexlab{a}}.
\newblock ISSN 0950-7051.
\newblock \doi{https://doi.org/10.1016/j.knosys.2022.109277}.
\newblock URL
  \url{https://www.sciencedirect.com/science/article/pii/S0950705122006402}.

\bibitem[Liu et~al.(2022{\natexlab{b}})Liu, Tao, and Pan]{liu2022convergence}
Xin Liu, Wei Tao, and Zhisong Pan.
\newblock A convergence analysis of nesterov’s accelerated gradient method in
  training deep linear neural networks.
\newblock \emph{Information Sciences}, 612:\penalty0 898--925,
  2022{\natexlab{b}}.
\newblock ISSN 0020-0255.
\newblock \doi{https://doi.org/10.1016/j.ins.2022.08.090}.
\newblock URL
  \url{https://www.sciencedirect.com/science/article/pii/S0020025522010003}.

\bibitem[Mianjy and Arora(2020)]{mianjy2020on}
Poorya Mianjy and Raman Arora.
\newblock On convergence and generalization of dropout training.
\newblock In H.~Larochelle, M.~Ranzato, R.~Hadsell, M.F. Balcan, and H.~Lin,
  editors, \emph{Advances in Neural Information Processing Systems}, volume~33,
  pages 21151--21161. Curran Associates, Inc., 2020.
\newblock URL
  \url{https://proceedings.neurips.cc/paper_files/paper/2020/file/f1de5100906f31712aaa5166689bfdf4-Paper.pdf}.

\bibitem[Nesterov(2018)]{nesterov2018lectures}
Yurii Nesterov.
\newblock \emph{Lectures on Convex Optimization}.
\newblock Springer Publishing Company, Incorporated, 2nd edition, 2018.
\newblock ISBN 3319915770.

\bibitem[Nguyen(2021)]{nguyen2021ontheproof}
Quynh Nguyen.
\newblock On the proof of global convergence of gradient descent for deep relu
  networks with linear widths.
\newblock \emph{CoRR}, abs/2101.09612, 2021.
\newblock URL \url{https://arxiv.org/abs/2101.09612}.

\bibitem[Nguyen and Mondelli(2020)]{nguyen2020global}
Quynh Nguyen and Marco Mondelli.
\newblock Global convergence of deep networks with one wide layer followed by
  pyramidal topology, 2020.

\bibitem[Oymak and Soltanolkotabi(2019)]{oymak2019moderate}
Samet Oymak and Mahdi Soltanolkotabi.
\newblock Towards moderate overparameterization: global convergence guarantees
  for training shallow neural networks, 2019.

\bibitem[Safran and Shamir(2018)]{safran2018spurious}
Itay Safran and Ohad Shamir.
\newblock Spurious local minima are common in two-layer relu neural networks,
  2018.

\bibitem[Shi et~al.(2022)Shi, Du, Jordan, and Su]{shi2022understanding}
Bin Shi, Simon Du, Michael Jordan, and Weijie Su.
\newblock Understanding the acceleration phenomenon via high-resolution
  differential equations.
\newblock \emph{Mathematical Programming}, 195:\penalty0 79--148, 01 2022.
\newblock \doi{10.1007/s10107-021-01681-8}.

\bibitem[Song et~al.(2021)Song, Ramezani-Kebrya, Pethick, Eftekhari, and
  Cevher]{song2021subquadratic}
Chaehwan Song, Ali Ramezani-Kebrya, Thomas Pethick, Armin Eftekhari, and Volkan
  Cevher.
\newblock Subquadratic overparameterization for shallow neural networks.
\newblock In A.~Beygelzimer, Y.~Dauphin, P.~Liang, and J.~Wortman Vaughan,
  editors, \emph{Advances in Neural Information Processing Systems}, 2021.
\newblock URL \url{https://openreview.net/forum?id=NhbFhfM960}.

\bibitem[Song and Yang(2020)]{song2020quadratic}
Zhao Song and Xin Yang.
\newblock Quadratic suffices for over-parametrization via matrix chernoff
  bound, 2020.

\bibitem[Su and Yang(2019)]{su2019learning}
Lili Su and Pengkun Yang.
\newblock On learning over-parameterized neural networks: A functional
  approximation perspective, 2019.

\bibitem[Sutskever et~al.(2013)Sutskever, Martens, Dahl, and
  Hinton]{sutskever2013}
Ilya Sutskever, James Martens, George Dahl, and Geoffrey Hinton.
\newblock On the importance of initialization and momentum in deep learning.
\newblock In Sanjoy Dasgupta and David McAllester, editors, \emph{Proceedings
  of the 30th International Conference on Machine Learning}, volume~28 of
  \emph{Proceedings of Machine Learning Research}, pages 1139--1147, Atlanta,
  Georgia, USA, 17--19 Jun 2013. PMLR.
\newblock URL \url{https://proceedings.mlr.press/v28/sutskever13.html}.

\bibitem[Vershynin(2018)]{vershynin_2018}
Roman Vershynin.
\newblock \emph{High-Dimensional Probability: An Introduction with Applications
  in Data Science}.
\newblock Cambridge Series in Statistical and Probabilistic Mathematics.
  Cambridge University Press, 2018.
\newblock \doi{10.1017/9781108231596}.

\bibitem[Wang et~al.(2021)Wang, Lin, and Abernethy]{wang2021modular}
Jun-Kun Wang, Chi-Heng Lin, and Jacob~D Abernethy.
\newblock A modular analysis of provable acceleration via polyak’s momentum:
  Training a wide relu network and a deep linear network.
\newblock In Marina Meila and Tong Zhang, editors, \emph{Proceedings of the
  38th International Conference on Machine Learning}, volume 139 of
  \emph{Proceedings of Machine Learning Research}, pages 10816--10827. PMLR,
  18--24 Jul 2021.
\newblock URL \url{https://proceedings.mlr.press/v139/wang21n.html}.

\bibitem[Wang et~al.(2022)Wang, Lin, Wibisono, and Hu]{wang2022provable}
Jun-Kun Wang, Chi-Heng Lin, Andre Wibisono, and Bin Hu.
\newblock Provable acceleration of heavy ball beyond quadratics for a class of
  polyak-lojasiewicz functions when the non-convexity is averaged-out.
\newblock In Kamalika Chaudhuri, Stefanie Jegelka, Le~Song, Csaba Szepesvari,
  Gang Niu, and Sivan Sabato, editors, \emph{Proceedings of the 39th
  International Conference on Machine Learning}, volume 162 of
  \emph{Proceedings of Machine Learning Research}, pages 22839--22864. PMLR,
  17--23 Jul 2022.
\newblock URL \url{https://proceedings.mlr.press/v162/wang22p.html}.

\bibitem[Woodworth et~al.(2020)Woodworth, Gunasekar, Lee, Moroshko, Savarese,
  Golan, Soudry, and Srebro]{woodworth2020kernel}
Blake Woodworth, Suriya Gunasekar, Jason~D. Lee, Edward Moroshko, Pedro
  Savarese, Itay Golan, Daniel Soudry, and Nathan Srebro.
\newblock Kernel and rich regimes in overparametrized models, 2020.

\bibitem[Wu et~al.(2023)Wu, Kungurtsev, and Mondelli]{wu2023meanfield}
Diyuan Wu, Vyacheslav Kungurtsev, and Marco Mondelli.
\newblock Mean-field analysis for heavy ball methods: Dropout-stability,
  connectivity, and global convergence, 2023.

\bibitem[Xu and Zhu(2021)]{xu2021onepass}
Jiaming Xu and Hanjing Zhu.
\newblock One-pass stochastic gradient descent in overparametrized two-layer
  neural networks, 2021.

\bibitem[Xu and Du(2023)]{xu2023overparameterization}
Weihang Xu and Simon~S. Du.
\newblock Over-parameterization exponentially slows down gradient descent for
  learning a single neuron, 2023.

\bibitem[Yang and Hu(2021)]{yang2021tensor}
Greg Yang and Edward~J. Hu.
\newblock Tensor programs iv: Feature learning in infinite-width neural
  networks.
\newblock In Marina Meila and Tong Zhang, editors, \emph{Proceedings of the
  38th International Conference on Machine Learning}, volume 139 of
  \emph{Proceedings of Machine Learning Research}, pages 11727--11737. PMLR,
  18--24 Jul 2021.
\newblock URL \url{https://proceedings.mlr.press/v139/yang21c.html}.

\bibitem[Yehudai and Shamir(2022)]{yehudai2022power}
Gilad Yehudai and Ohad Shamir.
\newblock On the power and limitations of random features for understanding
  neural networks, 2022.

\bibitem[Yue et~al.(2022)Yue, Fang, and Lin]{yue2022lower}
Pengyun Yue, Cong Fang, and Zhouchen Lin.
\newblock On the lower bound of minimizing polyak-{\l}ojasiewicz functions,
  2022.

\bibitem[Yun et~al.(2019)Yun, Sra, and Jadbabaie]{yun2018small}
Chulhee Yun, Suvrit Sra, and Ali Jadbabaie.
\newblock Small nonlinearities in activation functions create bad local minima
  in neural networks.
\newblock In \emph{International Conference on Learning Representations}, 2019.
\newblock URL \url{https://openreview.net/forum?id=rke_YiRct7}.

\bibitem[Zhang et~al.(2017)Zhang, Bengio, Hardt, Recht, and
  Vinyals]{zhang2017understanding}
Chiyuan Zhang, Samy Bengio, Moritz Hardt, Benjamin Recht, and Oriol Vinyals.
\newblock Understanding deep learning requires rethinking generalization, 2017.

\bibitem[Zhang et~al.(2018)Zhang, Yu, Wang, and Gu]{zhang2018learning}
Xiao Zhang, Yaodong Yu, Lingxiao Wang, and Quanquan Gu.
\newblock Learning one-hidden-layer relu networks via gradient descent, 2018.

\bibitem[Zhang et~al.(2019)Zhang, Yu, Wang, and Gu]{zhao2019learning}
Xiao Zhang, Yaodong Yu, Lingxiao Wang, and Quanquan Gu.
\newblock Learning one-hidden-layer relu networks via gradient descent.
\newblock In Kamalika Chaudhuri and Masashi Sugiyama, editors,
  \emph{Proceedings of the Twenty-Second International Conference on Artificial
  Intelligence and Statistics}, volume~89 of \emph{Proceedings of Machine
  Learning Research}, pages 1524--1534. PMLR, 16--18 Apr 2019.
\newblock URL \url{https://proceedings.mlr.press/v89/zhang19g.html}.

\bibitem[Zou and Gu(2019)]{zou2019an}
Difan Zou and Quanquan Gu.
\newblock An improved analysis of training over-parameterized deep neural
  networks.
\newblock In H.~Wallach, H.~Larochelle, A.~Beygelzimer, F.~d\textquotesingle
  Alch\'{e}-Buc, E.~Fox, and R.~Garnett, editors, \emph{Advances in Neural
  Information Processing Systems}, volume~32. Curran Associates, Inc., 2019.
\newblock URL
  \url{https://proceedings.neurips.cc/paper_files/paper/2019/file/6a61d423d02a1c56250dc23ae7ff12f3-Paper.pdf}.

\bibitem[Zou et~al.(2018)Zou, Cao, Zhou, and Gu]{zou2018stochastic}
Difan Zou, Yuan Cao, Dongruo Zhou, and Quanquan Gu.
\newblock Stochastic gradient descent optimizes over-parameterized deep relu
  networks, 2018.

\end{thebibliography}
\clearpage

\appendix
\tableofcontents

\section{Proofs for Section \ref{sec:preliminary}}
\label{sec:proof_preliminary}
\subsection{Proof of Theorem \ref{theo:strong_cvx_smooth_suffice}}
Let $\tilde{f}:\R^d\rightarrow\R$ be a $\tilde{L}$-smooth and $\tilde{\mu}$-strongly convex function. Let $\{\bm{0}\}$ be the zero-dimensional vector space. We define $h(\bfx, \bfu):\R^d\times \{\bm{0}\}\rightarrow\R^d$ as $h(\bfx,\bfu) = \bfx$ for all $\bfx\in\R^d$. Moreover, we define $g:\R\rightarrow\R$ as $g(s) = \tilde{f}(\bfs)$ for all $s\in\R$, and $f(\bfx,\bfu) = g(h(\bfx,\bfu))$. Then $f(\bfx,\bfu)$ is a partitioned equivalence of $\tilde{f}(\bfx)$ since
\[
    f(\bfx,\bfu) = g(\bfx) = \tilde{f}(\bfx)\quad\forall \bfx\in\R^d
\]
Choosing $R_{\bfx} = R_{\bfu} = \infty$, we have that $\mathcal{B}^{(1)}_{R_{\bfx}} = \R^d$ and $\mathcal{B}^{(2)}_{R_{\bfu}} = \{\bm{0}\}$. Therefore, we have
\begin{align*}
    f(\bfy,\bfu) = \tilde{f}(\bfy) & \geq \tilde{f}(\bfx) + \inner{\nabla\tilde{f}(\bfx)}{\bfy-\bfx} + \frac{\tilde{\mu}}{2}\norm{\bfy-\bfx}_2^2\\
    & = f(\bfx,\bfu) + \inner{\nabla_1f(\bfx,\bfu)}{\bfy-\bfx} + \frac{\tilde{\mu}}{2}\norm{\bfy-\bfx}_2^2\\
    f(\bfy,\bfu) = \tilde{f}(\bfy) & \leq \tilde{f}(\bfx) + \inner{\nabla\tilde{f}(\bfx)}{\bfy-\bfx} + \frac{\tilde{L}}{2}\norm{\bfy-\bfx}_2^2\\
    & = f(\bfx,\bfu) + \inner{\nabla_1f(\bfx,\bfu)}{\bfy-\bfx} + \frac{\tilde{L}}{2}\norm{\bfy-\bfx}_2^2
\end{align*}
where the first and second inequality follows from the strong convexity and smoothness of $\tilde{f}$. This shows that $f(\bfx,\bfu)$ satisfies Assumption \ref{asump:strong_cvx}, \ref{asump:f_smooth} with $\mu = \tilde{\mu}$ and $L_1 = \tilde{L}$. Since $g(\bfs) = \tilde{f}(\bfs)$, it must be $L_2$-smooth with $L_2 = \tilde{L}$ as well. Moreover, we must have $g^{\star} = \tilde{f}^{\star} = f^\star$. This shows that Assumption \ref{asump:g_smooth} holds. Also, since for all $\bfu\in\{\bm{0}\}$ we must have $\bfu = \bm{0}$, it holds naturally that
\[
    h(\bfx,\bfu) = h(\bfx,\bfv);\quad\nabla_1f(\bfx,\bfu) = \nabla_1f(\bfx,\bfv)
\]
Therefore, Assumption \ref{asump:h_lip}, \ref{asump:grad_lip} hold with $G_1 = G_2 = 0$. Lastly, since $\tilde{f}$ is strongly convex, there must exist a unique $\bfx^{\star} = \argmin_{\bfx\in\R^d}\tilde{f}(\bfx)$. Therefore, we must have that
\[
    \min_{x\in\R^d}f(\bfx,\bfu) \leq f(\bfx^{\star},\bfu) = \tilde{f}(\bfx^{\star}) = f^\star\quad\forall \bfu\in\{\bm{0}\}
\]
This shows that Assumption \ref{asump:universal_opt} is satisfied.

\section{Proofs for Section \ref{sec:gd_conv}}
\label{sec:proof_gd_conv}
\subsection{Proof of Lemma \ref{lem:PL_condition}}
Fix any $\bfu\in\mathcal{B}^{(2)}_{R_{\bfu}}$ and let $\bfx^{\star} = \argmin_{\bfx\in\R^{d_1}}f(\bfx,\bfu)$. By Assumption \ref{asump:strong_cvx}, we have
\begin{align*}
    f(\bfx^{\star},\bfu) & \geq f(\bfx,\bfu) + \inner{\nabla_1f(\bfx,\bfu)}{\bfx^{\star} -\bfu} + \frac{\mu}{2}\norm{\bfx^{\star} - \bfx}_2^2\\
    & \geq \min_{\bfy\in\R^{d_1}}\paren{\underbrace{f(\bfx,\bfu) + \inner{\nabla_1f(\bfx,\bfu)}{\bfy -\bfu} + \frac{\mu}{2}\norm{\bfy - \bfx}_2^2}_{f_{\bfx,\bfu}\paren{\bfy}}}
\end{align*}
Notice that $\nabla^2f_{\bfx,\bfu}\paren{\bfy} = \mu$. Therefore, $f_{\bfx,\bfu}\paren{\bfy}$ is strongly convex with respect to $\bfy$. Thus, $\bfy^{\star} = \argmin_{\bfy\in\R^{d_1}}f_{\bfx,\bfu}\paren{\bfy}$ must satisfy
\[
    \nabla f_{\bfx,\bfu}\paren{\bfy^{\star}} = \nabla_1f(\bfx,\bfu) + \mu\paren{\bfy^{\star}-\bfx} = 0
\]
which implies that $\bfy^{\star} = \bfx - \frac{1}{\mu}\nabla_1f(\bfx,\bfu)$, and $\min_{\bfy\in\R^{d_1}}f_{\bfx,\bfu}\paren{\bfy} = f(\bfx,\bfu) - \frac{1}{2\mu}\norm{\nabla_1f(\bfx,\bfu)}_2^2$. This implies that
\[
    f(\bfx^{\star},\bfu) \geq f(\bfx,\bfu) - \frac{1}{2\mu}\norm{\nabla_1f(\bfx,\bfu)}_2^2\Rightarrow \norm{\nabla_1f(\bfx,\bfu)}_2^2\geq 2\mu\paren{f(\bfx,\bfu) - f(\bfx^{\star},\bfu)}
\]
By Assumption \ref{asump:universal_opt}, we have that $f(\bfx^{\star},\bfu) = f^\star$. Also, by definition of $\nabla f(\bfx,\bfu)$, we have $\norm{\nabla f(\bfx,\bfu)}_2^2 = \norm{\nabla_1f(\bfx,\bfu)}_2^2 + \norm{\nabla_2f(\bfx,\bfu)}_2^2$. Therefore
\[
    \norm{\nabla f(\bfx,\bfu)}_2^2 \geq \norm{\nabla_1f(\bfx,\bfu)}_2^2 \geq 2\mu\paren{f(\bfx,\bfu) - f^\star}
\]
\subsection{Proof of Lemma \ref{lem:df_du}}
Since Assumption \ref{asump:g_smooth} holds, we can invoke Lemma \ref{lem:g_grad_bound} to get that
\begin{align*}
    \norm{\nabla g(\bfs)}_2^2 \leq 2L_2\paren{g(\bfs) - f^\star}
\end{align*}
Moreover, Assumption \ref{asump:g_smooth} also implies that
\[
    f(\bfx,\bfu) = g(h(\bfx,\bfu)) \leq f(\bfx,\bfv) + \inner{\nabla g(h(\bfx,\bfv))}{h(\bfx,\bfu)-h(\bfx,\bfv)} + \frac{L_2}{2}\norm{h(\bfx,\bfu)-h(\bfx,\bfv)}_2^2
\]
Assumption \ref{asump:h_lip} implies that
\[
    \norm{h(\bfx,\bfu)-h(\bfx,\bfv)}_2 \leq G_1\norm{\bfu-\bfv}_2
\]
Therefore, applying the triangle inequality to the inner-product term, we have
\begin{align*}
    f(\bfx,\bfu) & \leq f(\bfx,\bfv) + G_1\norm{\nabla g(h(\bfx,\bfv))}_2\cdot\norm{\bfu-\bfv}_2 + \frac{G_1^2L_2}{2}\norm{\bfu-\bfv}_2^2\\
    & \leq f(\bfx,\bfv) + \mathcal{Q}^{-1}\norm{\nabla g(h(\bfx,\bfv))}_2^2 + \frac{G_1^2}{2}\paren{L_2 + \mathcal{Q}}\norm{\bfu-\bfv}_2^2\\
    & \leq f(\bfx,\bfv) + \mathcal{Q}^{-1}\paren{f(\bfx,\bfv) - f^\star} + \frac{G_1^2}{2}\paren{L_2 + \mathcal{Q}}\norm{\bfu-\bfv}_2^2
\end{align*}
This implies that
\[
    f(\bfx,\bfu) - f(\bfx,\bfv) \leq \mathcal{Q}^{-1}\paren{f(\bfx,\bfv) - f^\star} + \frac{G_1^2}{2}\paren{L_2 + \mathcal{Q}}\norm{\bfu-\bfv}_2^2
\]

\subsection{Proof of Theorem \ref{theo:gd_conv}}
We prove the following two conditions by induction.
\begin{cond}
\label{cond:gd_bounded_change}
Let $k\in\N$, then for all $t\leq k$, we have $\norm{\bfx_t - \bfx_0}_2 \leq R_{\bfx}$ and $\norm{\bfu_t - \bfu_0}_2\leq R_{\bfu}$.
\end{cond}
\begin{cond}
\label{cond:gd_linear_conv}
Let $k\in\N$, then for all $t\leq k$, we have
\[
    f(\bfx_{t},\bfu_{t}) - f^\star \leq \paren{1 - \frac{c}{\kappa}}^t\paren{f(\bfx_0,\bfu_0)-f^\star}
\]
\end{cond}
\subsubsection{Base Case: $k=0$}
When $k=0$, Condition \ref{cond:gd_bounded_change} reads $\norm{\bfx_t - \bfx_0}_2 \leq R_{\bfx}$ and $\norm{\bfu_t - \bfu_0}_2\leq R_{\bfu}$ for $t =0$. This is automatically true since $\norm{\bfx_0 - \bfx_0}_2 = \norm{\bfu_0 - \bfu_0}_2 = 0$. Condition \ref{cond:gd_linear_conv} is also true since
\[
    f(\bfx_{0},\bfu_{0}) - f^\star \leq \paren{1 - \frac{c}{\kappa}}^0\paren{f(\bfx_0,\bfu_0)-f^\star}
\]
\subsubsection{Inductive Step 1: Condition \ref{cond:gd_bounded_change} $\Rightarrow$ Condition \ref{cond:gd_linear_conv}}
Suppose that Condition \ref{cond:gd_bounded_change} and Condition \ref{cond:gd_linear_conv} holds for all $t\leq k$. We show that Condition \ref{cond:gd_linear_conv} holds for all $t \leq k+1$.
Since $\norm{\bfx_t -\bfx_0}_2\leq R_{\bfx}$ and $\norm{\bfu_t -\bfu_0}_2\leq R_{\bfu}$, we have that $\bfx_t \in\mathcal{B}^{(1)}_{R_{\bfx}}$ and $\bfu_t\in\mathcal{B}^{(2)}_{R_{\bfu}}$ for all $t\leq k$. Therefore, Assumption \ref{asump:f_smooth} holds for all $\bfx\in\R^{d_1}$ and $\bfu = \bfu_k$, which implies that
\begin{equation}
    \label{eq:theo1.1}
    \begin{aligned}
        f(\bfx_{k+1},\bfu_k) & \leq f(\bfx_k,\bfu_k) + \inner{\nabla_1f(\bfx_k,\bfu_k)}{\bfx_{k+1} - \bfu_k} + \frac{L_1}{2}\norm{\bfx_{k+1} - \bfx_k}_2^2\\
        & = f(\bfx_k,\bfu_k) - \eta\norm{\nabla_1f(\bfx_k,\bfu_k)}_2^2 + \frac{L_2}{2}\eta^2\norm{\nabla_1f(\bfx_k,\bfu_k)}_2^2\\
        & = f(\bfx_k,\bfu_k) - \eta\paren{1 - \frac{\eta L_1}{2}}\norm{\nabla_1f(\bfx_k,\bfu_k)}_2^2
    \end{aligned}
\end{equation}
Since Assumption \ref{asump:strong_cvx}, \ref{asump:universal_opt} holds, we can apply Lemma \ref{lem:PL_condition} to (\ref{eq:theo1.1}) and choose $\eta = \frac{1}{L_1}$ to get that
\begin{equation}
    \label{eq:theo1.2}
    f(\bfx_{k+1},\bfu_k) - f^\star \leq \paren{1 - \frac{1}{\kappa}}\paren{f(\bfx_k,\bfu_k) - f^\star}
\end{equation}
Moreover, since $\bfu_{k+1} = \bfu_k - \eta\nabla_2f(\bfx_k,\bfu)$, we have
\begin{align*}
    \norm{\bfu_{k+1} - \bfu_k}_2^2 = \eta\norm{\nabla_2f(\bfx_k,\bfu)}_2^2\leq \eta 2G_1^2L_2\paren{f(\bfx_k,\bfu_k) - f^\star}
\end{align*}
where the last inequality follows from Lemma \ref{lem:du_grad_bound}.
Since Assumption \ref{asump:g_smooth}, \ref{asump:h_lip} holds, we can use Lemma \ref{lem:df_du} with $\mathcal{Q} = 2(\kappa - 1)L_2$ to get that
\begin{equation}
    \label{eq:theo1.3}
    \begin{aligned}
        f(\bfx_{k+1},\bfu_{k+1}) - f(\bfx_{k+1},\bfu_k) & \leq \mathcal{Q}^{-1}L_2\paren{f(\bfx_{k+1},\bfu_k) - f^\star}\\
        & \quad\quad\quad + \frac{G_1^2}{2}\paren{L_2 + \mathcal{Q}}\norm{\bfu_{k+1}-\bfu_k}_2^2\\
        & \leq \mathcal{Q}^{-1}L_2\paren{f(\bfx_{k+1},\bfu_k) - f^\star}\\
        & \quad\quad\quad + \eta^2G_1^4L_2\paren{L_2 + \mathcal{Q}}\paren{f(\bfx_k,\bfu_k) - f^\star}\\
        & \leq \frac{1}{2(\kappa-1)}\paren{f(\bfx_{k+1},\bfu_k) - f^\star}\\
        & \quad\quad\quad + \eta^2G_1^4L_2^2\paren{2\kappa - 1}\paren{f(\bfx_k,\bfu_k) - f^\star}
    \end{aligned}
\end{equation}
Combining (\ref{eq:theo1.1}) with (\ref{eq:theo1.2}), we have
\begin{align*}
    f(\bfx_{k+1},\bfu_{k+1}) -f^\star & = f(\bfx_{k+1},\bfu_k) - f^\star + f(\bfx_{k+1},\bfu_{k+1}) - f(\bfx_{k+1},\bfu_k)\\
    & \leq \paren{1 + \frac{1}{2(\kappa-1)}}\paren{f(\bfx_{k+1},\bfu_k) - f^\star}\\
    & \quad\quad\quad + \eta^2G_1^4L_2^2\paren{2\kappa - 1}\paren{f(\bfx_k,\bfu_k) - f^\star}\\
    & \leq \paren{\frac{2\kappa - 1}{2(\kappa - 1)}\paren{1 - \frac{1}{\kappa}} + \eta^2G_1^4L_2^2\paren{2\kappa - 1}}\cdot\paren{f(\bfx_k,\bfu_k) - f^\star}\\
    & \leq \paren{1 - \frac{1}{2\kappa} + \frac{2G_1^4L_2^2\kappa}{L_1^2}}\cdot\paren{f(\bfx_k,\bfu_k) - f^\star}
\end{align*}
As long as $G_1^4 \leq \frac{\mu^2}{8L_2^2}$, we can guarantee that
\[
    1 - \frac{1}{2\kappa} + \frac{G_1^4L_2^2}{L_1^2}\paren{2\kappa - 1} \leq 1 - \frac{1}{2\kappa} + \frac{\mu^2\kappa}{4L_1^2} \leq 1 - \frac{1}{4\kappa}
\]
Therefore, we have
\[
    f(\bfx_{k+1},\bfu_{k+1}) -f^\star \leq \paren{1 - \frac{1}{4\kappa}}\paren{f(\bfx_k,\bfu_k) - f^\star}
\]
which implies Condition \ref{cond:gd_linear_conv} for all $t\leq k+1$.
\subsubsection{Inductive Step 1: Condition \ref{cond:gd_linear_conv} $\Rightarrow$ Condition \ref{cond:gd_bounded_change}}
Assume that Condition \ref{cond:gd_linear_conv} holds for all $t\leq k$. We prove that Condition \ref{cond:gd_bounded_change} holds for $k+1$. Notice that the change of $\bfx$ and $\bfu$ from initialization can be bounded as
\begin{equation}
    \label{eq:theo1.4}
    \begin{aligned}
        \norm{\bfx_{k+1} - \bfx_0}_2 \leq \sum_{t=0}^{k}\norm{\bfx_{t+1}-\bfx_{t}}_2 = \eta\sum_{t=0}^{k}\norm{\nabla_1f(\bfx_t,\bfu_t)}_2\\
        \norm{\bfu_{k+1} - \bfu_0}_2 \leq \sum_{t=0}^{k}\norm{\bfu_{t+1}-\bfu_{t}}_2 = \eta\sum_{t=0}^{k}\norm{\nabla_2f(\bfx_t,\bfu_t)}_2
    \end{aligned}
\end{equation}
Since Assumption \ref{asump:f_smooth}, \ref{asump:g_smooth}, \ref{asump:h_lip} holds, we can invoke Lemma \ref{lem:smooth_grad_bound}, \ref{lem:du_grad_bound} to have that
\[
    \norm{\nabla_1f(\bfx_t,\bfu_t)}_2^2 \leq 2L_1\paren{f(\bfx_t,\bfu_t) - f^\star};\quad \norm{\nabla_2f(\bfx_t,\bfu_t)}_2^2\leq 2G_1^2L_2\paren{f(\bfx_t,\bfu_t) - f^\star}
\]
Therefore, (\ref{eq:theo1.4}) boils down to
\begin{equation}
    \label{eq:theo1.5}
    \begin{aligned}
        \norm{\bfx_{k+1} - \bfx_0}_2 & \leq \eta\sqrt{2L_1}\sum_{t=0}^{k}\paren{f(\bfx_t,\bfu_t) - f^\star}^{\frac{1}{2}}\\
        \norm{\bfu_{k+1} - \bfu_0}_2 & \leq \eta G_1\sqrt{2L_2}\sum_{t=0}^{k}\paren{f(\bfx_t,\bfu_t) - f^\star}^{\frac{1}{2}}
    \end{aligned}
\end{equation}
Moreover, Condition \ref{cond:gd_linear_conv} implies that
\[
    \paren{f(\bfx_t,\bfu_t) - f^\star}^{\frac{1}{2}} \leq \paren{1 - \frac{1}{4\kappa}}^{\frac{t}{2}}\paren{f(\bfx_0,\bfu_0) - f^\star}^{\frac{1}{2}} \leq \paren{1 - \frac{1}{8\kappa}}^{t}\paren{f(\bfx_0,\bfu_0) - f^\star}^{\frac{1}{2}} 
\]
Plugging this bound into (\ref{eq:theo1.5}) gives
\begin{equation}
    \begin{aligned}
        \norm{\bfx_{k+1} - \bfx_0}_2 & \leq \eta\sqrt{2L_1}\paren{f(\bfx_0,\bfu_0) - f^\star}^{\frac{1}{2}}\sum_{t=0}^{k}\paren{1 - \frac{1}{8\kappa}}^{t}\\
        & \leq 16\eta\kappa\sqrt{L_1}\paren{f(\bfx_0,\bfu_0) - f^\star}^{\frac{1}{2}}\\
        \norm{\bfu_{k+1} - \bfu_0}_2 & \leq \eta G_1\sqrt{2L_2}\paren{f(\bfx_0,\bfu_0) - f^\star}^{\frac{1}{2}}\sum_{t=0}^{k}\paren{1 - \frac{1}{8\kappa}}^{t}\\
        & \leq 16\eta\kappa G_1\sqrt{L_2}\paren{f(\bfx_0,\bfu_0) - f^\star}^{\frac{1}{2}}
    \end{aligned}
\end{equation}
Therefore, as long as
\[
    R_{\bfx} \geq 16\eta\kappa\sqrt{L_1}\paren{f(\bfx_0,\bfu_0) - f^\star}^{\sfrac{1}{2}};\;R_{\bfu} \geq 16\eta\kappa G_1\sqrt{L_2}\paren{f(\bfx_0,\bfu_0) - f^\star}^{\sfrac{1}{2}}
\]
we can guarantee that $\norm{\bfx_{k+1} - \bfx_0}_2\leq R_{\bfx}$ and $\norm{\bfu_{k+1} - \bfu_0}_2\leq R_{\bfu}$. Combining the two inductive steps and the base case completes the proof.
\section{Proofs for Section \ref{sec:difficulty}}
\label{sec:proof_difficulty}
\subsection{Proof of Lemma \ref{lem:minimum_moving}}
By Assumption \ref{asump:strong_cvx}, we have that for all $\bfx,\bfy\in\R^d$ and $\bfv\in\mathcal{B}^{(2)}_{R_{\bfu}}$, it holds that
\begin{align*}
    f(\bfy,\bfu)\geq f(\bfx,\bfu) + \inner{\nabla_1f(\bfx,\bfu)}{\bfy-\bfx} + \frac{\mu}{2}\norm{\bfy-\bfx}_2^2\\
    f(\bfx,\bfu)\geq f(\bfy,\bfu) + \inner{\nabla_1f(\bfy,\bfu)}{\bfx-\bfy} + \frac{\mu}{2}\norm{\bfx-\bfy}_2^2
\end{align*}
Summing the two inequality gives
\[
    \inner{\nabla_1f(\bfy,\bfu) - \nabla_1f(\bfx,\bfu)}{\bfy-\bfx} \geq \mu\norm{\bfy-\bfx}_2^2
\]
Applying Cauchy-Schwarz inequality to the left-hand side gives
\begin{equation}
    \label{eq:lem3.1}
    \norm{\nabla_1f(\bfy,\bfu) - \nabla_1f(\bfx,\bfu)}\geq \mu\norm{\bfy-\bfx}_2
\end{equation}
Let $\bfu,\bfv\in\mathcal{B}^{(2)}_{R_{\bfu}}$ be given. Then (\ref{eq:lem3.1}) implies that
\begin{equation}
    \label{eq:lem3.2}
    \norm{\bfx^{\star}(\bfu) - \bfx^{\star}(\bfv)}_2 \leq \frac{1}{\mu}\norm{\nabla_1f(\bfx^{\star}(\bfu),\bfu) - \nabla_1f(\bfx^{\star}(\bfv),\bfu)}_2
\end{equation}
By the definition of $\bfx^{\star}(\bfu)$ and $\bfx^{\star}(\bfv)$, we have
\[
    \nabla_1f(\bfx^{\star}(\bfu),\bfu) = \bm{0} = \nabla_1f(\bfx^{\star}(\bfv),\bfv)
\]
Therefore, (\ref{eq:lem3.2}) reduces to
\begin{equation*}
    \label{eq:lem3.3}
    \norm{\bfx^{\star}(\bfu) - \bfx^{\star}(\bfv)}_2 \leq \frac{1}{\mu}\norm{\nabla_1f(\bfx^{\star}(\bfv),\bfv) - \nabla_1f(\bfx^{\star}(\bfv),\bfu)}_2 \leq \frac{G_2}{\mu}\norm{\bfu-\bfv}_2
\end{equation*}
where the last inequality follows from Assumption \ref{asump:grad_lip}.

\subsection{Proof of Lemma \ref{lem:grad_dominance}}
Since Assumption \ref{asump:strong_cvx}, \ref{asump:g_smooth}, \ref{asump:h_lip}, \ref{asump:universal_opt} holds, we can invoke Lemma \ref{lem:PL_condition}, \ref{lem:du_grad_bound} to get that
\[
    \norm{\nabla_1f(\bfx,\bfu)}_2^2 \geq 2\mu\paren{f(\bfx,\bfu) - f^\star};\quad \norm{\nabla_2f(\bfx,\bfu)}_2^2\leq 2G_1^2L_2\paren{f(\bfx,\bfu) - f^\star}
\]
for all $\bfx\in\mathcal{B}^{(1)}_{R_{\bfx}}$ and $\bfu\in\mathcal{B}^{(2)}_{R_{\bfu}}$. Combining the two inequality gives
\[
    \norm{\nabla_2f(\bfx,\bfu)}_2^2 \leq \frac{G_1^2L_2}{\mu}\norm{\nabla_1f(\bfx,\bfu)}_2^2
\]
\section{Proofs for Section \ref{sec:nesterov_conv}}
\label{sec:proof_nesterov_conv}
\subsection{Proof of Theorem \ref{theo:nesterov_conv}}
Define $\gamma = \frac{c}{2\sqrt{\kappa} -c}$ such that $1 + \gamma = \paren{1 - \frac{c}{2\sqrt{\kappa}}}^{-1}$. Let $\lambda = (1 + \gamma)^3 - 1$. Denote $\bfx_{-1} = \bfy_{-1} = \bfx_0$ and $\bfu_{-1} = \bfv_{-1} = \bfu_0$. In this proof, we will focus on the following Lyapunov function
\begin{equation}
    \label{eq:lyapnov_exact}
    \phi_k = f(\bfx_k,\bfu_k) - f^\star + \mathcal{Q}_1\norm{\bfz_k - \bfx_{k-1}^{\star}}_2^2 + \frac{\eta}{8}\norm{\nabla_1f(\bfy_{k-1},\bfv_{k-1})}_2^2
\end{equation}
where $\bfx_k^{\star}, \bfz_k$ and $\mathcal{Q}_1$ are defined as 
\[
    \bfx_k^{\star} = \argmin_{\bfx\in\R^{d_1}}f(\bfx,\bfv_k);\;\;\bfz_k = \frac{1 - \beta\lambda}{\beta\lambda}(\bfy_k-\bfx_k) + \bfy_k;\;\;\mathcal{Q}_1 = \frac{\lambda^2}{2\eta(1 + \gamma)^5}
\]
We will prove the following two conditions by induction
\begin{cond}
    \label{cond:phi_conv_cond}
    For all $k\leq \hat{k}$, we have
    \begin{align*}
        \phi_k \leq \paren{1 - \frac{c}{4\sqrt{\kappa}}}^k\phi_0
    \end{align*}
\end{cond}
\begin{cond}
    \label{cond:dx_bound_cond}
    For all $k\leq \hat{k}$, we have that $\bfx_k,\bfy_k\in\mathcal{B}^{(1)}_{\sfrac{R_{\bfx}}{2}}$ and $\bfu_k,\bfv_k\in\mathcal{B}^{(2)}_{\sfrac{R_{\bfu}}{2}}$, and 
    \begin{equation}
        \label{eq:xu_bound_rep}
        \begin{gathered}
            \norm{\bfx_k - \bfx_{k-1}}_2^2\leq \frac{6\eta(L_2 + 1)}{1 - \beta}\paren{1 - \frac{c}{4\sqrt{\kappa}}}^k\phi_0\\
            \norm{\bfu_k - \bfu_{k-1}}_2^2 \leq G_1^2\frac{6\eta L_2(L_2+1)(1+\beta)^3}{\mu\beta(1-\beta)^3}\paren{1 - \frac{c}{4\sqrt{\kappa}}}^k\phi_0
        \end{gathered}
    \end{equation}
    Moreover, we have $\eta\norm{\nabla_1f(\bfy_k,\bfv_k)}_2 \leq \beta\norm{\bfx_k - \bfx_{k-1}}_2 + \norm{\bfx_{k+1} - \bfx_{k}}_2$.
\end{cond}
Notice that $f(\bfx_k,\bfu_k) - f^\star \leq \phi_k$. By Lemma \ref{lem:phi0_upper_bound}, we have $\phi_0 \leq 2\paren{f(\bfx_0,\bfu_0) - f^\star}$. Thus, Condition \ref{cond:phi_conv_cond} implies that
\[
    f(\bfx_k,\bfu_k) - f^\star \leq 2\paren{1 - \frac{c}{4\sqrt{\kappa}}}^k\paren{f(\bfx_0,\bfu_0) - f^\star}
\]
Therefore, Condition \ref{cond:phi_conv_cond} and \ref{cond:dx_bound_cond} together implies Theorem \ref{theo:nesterov_conv}. We now show that the two conditions hold by induction.
\subsubsection{Base Case: $\hat{k} = 0$}
When $\hat{k} = 0$, the only possible $k\leq \hat{k}$ is $k = 0$. In this case, Condition \ref{cond:phi_conv_cond} reads $\phi_0 \leq \phi_0$, which is automatically true. Condition \ref{cond:dx_bound_cond} also holds automatically since when $k=0$, we have
\[
    \norm{\bfx_k - \bfx_{k-1}}_2 = \norm{\bfx_0 - \bfx_{-1}}_2 = 0;\;\;\norm{\bfu_k - \bfu_{k-1}}_2 = \norm{\bfu_0 - \bfu_{-1}}_2 = 0
\]
\subsubsection{Inductive Step 1: Condition \ref{cond:phi_conv_cond} $\Rightarrow$ Condition \ref{cond:dx_bound_cond}}
Assume that Condition \ref{cond:phi_conv_cond} holds for all $k\leq \hat{k}$. We want to show that Condition \ref{cond:dx_bound_cond} holds for all $k\leq \hat{k}+1$ Notice that $f(\bfx_k,\bfu_k) - f^\star \leq \phi_k$. This implies that
\[
    f(\bfx_k,\bfu_k) - f^\star \leq\paren{1 - \frac{c}{4\sqrt{\kappa}}}^k\phi_0
\]
Combining with Assumption \ref{asump:strong_cvx}-\ref{asump:universal_opt}, and the condition that $G_1 \leq \frac{C_1\mu^2}{L_2(L_2+1)^2}\paren{\frac{1-\beta}{1+\beta}}^3$, we can invoke Lemma \ref{lem:xu_iter_bound_ext} to conclude directly that Condition \ref{cond:dx_bound_cond} holds for all $k\leq \hat{k}+1$.
\subsubsection{Inductive Step 2: Condition \ref{cond:dx_bound_cond} $\Rightarrow$ Condition \ref{cond:phi_conv_cond}}
Assume that Condition \ref{cond:phi_conv_cond} and \ref{cond:dx_bound_cond} holds for all $k\leq \hat{k}$. We show that Condition \ref{cond:phi_conv_cond} holds for all $k\leq \hat{k}+1$. To start, we first show that $\bfx_{\hat{k}+1} \in\mathcal{B}^{(1)}_{R_{\bfx}}$ and $\bfu_{\hat{k}+1} \in\mathcal{B}^{(2)}_{R_{\bfu}}$. By the triangle inequality, we have
\[
    \norm{\bfx_{\hat{k}+1} - \bfx_0}_2 \leq \norm{\bfy_{\hat{k}}-\bfx_0}_2 + \norm{\bfx_{\hat{k}+1}-\bfy_{\hat{k}}}_2;\;\norm{\bfu_{\hat{k}+1} - \bfx_0}_2 \leq \norm{\bfv_{\hat{k}}-\bfx_0}_2 + \norm{\bfu_{\hat{k}+1}-\bfv_{\hat{k}}}_2
\]
Since Condition \ref{cond:dx_bound_cond} implies that $\bfy_{\hat{k}}\in\mathcal{B}^{(1)}_{\sfrac{R_{\bfx}}{2}}$ and $\bfv_{\hat{k}}\in\mathcal{B}^{(2)}_{\sfrac{R_{\bfu}}{2}}$, it suffice to show that $\norm{\bfx_{\hat{k}+1}-\bfy_{\hat{k}}}_2\leq \sfrac{R_{\bfx}}{2}$ and $\norm{\bfu_{\hat{k}+1}-\bfv_{\hat{k}}}_2 \leq \sfrac{R_{\bfu}}{2}$. By Lemma \ref{lem:xu_iter_bound_ext}, we have
\begin{equation}
    \begin{aligned}
        \norm{\bfx_{\hat{k}+1}-\bfy_{\hat{k}}}_2 & = \eta\norm{\nabla_1f(\bfy_{\hat{k}},\bfv_{\hat{k}})}_2 \leq \beta\norm{\bfx_{\hat{k}} - \bfx_{\hat{k}-1}}_2 + \norm{\bfx_{\hat{k}+1} - \bfx_{\hat{k}}}_2\\
        & \leq 2\paren{6\frac{\eta(L_2+1)}{1-\beta}\phi_0}^{\frac{1}{2}}\\
        & \leq 8\paren{\frac{\eta(L_2+1)}{1-\beta}}^{\frac{1}{2}}(f(\bfx_0,\bfu_0) - f^\star)^{\frac{1}{2}}\\
        & \leq \frac{R_{\bfx}}{2}
    \end{aligned}
\end{equation}
This shows that $\bfx_{\hat{k}+1}\in\mathcal{B}^{(1)}_{R_{\bfx}}$. Similarly, we have
\begin{equation}
    \begin{aligned}
        \norm{\bfx_{\hat{k}+1}-\bfy_{\hat{k}}}_2 & = \eta\norm{\nabla_2f(\bfy_{\hat{k}},\bfv_{\hat{k}})}_2 \leq G_1\sqrt{\frac{L_2}{\mu}}\paren{\beta\norm{\bfx_{\hat{k}} - \bfx_{\hat{k}-1}}_2 + \norm{\bfx_{\hat{k}+1} - \bfx_{\hat{k}}}_2}\\
        & \leq 2G_1\sqrt{\frac{L_2}{\mu}}\paren{6\frac{\eta(L_2+1)}{1-\beta}\phi_0}^{\frac{1}{2}}\\
        & \leq 8\paren{\frac{\eta G_1^2L_2(L_2+1)}{\mu(1-\beta)}}^{\frac{1}{2}}(f(\bfx_0,\bfu_0) - f^\star)^{\frac{1}{2}}\\
        & \leq \frac{R_{\bfu}}{2}
    \end{aligned}
\end{equation}
Thus, $\bfu_{\hat{k}+1}\in\mathcal{B}^{(1)}_{R_{\bfu}}$. Now, we can invoke Lemma \ref{lem:lyapunov_conv_ext} to get that 
\begin{equation}
    \label{eq:theo3.2.1}
    \begin{aligned}
        (1+\gamma)\phi_{\hat{k}+1}\leq \phi_{\hat{k}} + \paren{\frac{G_1^2L_2}{2\gamma} + \frac{4\mathcal{Q}_1G_2^2}{\gamma\mu^2}(1 + \gamma)}\beta^2\norm{\bfu_{\hat{k}} - \bfu_{\hat{k}-1}}_2^2
    \end{aligned}
\end{equation}
With Condition \ref{cond:dx_bound_cond}, we can write (\ref{eq:theo3.2.1}) as
\begin{equation}
    \label{eq:theo3.2.2}
    \begin{aligned}
    (1+\gamma)\phi_{\hat{k}+1} - \phi_{\hat{k}} & \leq \paren{\frac{G_1^4L_2}{2\gamma} + \frac{4\mathcal{Q}_1G_1^2G_2^2}{\gamma\mu^2}(1 + \gamma)}\frac{6\eta\beta L_2(L_2+1)(1+\beta)^3}{\mu(1-\beta)^3}\\
    &\quad\quad\quad \cdot\paren{1 - \frac{c}{4\sqrt{\kappa}}}^{\hat{k}}\phi_0
    \end{aligned}
\end{equation}
Since $G_1^4 \leq \frac{C_1\mu^2}{L_2(L_2+1)^2}\paren{\frac{1-\beta}{1+\beta}}^3$, we have
\begin{equation}
    \label{eq:theo3.2.3}
    \frac{G_1^4L_2}{2\gamma} = \frac{C_1\mu^2}{2(L_2+1)^2\gamma}\paren{\frac{1-\beta}{1+\beta}}^3 \leq \frac{C_1\mu L_1}{c\sqrt{\kappa}(L_2+1)^2}\paren{\frac{1-\beta}{1+\beta}}^3 = \frac{C_1\mu}{\eta\sqrt{\kappa}(L_2+1)^2}\paren{\frac{1-\beta}{1+\beta}}^3
\end{equation}
where the inequality follows from $\frac{1}{\gamma}\leq \frac{2}{c}\sqrt{\kappa} = \frac{2 L_1}{c\mu\sqrt{\kappa}}$. Since $G_1^2G_2^2 \leq \frac{C_2\mu^3}{L_2(L_2+1)\sqrt{\kappa}}\paren{\frac{1-\beta}{1+\beta}}^2$, we have
\begin{equation}
    \label{eq:theo3.2.4}
    \begin{aligned}
        \frac{4\mathcal{Q}_1G_1^2G_2^2}{\gamma\mu^2}(1+\gamma) & = \frac{100C_2\gamma\mu}{\eta(1+\gamma)^4L_2(L_2+1)\sqrt{\kappa}}\paren{\frac{1-\beta}{1+\beta}}^2\\
        & \leq \frac{100C_2\gamma\mu}{\eta\sqrt{\kappa} L_2(L_2+1)}\paren{\frac{1-\beta}{1+\beta}}^2\\
        & \leq \frac{100C_2\mu}{\eta\sqrt{\kappa} L_2(L_2+1)}\paren{\frac{1-\beta}{1+\beta}}^3
    \end{aligned}
\end{equation}
where the last inequality follows from
\[
    \gamma \leq \frac{c}{\sqrt{\kappa}} \leq \frac{\sqrt{c}}{\sqrt{\kappa} + \sqrt{c}} \leq \frac{1-\beta}{1+\beta}
\]
Combining (\ref{eq:theo3.2.3}) and (\ref{eq:theo3.2.4}) gives
\[
    \frac{G_1^4L_2}{2\gamma} + \frac{4\mathcal{Q}_1G_1^2G_2^2}{\gamma\mu^2}(1+\gamma) \leq \frac{(C_1 + 100C_2)\mu}{\eta\sqrt{\kappa}L_2(L_2+1)}\paren{\frac{1-\beta}{1+\beta}}^3
\]
Thus (\ref{eq:theo3.2.2}) becomes
\begin{equation}
    \label{eq:theo3.2.5}
    (1+\gamma)\phi_{\hat{k}+1} - \phi_{\hat{k}} \leq \frac{6}{\sqrt{\kappa}}(C_1 + 100C_2)\paren{1 - \frac{c}{4\sqrt{\kappa}}}^{\hat{k}}\phi_0
\end{equation}
Plugging in the value of $\gamma = \frac{c}{2\sqrt{\kappa} - c}$ and choose a small enough $C_1$ and $C_2$ gives
\begin{equation}
    \label{eq:theo3.2.6}
    \paren{1 - \frac{c}{2\sqrt{\kappa}}}^{-1}\phi_{\hat{k}+1} - \phi_{\hat{k}} \leq \frac{1}{4\sqrt{\kappa}}\paren{1 - \frac{c}{4\sqrt{\kappa}}}^{\hat{k}}\phi_0
\end{equation}
Therefore
\begin{align*}
    \phi_{\hat{k}+1} & \leq \paren{1 - \frac{c}{2\sqrt{\kappa}}}\phi_{\hat{k}}+ \frac{1}{4\sqrt{\kappa}}\paren{1 - \frac{c}{4\sqrt{\kappa}}}^{\hat{k}+1}\phi_0\\
    & \leq \paren{1 - \frac{c}{2\sqrt{\kappa}}}\paren{1 - \frac{c}{4\sqrt{\kappa}}}^{\hat{k}}\phi_0 + \frac{1}{4\sqrt{\kappa}}\paren{1 - \frac{c}{4\sqrt{\kappa}}}^{\hat{k}}\phi_0\\
    & \leq \paren{1 - \frac{c}{4\sqrt{\kappa}}}^{\hat{k}+1}\phi_0
\end{align*}
This proves Condition \ref{cond:phi_conv_cond} for $\hat{k}+1$. Since we have established the induction for both conditions, we have completed the proof.
\subsection{Proof of Lemma \ref{lem:xu_iter_bound}/\ref{lem:xu_iter_bound_ext}}
We first restate an extended version of Lemma \ref{lem:xu_iter_bound} here.
\begin{lemma}
    \label{lem:xu_iter_bound_ext}
    Suppose that Assumption \ref{asump:strong_cvx}-\ref{asump:universal_opt} holds with $G_1^4 \leq \frac{C_1\mu^2}{L_2(L_2 + 1)^2}\paren{\frac{1-\beta}{1+\beta}}^3$ and for all $k\leq \hat{k}$
    \begin{equation}
        \label{eq:lem5.0}
        f(\bfx_k,\bfu_k) - f^\star\leq \paren{1 - \frac{c}{4\sqrt{\kappa}}}^k\phi_0
    \end{equation}
    Then we have $\bfx_k,\bfy_k\in\mathcal{B}^{(1)}_{R_{\bfx}}$ and $\bfu_k,\bfv_k\in\mathcal{B}^{(2)}_{R_{\bfu}}$ for all $k\leq \hat{k}+1$ with
    \begin{gather*}
        R_{\bfx} \geq \frac{18}{c}\sqrt{\kappa}\paren{\frac{3\eta(L_2+1)}{1 -\beta}}^{\frac{1}{2}}\paren{f(\bfx_0,\bfu_0) - f^\star}^{\frac{1}{2}}\\
        R_{\bfu} \geq \frac{18}{c}\sqrt{\kappa}\paren{\frac{6\eta G_1^2 L_2(L_2+1)(1+\beta)^3}{\mu\beta(1-\beta)^3}}\paren{f(\bfx_0,\bfu_0) - f^\star}^{\frac{1}{2}} 
    \end{gather*}
    Moreover, for all $k\leq \hat{k}+1$, we have $\eta\norm{\nabla_1f(\bfy_k,\bfv_k)}_2 \leq \beta\norm{\bfx_k - \bfx_{k-1}}_2 + \norm{\bfx_{k+1} - \bfx_{k}}_2$ and
    \begin{equation}
        \label{eq:xu_bound}
        \begin{gathered}
            \norm{\bfx_k - \bfx_{k-1}}_2^2\leq \frac{6\eta(L_2 + 1)}{1 - \beta}\paren{1 - \frac{c}{4\sqrt{\kappa}}}^k\phi_0\\
            \norm{\bfu_k - \bfu_{k-1}}_2^2 \leq G_1^2\frac{6\eta L_2(L_2+1)(1+\beta)^3}{\mu\beta(1-\beta)^3}\paren{1 - \frac{c}{4\sqrt{\kappa}}}^k\phi_0
        \end{gathered}
    \end{equation}
\end{lemma}
\begin{proof}
We prove Lemma \ref{lem:xu_iter_bound_ext} by induction. Notice that for $k = 0$, all the statements are automatically true. Assume that \ref{eq:xu_bound} holds up to iteration $k$, we first show that $\bfx_k,\bfy_k\in\mathcal{B}^{(1)}_{R_{\bfx}}$ and $\bfu_k,\bfv_k\in\mathcal{B}^{(2)}_{R_{\bfu}}$. 

\noindent\textbf{Step 1: $\bfx_k,\bfy_k\in\mathcal{B}^{(1)}_{R_{\bfx}}$.} By the triangle inequality, we have
\begin{equation}
    \label{eq:lem5.1.1}
    \begin{aligned}
        \norm{\bfx_k - \bfx_0}_2 & \leq \sum_{t=0}^{k-1}\norm{\bfx_{t+1} - \bfx_t}_2\\
        & \leq \sqrt{\frac{6\eta(L_2 + 1)}{1 - \beta}\phi_0}\sum_{t=0}^{k-1}\paren{1 - \frac{c}{4\sqrt{\kappa}}}^{\frac{k}{2}}\\
        & \leq \sqrt{\frac{6\eta(L_2 + 1)}{1 - \beta}\phi_0}\cdot\frac{1}{1 - \paren{1 - \frac{c}{4\sqrt{\kappa}}}^{\frac{1}{2}}}
    \end{aligned}
\end{equation}
Notice that $\paren{1 - \frac{c}{4\sqrt{\kappa}}}^{\frac{1}{2}} \leq 1 - \frac{c}{8\sqrt{\kappa}}$. Moreover, by lemma \ref{lem:phi0_upper_bound}, we have $\phi_0 \leq 2\paren{f(\bfx_0,\bfu_0) - f^\star}$ Thus (\ref{eq:lem5.1.1}) becomes
\begin{equation}
    \label{eq:lem5.1.2}
    \norm{\bfx_k - \bfx_0}_2 \leq \frac{8}{c}\sqrt{\kappa}\cdot\sqrt{\frac{6\eta(L_2 + 1)}{1 - \beta}\phi_0} \leq \frac{16}{c}\sqrt{\kappa}\paren{\frac{3\eta(L_2+1)}{1 -\beta}}^{\frac{1}{2}}\paren{f(\bfx_0,\bfu_0) - f^\star}^{\frac{1}{2}}
\end{equation}
This shows that $\norm{\bfx_k - \bfx_0}_2 \leq R_{\bfx}$. Moreover, since $\bfy_k - \bfx_k = \beta(\bfx_k - \bfx_{k-1})$, we have
\[
    \norm{\bfy_k - \bfx_0}_2 \leq \norm{\bfy_k - \bfx_k}_2 + \norm{\bfx_k - \bfx_0}_2 = \beta\norm{\bfx_k - \bfx_{k-1}} + \norm{\bfx_k - \bfx_0}_2
\]
By the inductive hypothesis, we have $\norm{\bfx_k - \bfx_{k-1}}_2^2\leq \frac{6\eta(L_2 + 1)}{1 - \beta}\paren{1 - \frac{c}{4\sqrt{\kappa}}}^k\phi_0$. Combining with the bound in (\ref{eq:lem5.1.2}), we have
\[
    \norm{\bfy_k - \bfx_0}_2 \leq \paren{\beta + \frac{8}{c}\sqrt{\kappa}}\sqrt{\frac{6\eta(L_2 + 1)}{1 - \beta}\phi_0} \leq \frac{18}{c}\sqrt{\kappa}\paren{\frac{3\eta(L_2+1)}{1 -\beta}}^{\frac{1}{2}}\paren{f(\bfx_0,\bfu_0) - f^\star}^{\frac{1}{2}}
\]
This shows that $\norm{\bfy_k - \bfx_0}_2 \leq R_{\bfx}$. 

\noindent\textbf{Step 2: $\bfu_k,\bfv_k\in\mathcal{B}^{(2)}_{R_{\bfu}}$.} Now, we focus on $\bfu_k$ and $\bfv_k$. Similar to (\ref{eq:lem5.1.1}), we have
\begin{equation}
    \label{eq:lem5.1.3}
    \begin{aligned}
        \norm{\bfu_k - \bfu_0}_2 & \leq \sum_{t=0}^{k-1}\norm{\bfu_{t+1} - \bfu_t}_2\\
        & \leq G_1\sqrt{\frac{6\eta L_2(L_2+1)(1+\beta)^3}{\mu\beta(1-\beta)^3}\phi_0}\sum_{t=0}^{k-1}\paren{1 - \frac{c}{4\sqrt{\kappa}}}^{\frac{k}{2}}\\
        & \leq G_1\sqrt{\frac{6\eta L_2(L_2+1)(1+\beta)^3}{\mu\beta(1-\beta)^3}\phi_0}\cdot\frac{1}{1 - \paren{1 - \frac{c}{4\sqrt{\kappa}}}^{\frac{1}{2}}}\\
        & \leq \frac{8}{c}\sqrt{\kappa}G_1\sqrt{\frac{6\eta L_2(L_2+1)(1+\beta)^3}{\mu\beta(1-\beta)^3}\phi_0}\\
        & \leq \frac{16}{c}\sqrt{\kappa}\paren{\frac{6\eta G_1^2 L_2(L_2+1)(1+\beta)^3}{\mu\beta(1-\beta)^3}}\paren{f(\bfx_0,\bfu_0) - f^\star}^{\frac{1}{2}} 
    \end{aligned}
\end{equation}
This shows that $\norm{\bfu_k - \bfu_0}_2 \leq R_{\bfu}$. Moreover, for $\bfv_k$ we have
\[
    \norm{\bfv_k - \bfu_0}_2 \leq \beta\norm{\bfu_k - \bfu_{k-1}}_2 + \norm{\bfu_k - \bfu_0}_2
\]
By the inductive hypothesis, we have $\norm{\bfu_k - \bfu_{k-1}}_2^2 \leq G_1^2\frac{6\eta L_2(L_2+1)(1+\beta)^3}{\mu\beta(1-\beta)^3}\paren{1 - \frac{c}{4\sqrt{\kappa}}}^k\phi_0$. Therefore
\begin{align*}
    \norm{\bfv_k - \bfu_0}_2 & \leq \paren{\beta + \frac{8}{c}\sqrt{\kappa}}G_1\sqrt{\frac{6\eta L_2(L_2+1)(1+\beta)^3}{\mu\beta(1-\beta)^3}\phi_0}\\
    & \leq \frac{18}{c}\sqrt{\kappa}\paren{\frac{6\eta G_1^2 L_2(L_2+1)(1+\beta)^3}{\mu\beta(1-\beta)^3}}\paren{f(\bfx_0,\bfu_0) - f^\star}^{\frac{1}{2}}
\end{align*}
This shows that $\norm{\bfv_k - \bfu_0}_2 \leq R_{\bfu}$. Thus, we have shown that $\bfx_k,\bfy_k\in\mathcal{B}^{(1)}_{R_{\bfx}}$ and $\bfu_k,\bfv_k\in\mathcal{B}^{(2)}_{R_{\bfu}}$. 

\noindent\textbf{Step 3: Bounding $\norm{\bfx_{k+1} - \bfx_k}_2$.} Now, we show that (\ref{eq:xu_bound}) holds with $k+1$. We start by recalling the iterates of Nesterov's momentum (\ref{eq:nesterov}). This iteration implies that
\begin{align*}
    \bfx_{k+1} = \bfy_k - \eta\nabla_1f(\bfy_k,\bfv_k) = \bfx_k + \beta(\bfx_k - \bfx_{k-1}) - \eta\nabla_1f(\bfy_k,\bfv_k)\\
    \bfu_{k+1} = \bfv_k - \eta\nabla_2f(\bfy_k,\bfv_k) = \bfu_k + \beta(\bfu_k - \bfu_{k-1}) - \eta\nabla_2f(\bfy_k,\bfv_k)
\end{align*}
Therefore, we can conclude that
\[
    \bfx_{k+1} - \bfx_k = \beta(\bfx_k - \bfx_{k-1}) - \eta\nabla_1f(\bfy_k,\bfv_k);\;\bfu_{k+1} - \bfu_k = \beta(\bfu_k - \bfu_{k-1}) - \eta\nabla_2f(\bfy_k,\bfv_k)
\]
For the convenience of our analysis, we will define
\[
    \dx{k} = \norm{\bfx_k - \bfx_{k-1}}_2;\quad\du{k} = \norm{\bfu_{k}-\bfu_{k-1}}_2
\]
To start, we notice that $\dx{k+1}$ can be expanded as
\begin{equation}
    \label{eq:lem5.1}
    \begin{aligned}
        \dx{k+1}^2 & = \norm{\bfx_{k+1} - \bfx_k}_2^2\\
        & = \norm{\beta(\bfx_k - \bfx_{k-1}) - \eta\nabla_1f(\bfy_k,\bfv_k)}_2^2\\
        & = \beta^2\dx{k}^2 - 2\eta\beta\inner{\nabla_1f(\bfy_k,\bfv_k)}{\bfx_k - \bfx_{k-1}} + \eta^2\norm{\nabla_1f(\bfy_k,\bfv_k)}_2^2
    \end{aligned}
\end{equation}
Both the inner product term and the gradient norm in (\ref{eq:lem5.1}) need to be bounded carefully. For the inner product term, since $\bfv_k\in\mathcal{B}^{(2)}_{R_{\bfu}}$, we invoke Assumption \ref{asump:strong_cvx} 
\begin{equation}
    \label{eq:lem5.2}
    \begin{aligned}
        f(\bfx_k,\bfv_k) & \geq f(\bfy_k,\bfv_k) + \inner{\nabla_1f(\bfy_k,\bfv_k)}{\bfx_k-\bfy_k} + \frac{\mu}{2}\norm{\bfx_k-\bfy_k}_2^2\\
        & \geq f(\bfy_k,\bfv_k) - \beta\inner{\nabla_1f(\bfy_k,\bfv_k)}{\bfx_k-\bfx_{k-1}}
    \end{aligned}
\end{equation}
where the last inequality follows from $\bfx_k-\bfy_k = -\beta(\bfx_k-\bfx_{k-1})$ and $\norm{\bfx_k-\bfy_k}_2^2 \geq 0$. Therefore (\ref{eq:lem5.2}) implies that
\begin{equation}
    \label{eq:lem5.3}
    - \beta\inner{\nabla_1f(\bfy_k,\bfv_k)}{\bfx_k-\bfx_{k-1}} \leq f(\bfx_k,\bfv_k) - f(\bfy_k,\bfv_k)
\end{equation}
For the gradient norm, since $\bfv_k\in\mathcal{B}^{(2)}_{R_{\bfu}}$, Assumption \ref{asump:f_smooth} must hold. Therefore, we can invoke Lemma \ref{lem:smooth_grad_bound} to get that
\begin{equation}
    \label{eq:lem5.4}
    \norm{\nabla_1f(\bfy_k,\bfv_k)}_2^2 \leq 2L_1(f(\bfy_k,\bfv_k) - f^\star)
\end{equation}
Plugging (\ref{eq:lem5.3}) and (\ref{eq:lem5.4}) into (\ref{eq:lem5.1}), we can get that
\begin{equation}
    \label{eq:lem5.5}
    \begin{aligned}
        \dx{k+1}^2 \leq \beta^2\dx{k}^2 + 2\eta\paren{f(\bfx_k,\bfv_k) - f(\bfy_k,\bfv_k)} + 2\eta^2L_1(f(\bfy_k,\bfv_k) - f^\star)
    \end{aligned}
\end{equation}
Since we choose $\eta = \frac{c}{L_1}$, as long as $c\leq 1$, we will have $2\eta^2L_1 \leq 2\eta$. Moreover, since $f(\bfy_k,\bfv_k) - f^\star \geq 0$, (\ref{eq:lem5.5}) can be further bounded as
\begin{equation}
    \label{eq:lem5.6}
    \begin{aligned}
        \dx{k+1}^2 & \leq \beta^2\dx{k}^2 + 2\eta\paren{f(\bfx_k,\bfv_k) - f(\bfy_k,\bfv_k)} + 2\eta(f(\bfy_k,\bfv_k) - f^\star)\\
        & = \beta^2\dx{k}^2 + 2\eta\paren{f(\bfx_k,\bfv_k) - f^\star}
    \end{aligned}
\end{equation}
With $\bfx_k\in\mathcal{B}^{(1)}_{R_{\bfx}}$ and $\bfu_k,\bfv_k\in\mathcal{B}^{(2)}_{R_{\bfu}}$, we must have that Assumption \ref{asump:g_smooth}, \ref{asump:h_lip} holds. Therefore, we can invoke Lemma \ref{lem:df_du} with $\mathcal{Q} = 1$ to get that
\[
    f(\bfx_k,\bfv_k) \leq f(\bfx_k,\bfu_k) + L_2(f(\bfx_k,\bfu_k) - f^\star) + \frac{G_1^2}{2}\paren{L_2 + 1}\norm{\bfu_k - \bfv_k}_2^2
\]
Subtracting $f^\star$ from both sides, and use the fact that $\bfu_k - \bfv_k = -\beta(\bfu_k - \bfu_{k-1})$, we have
\begin{equation}
    \label{eq:lem5.7}
    f(\bfx_k,\bfv_k) - f^\star \leq (L_2 + 1)\paren{f(\bfx_k,\bfu_k) - f^\star + \frac{G_1^2\beta^2}{2}\norm{\bfu_k - \bfu_{k-1}}_2^2}
\end{equation}
Plugging (\ref{eq:lem5.7}) into (\ref{eq:lem5.6}), and recalling that $\norm{\bfu_k - \bfu_{k-1}}_2^2 = \du{k}^2$ gives
\begin{equation}
    \label{eq:lem5.8}
    \dx{k+1}^2 \leq \beta^2\dx{k}^2 + 2\eta(L_2 + 1)\paren{f(\bfx_k,\bfu_k) - f^\star + \frac{G_1^2\beta^2}{2}\du{k}^2}
\end{equation}
Recall that $f(\bfx_k,\bfu_k) - f^\star$ satisfies (\ref{eq:lem5.0}). Moreover, by the inductive assumption, we have
\begin{align*}
    \dx{k}^2\leq \frac{6\eta(L_2 + 1)}{1 - \beta}\paren{1 - \frac{c}{4\sqrt{\kappa}}}^k\phi_0;\quad \du{k}^2 \leq G_1^2\frac{6\eta L_2(L_2+1)(1+\beta)^3}{\mu\beta(1-\beta)^3}\paren{1 - \frac{c}{4\sqrt{\kappa}}}^k\phi_0
\end{align*}
Therefore, (\ref{eq:lem5.7}) becomes
\begin{equation}
    \label{eq:lem5.9}
    \begin{aligned}
        \dx{k+1}^2 \leq \paren{\underbrace{\frac{6\beta^2\eta(L_2+1)}{1-\beta} + 4\eta(L_2+1) + \frac{6\eta^2\beta G_1^4L_2(L_2+1)^2(1+\beta)^3}{\mu(1-\beta)^3}}_{\zeta_1}}\paren{1 - \frac{c}{4\sqrt{\kappa}}}^k\phi_0
    \end{aligned}
\end{equation}
With $G_1^4 \leq \frac{C_1\mu^2}{L_2(L_2 + 1)^2}\paren{\frac{1-\beta}{1+\beta}}^3$ and $\eta = \frac{c}{L_1} \leq \frac{1}{\mu}$, we can derive that
\[
    \frac{6\eta^2\beta G_1^4L_2(L_2+1)^2(1+\beta)^3}{\mu(1-\beta)^3} \leq 6C_1\eta^2\beta\mu \leq 6C_1\eta\beta \leq 2\eta
\]
with a small enough $C_1$. Therefore, $\zeta_1$ in (\ref{eq:lem5.9}) can be simplified to
\begin{align*}
    \zeta_1 & \leq \frac{6\beta^2\eta(L_2+1)}{1-\beta} + 4\eta(L_2+1) + 2\eta\\
    & \leq \frac{6\eta(L_2 + 1)}{1 - \beta}\paren{\beta^2 + \frac{2}{3}(1-\beta) + \frac{1-\beta}{3(L_2 + 1)}}\\
    & \leq \frac{6\eta(L_2 + 1)}{1 - \beta}\paren{1 - \beta + \beta^2}
\end{align*}
By the choice of $\beta = \frac{4\sqrt{\kappa} - \sqrt{c}}{4\sqrt{\kappa} + 7\sqrt{c}}$, we have
\[
    \beta - \beta^2 = \frac{8\sqrt{c}(4\sqrt{\kappa} - \sqrt{c})}{(4\sqrt{\kappa} + 7\sqrt{c})^2} \leq \frac{c}{4\sqrt{\kappa}}
\]
Thus, $\zeta_1$ in (\ref{eq:lem5.9}) satisfies $\zeta_1 \leq \frac{6\eta(L_2 + 1)}{1 - \beta}\paren{1 - \frac{c}{4\sqrt{\kappa}}}$ and (\ref{eq:lem5.9}) becomes
\begin{equation}
    \label{eq:lem5_dx_bound}
    \dx{k+1}^2 \leq\frac{6\eta(L_2 + 1)}{1 - \beta}\paren{1 - \frac{c}{4\sqrt{\kappa}}}^{k+1}\phi_0
\end{equation}
This proves the inductive step for $\dx{k+1}^2$. 

\noindent\textbf{Step 4: Bounding $\norm{\bfu_{k+1} - \bfu_k}_2$.} Next, we focus on $\du{k}$. We first need a bound on $\norm{\nabla_2f(\bfy_k,\bfv_k)}_2$. To start, (\ref{eq:lem5.1}) implies that
\begin{align*}
    \eta^2\norm{\nabla_1f(\bfy_k,\bfv_k)}_2^2 & = \dx{k+1}^2 - \beta^2\dx{k}^2 + 2\eta\beta\inner{\nabla_1f(\bfy_k,\bfv_k)}{\bfx_k-\bfx_{k-1}}\\
    & \leq \dx{k+1}^2 - \beta^2\dx{k}^2 + 2\eta\beta\dx{k}\norm{\nabla_1f(\bfy_k,\bfv_k)}_2
\end{align*}
which implies that
\[
    \eta^2\norm{\nabla_1f(\bfy_k,\bfv_k)}_2^2 - 2\eta\beta\dx{k}\norm{\nabla_1f(\bfy_k,\bfv_k)}_2 + \beta^2\dx{k}^2 - \dx{k+1}^2 \leq 0
\]
Therefore, for all $\mathcal{G} = \eta\norm{\nabla_1f(\bfy_k,\bfv_k)}_2$, $\mathcal{G}$ must satisfy
\begin{equation}
    \label{eq:lem5.10}
    \mathcal{G}^2 - 2\beta\dx{k}\mathcal{G} + \beta^2\dx{k}^2 - \dx{k+1}^2 \leq 0
\end{equation}
When (\ref{eq:lem5.10}) takes equality, the two solutions of $\mathcal{G}$ are
\[
    \mathcal{G}_{\text{lb}} = \beta\dx{k} - \dx{k+1};\quad \mathcal{G}_{\text{ub}} = \beta\dx{k} + \dx{k+1}
\]
Therefore, $\norm{\nabla_1f(\bfy_k,\bfv_k)}_2$ must be bounded by
\begin{equation}
    \label{eq:lem5.11}
    \eta\norm{\nabla_1f(\bfy_k,\bfv_k)}_2 \leq \mathcal{G}_{\text{ub}} = \beta\dx{k} + \dx{k+1}
\end{equation}
Moreover, since Assumption \ref{asump:strong_cvx},\ref{asump:g_smooth},\ref{asump:h_lip}, and \ref{asump:universal_opt} holds, and that $\bfy_k\in\mathcal{B}^{(1)}_{R_{\bfx}}$ and $\bfv_k\in\mathcal{B}^{(2)}_{R_{\bfu}}$, we can apply Lemma \ref{lem:grad_dominance} to (\ref{eq:lem5.9}) get that
\begin{equation}
    \label{eq:lem5.12}
    \eta\norm{\nabla_2f(\bfy_k,\bfv_k)}_2 \leq G_1\sqrt{\frac{L_2}{\mu}}\paren{\beta\dx{k} + \dx{k+1}}
\end{equation}
Recall that $\bfu_{k+1} - \bfu_k = \beta(\bfu_k - \bfu_{k-1}) - \eta\nabla_2f(\bfy_k,\bfv_k)$. Unrolling this recursion gives
\[
    \bfu_{k+1} - \bfu_k = \eta\sum_{t=0}^k\beta^{k-t}\nabla_2f(\bfy_t,\bfv_t)
\]
Now, we can bound $\du{k+1}$ as
\begin{equation}
    \label{eq:lem5.13}
    \begin{aligned}
        \du{k+1} & \leq \eta\norm{\sum_{t=0}^k\beta^{k-t}\nabla_2f(\bfy_t,\bfv_t)}_2 \leq \eta\sum_{t=0}^k\beta^{k-t}\norm{\nabla_2f(\bfy_t,\bfv_t)}_2\\
        & \leq G_1\sqrt{\frac{L_2}{\mu}}\sum_{t=0}^k\beta^{k-t}\paren{\beta\dx{t} + \dx{t+1}}\\
        & = G_1\sqrt{\frac{L_2}{\mu}}\paren{\sum_{t=0}^k\beta^{k-t+1}\dx{t} + \sum_{t=0}^k\beta^{k-t}\dx{t+1}}
    \end{aligned}
\end{equation}
Recall the inductive hypothesis
\[
    \dx{k}^2\leq \frac{6\eta(L_2 + 1)}{1 - \beta}\paren{1 - \frac{c}{4\sqrt{\kappa}}}^k\phi_0
\]
Since we have shows that this bound also hold for $k+1$ in (\ref{eq:lem5_dx_bound}), we can write (\ref{eq:lem5.13}) as
\begin{equation}
    \label{eq:lem5.14}
    \begin{aligned}
        \du{k+1} = G_1\sqrt{\frac{6\eta L_2(L_2 + 1)}{\mu(1 - \beta)}}\phi_0^{\frac{1}{2}}\paren{\sum_{t=0}^{k}\beta^{k-t+1}\paren{1 - \frac{c}{4\sqrt{\kappa}}}^{\frac{t}{2}} + \sum_{t=0}^{k}\beta^{k-t}\paren{1 - \frac{c}{4\sqrt{\kappa}}}^{\frac{t+1}{2}}}
    \end{aligned}
\end{equation}
By the standard geometric series result, we have
\begin{align*}
    & \sum_{t=0}^{k}\beta^{k-t+1}\paren{1 - \frac{c}{4\sqrt{\kappa}}}^{\frac{t}{2}} = \beta\cdot\frac{\beta^{k+1} - \paren{1 - \frac{c}{4\sqrt{\kappa}}}^{\frac{k+1}{2}}}{\beta - \paren{1 - \frac{c}{4\sqrt{\kappa}}}^{\frac{1}{2}}}\\
    & \sum_{t=0}^{k}\beta^{k-t}\paren{1 - \frac{c}{4\sqrt{\kappa}}}^{\frac{t+1}{2}} = \paren{1 - \frac{c}{4\sqrt{\kappa}}}^{\frac{1}{2}}\cdot\frac{\beta^{k+1} - \paren{1 - \frac{c}{4\sqrt{\kappa}}}^{\frac{k+1}{2}}}{\beta - \paren{1 - \frac{c}{4\sqrt{\kappa}}}^{\frac{1}{2}}}
\end{align*}
By our choice of $\beta$, we must have that $\beta \leq 1 - \frac{c}{4\sqrt{\kappa}}\leq \paren{1 - \frac{c}{4\sqrt{\kappa}}}^{\frac{1}{2}}$. Therefore (\ref{eq:lem5.14}) becomes
\begin{equation*}
    \begin{aligned}
        \du{k+1} & \leq G_1\sqrt{\frac{6\eta L_2(L_2 + 1)}{\mu(1 - \beta)}}\phi_0\cdot\paren{\beta + \paren{1 - \frac{c}{4\sqrt{\kappa}}}^{\frac{1}{2}}}\frac{\beta^{k+1} - \paren{1 - \frac{c}{4\sqrt{\kappa}}}^{\frac{k+1}{2}}}{\beta - \paren{1 - \frac{c}{4\sqrt{\kappa}}}^{\frac{1}{2}}}\\
        & = G_1\sqrt{\frac{6\eta L_2(L_2 + 1)}{\mu(1 - \beta)}}\paren{1 - \frac{c}{4\sqrt{\kappa}}}^{\frac{k+1}{2}}\phi_0\cdot\frac{\paren{1 - \frac{c}{4\sqrt{\kappa}}}^{\frac{1}{2}} + \beta}{\paren{1 - \frac{c}{4\sqrt{\kappa}}}^{\frac{1}{2}} - \beta}
    \end{aligned}
\end{equation*}
Which implies that
\begin{equation}
    \label{eq:lem5.15}
    \du{k+1}^2 \leq \frac{6\eta G_1^2 L_2(L_2 + 1)}{\mu(1 - \beta)}\paren{1 - \frac{c}{4\sqrt{\kappa}}}^{k+1}\phi_0\cdot\paren{\frac{\paren{1 - \frac{c}{4\sqrt{\kappa}}}^{\frac{1}{2}} + \beta}{\paren{1 - \frac{c}{4\sqrt{\kappa}}}^{\frac{1}{2}} - \beta}}^2
\end{equation}
We notice that, since $\paren{1 - \frac{c}{4\sqrt{\kappa}}}^{\frac{1}{2}} \leq 1 - \frac{c}{8\sqrt{\kappa}}$, we have 
\begin{align*}
    \frac{\paren{1 - \frac{c}{4\sqrt{\kappa}}}^{\frac{1}{2}} + \beta}{\paren{1 - \frac{c}{4\sqrt{\kappa}}}^{\frac{1}{2}} - \beta} & = \frac{1 - \frac{c}{4\sqrt{\kappa}} + \beta^2 + 2\beta\paren{1 - \frac{c}{4\sqrt{\kappa}}}^{\frac{1}{2}}}{1 - \frac{c}{4\sqrt{\kappa}} + \beta^2 + 2\beta\paren{1 - \frac{c}{4\sqrt{\kappa}}}^{\frac{1}{2}}}\\
    & \leq \frac{1 - \frac{c}{4\sqrt{\kappa}} + \beta^2 - \beta\paren{2 - \frac{c}{4\sqrt{\kappa}}}}{1 - \frac{c}{4\sqrt{\kappa}} + \beta^2 - \beta\paren{2 - \frac{c}{4\sqrt{\kappa}}}}\\
    & = \frac{\paren{1 - \frac{c}{4\sqrt{\kappa}} + \beta}(1+\beta)}{\paren{1 - \frac{c}{4\sqrt{\kappa}} - \beta}(1 - \beta)}\\
\end{align*}
Since $\beta - \beta^2 \leq \frac{c}{4\sqrt{\kappa}}$, we have
\[
    1 - \frac{c}{4\sqrt{\kappa}} + \beta \leq 1 + \beta;\quad 1 - \frac{c}{4\sqrt{\kappa}} - \beta \geq 1 - \beta + \beta^2 \geq \beta(1-\beta)
\]
This gives
\[
    \frac{\paren{1 - \frac{c}{4\sqrt{\kappa}}}^{\frac{1}{2}} + \beta}{\paren{1 - \frac{c}{4\sqrt{\kappa}}}^{\frac{1}{2}} - \beta} \leq \frac{(1+\beta)^2}{\beta(1-\beta)^2}
\]
Thus, (\ref{eq:lem5.15}) becomes the following form, which proves the inductive step for $\du{k+1}^2$.
\begin{equation}
    \label{eq:lem5_du_bound}
    \du{k+1}^2 \leq \frac{6\eta G_1^2 L_2(L_2 + 1)(1+\beta)^3}{\mu\beta(1 - \beta)^3}\paren{1 - \frac{c}{4\sqrt{\kappa}}}^{k+1}\phi_0
\end{equation}
\end{proof}
\subsection{Proof of Lemma \ref{lem:lyapunov_conv_raw}/\ref{lem:lyapunov_conv_ext}}
We first restate an extended version of Lemma \ref{lem:xu_iter_bound} here. 
\begin{lemma}
    \label{lem:lyapunov_conv_ext}
    Suppose that Assumption \ref{asump:strong_cvx}-\ref{asump:universal_opt} holds with $G_1^4 \leq \frac{C_1\mu^2}{L_2(L_2 + 1)^2}\paren{\frac{1-\beta}{1+\beta}}^3$ and $G_1^2G_2^2 \leq \frac{C_2\mu^3}{L_2(L_2+1)\sqrt{\kappa}}\paren{\frac{1-\beta}{1+\beta}}^2$. If $\bfx_{k+1},\bfx_k,\bfy_k\in\mathcal{B}^{(1)}_{R_{\bfx}}$ and $\bfu_{k+1},\bfu_k,\bfv_k\in\mathcal{B}^{(2)}_{R_{\bfu}}$ for all $k\leq \hat{k}$, then for all $k\leq \hat{k}$ we have
    \[
        (1+\gamma)\phi_{k+1}\leq \phi_k + \paren{\frac{G_1^2L_2}{2\gamma} + \frac{4\mathcal{Q}_1G_2^2}{\gamma\mu^2}(1 + \gamma)}\beta^2\norm{\bfu_k - \bfu_{k-1}}_2^2
    \]
    where $\gamma = \frac{c}{2\sqrt{\kappa} -c}$ and $\mathcal{Q}_1 = \frac{\lambda^2}{2\eta(1+\gamma)^5}$ with $\lambda = (1 + \gamma)^3 - 1$.
\end{lemma}
\begin{proof}
    We will derive the bound for $(1+\gamma)\phi_{k+1} - \phi_k$. This quantity can be written as
    \begin{equation}
        \label{eq:lem10.0}
        \begin{aligned}
            (1+\gamma)\phi_{k+1} - \phi_k & = \underbrace{(1+\gamma)\paren{f(\bfx_{k+1},\bfu_{k+1}) - f^\star} - \paren{f(\bfx_{k},\bfu_{k}) - f^\star}}_{\Delta_1}\\
            & \quad\quad\quad + \mathcal{Q}_1\paren{\underbrace{(1+\gamma)\norm{\bfz_{k+1} - \bfx_k^\star}_2^2 - \norm{\bfz_k - \bfx_{k-1}^\star}_2^2}_{\Delta_2}}\\
            & \quad\quad\quad + \frac{\eta}{8}(1+\gamma)\norm{\nabla_1f(\bfy_k,\bfv_k}_2^2 - \frac{\eta}{8}\norm{\nabla_1f(\bfy_{k-1},\bfv_{k-1})}_2^2
        \end{aligned}
    \end{equation}
\end{proof}
In the following parts of the proof, we will bound $\Delta_1$ and $\Delta_2$ respectively. Throughout the proof, we will define $\Delta\bfy_k = \bfy_k - \bfx_k^{\star}, \Delta\bfz_k = \bfz_k - \bfx_k^{\star}$, and $\Delta\bfx_k = \bfx_k - \bfx_k^{\star}$

\noindent\textbf{Bound on $\Delta_1$.} We first study $f(\bfx_{k+1},\bfu_{k+1}) - f^{\star}$. It can be decomposed as
\begin{equation}
    \label{eq:lem10.1}
    f(\bfx_{k+1},\bfu_{k+1}) - f^{\star} = f(\bfx_{k+1},\bfv_{k}) - f^{\star} + f(\bfx_{k+1},\bfu_{k+1}) - f(\bfx_{k+1},\bfv_{k})
\end{equation}
Since $\bfx_{k+1}\in\mathcal{B}^{(1)}_{R_{\bfx}}$, and $\bfu_{k+1},\bfv_k\in\mathcal{B}^{(2)}_{R_{\bfu}}$, we can invoke Lemma \ref{lem:df_du} with $\mathcal{Q} = \sfrac{L_2}{\gamma}$ to get that 
\[
    f(\bfx_{k+1},\bfu_{k+1}) - f(\bfx_{k+1},\bfv_{k}) \leq \gamma\paren{f(\bfx_{k+1},\bfv_k) - f^\star} + \frac{G_1^2L_2}{2\gamma}\paren{1 + \gamma}\norm{\bfu_{k+1} - \bfv_k}_2^2
\]
Notice that $\norm{\bfu_{k+1} - \bfv_k}_2^2 = \eta^2\norm{\nabla_2f(\bfy_k,\bfv_k)}_2^2 \leq \eta^2G_1^2\cdot\frac{L_2}{\mu}\norm{\nabla_1f(\bfy_k,\bfv_k)}_2^2$. Thus
\[
    \frac{G_1^2L_2}{2\gamma}\paren{1 + \gamma}\norm{\bfu_{k+1} - \bfv_k}_2^2 \leq \frac{\eta^2G_1^4L_2^2}{2\gamma\mu}(1+\gamma)\norm{\nabla_1f(\bfy_k,\bfv_k)}_2^2 \leq \frac{\eta}{4}(1+\gamma)\norm{\nabla_1f(\bfy_k,\bfv_k)}_2^2
\]
for the choice of $G_1^4 \leq \frac{C_1\mu^2}{L_2(L_2+1)^2}\paren{\frac{1-\beta}{1+\beta}}^3 \leq \frac{\mu L_1\gamma}{2cL_2^2} = \frac{\mu\gamma}{2\eta L_2^2}$. In this way, (\ref{eq:lem10.1}) becomes
\begin{equation}
    \label{eq:lem10.2}
    f(\bfx_{k+1},\bfu_{k+1}) - f^\star\leq (1 + \gamma)(f(\bfx_{k+1},\bfv_{k}) - f^{\star}) + \frac{\eta}{4}(1+\gamma)\norm{\nabla_1f(\bfy_k,\bfv_k)}_2^2
\end{equation}
Again, since $\bfx_{k+1}\in\mathcal{B}^{(1)}_{R_{\bfx}}$, and $\bfu_{k+1},\bfv_k\in\mathcal{B}^{(2)}_{R_{\bfu}}$, we can use Assumption \ref{asump:f_smooth} and the iterate $\bfx_{k+1} = \bfy_k - \eta\nabla_1f(\bfy_k,\bfv_k)$ to get the 
\begin{equation}
    \label{eq:lem10.3}
    \begin{aligned}
        f(\bfx_{k+1},\bfv_{k}) - f^\star& \leq f(\bfy_k,\bfv_k) - f^\star + \inner{\nabla_1f(\bfy_k,\bfv_k}{\bfx_{k+1} - \bfy_k} + \frac{L_1}{2}\norm{\bfx_{k+1} - \bfy_k}_2^2\\
        & = f(\bfy_k,\bfv_k) - f^\star - \eta\norm{\nabla_1f(\bfy_k,\bfv_k)}_2^2 + \frac{\eta^2L_1}{2}\norm{\nabla_1f(\bfy_k,\bfv_k)}_2^2\\
        & = f(\bfy_k,\bfv_k) - f^\star - \eta\paren{1 - \frac{c}{2}}\norm{\nabla_1f(\bfy_k,\bfv_k)}_2^2
    \end{aligned}
\end{equation}
Combining (\ref{eq:lem10.3}) with (\ref{eq:lem10.2}) gives
\begin{equation}
    f(\bfx_{k+1},\bfu_{k+1}) - f^\star\leq (1+\gamma)(f(\bfy_k,\bfv_k) - f^\star) - \eta(1+\gamma)\paren{\frac{3}{4} - \frac{c}{2}}\norm{\nabla_1f(\bfy_k,\bfv_k)}_2^2
\end{equation}
Next, we study $f(\bfx_k,\bfu_k) - f^\star$. It can be decomposed as
\begin{equation}
    \label{eq:lem10.4}
    f(\bfx_k,\bfu_k) - f^\star = f(\bfx_k,\bfv_k) - f^\star - (f(\bfx_k,\bfv_k) - f(\bfx_k,\bfu_k))
\end{equation}
Since $\bfx_{k}\in\mathcal{B}^{(1)}_{R_{\bfx}}$, and $\bfu_{k},\bfv_k\in\mathcal{B}^{(2)}_{R_{\bfu}}$, we can invoke Lemma \ref{lem:df_du} with $\mathcal{Q} = \sfrac{L_2}{\gamma}$ to get that
\[
    f(\bfx_k,\bfv_k) - f(\bfx_k,\bfu_k) \leq \gamma(f(\bfx_k,\bfu_k) - f^\star) + \frac{G_1^2L_2}{2\gamma}(1+\gamma)\norm{\bfu_k - \bfv_k}_2^2
\]
Therefore, (\ref{eq:lem10.4}) becomes
\[
    f(\bfx_k,\bfu_k) - f^\star \geq f(\bfx_k,\bfv_k) - f^\star - \gamma(f(\bfx_k,\bfu_k) - f^\star) - \frac{G_1^2L_2}{2\gamma}(1+\gamma)\norm{\bfu_k - \bfv_k}_2^2
\]
which implies that
\begin{equation}
    \label{eq:lem10.5}
    f(\bfx_k,\bfu_k) - f^\star \geq \frac{1}{1+\gamma}(f(\bfx_k,\bfv_k) - f^\star) - \frac{G_1^2L_2}{2\gamma}\norm{\bfu_k - \bfv_k}_2^2
\end{equation}
Now, recall the definition of $\Delta_1$ in (\ref{eq:lem10.0}). Combined with (\ref{eq:lem10.3}) and (\ref{eq:lem10.5}), we can write $\Delta_1$ as
\begin{equation}
    \label{eq:lem10.6}
    \begin{aligned}
        \Delta_1 &= (1+\gamma)(f(\bfx_{k+1},\bfu_{k+1}) - f^\star) - (f(\bfx_k,\bfu_k) - f^\star)\\
        & \leq (1+\gamma)^2(f(\bfy_k,\bfv_k) - f^\star) - \eta(1+\gamma)^2\paren{\frac{3}{4} - \frac{c}{2}}\norm{\nabla_1f(\bfy_k,\bfv_k)}_2^2\\
        & \quad\quad\quad - \frac{1}{1+\gamma}(f(\bfx_k,\bfv_k) - f^\star) + \frac{G_1^2L_2}{2\gamma}\norm{\bfu_k - \bfv_k}_2^2\\
        & = \frac{1}{1+\gamma}(f(\bfy_k,\bfv_k) - f(\bfx_k,\bfv_k)) + \frac{\lambda}{1+\gamma}(f(\bfy_k,\bfv_k) - f^\star)\\
        &\quad\quad\quad- \eta(1+\gamma)^2\paren{\frac{3}{4} - \frac{c}{2}}\norm{\nabla_1f(\bfy_k,\bfv_k)}_2^2+ \frac{G_1^2L_2}{2\gamma}\norm{\bfu_k - \bfv_k}_2^2 
    \end{aligned}
\end{equation}
Since $\bfv_k\in\mathcal{B}^{(2)}_{R_{\bfu}}$, we can apply Assumption \ref{asump:strong_cvx} to get that
\begin{gather*}
    f(\bfx_k,\bfv_k) \geq f(\bfy_k,\bfv_k) + \inner{\nabla_1f(\bfy_k,\bfv_k)}{\bfx_k - \bfy_k}\\
    f^\star = f(\bfx_k^\star,\bfv_k) \geq f(\bfy_k,\bfv_k) + \inner{\nabla_1f(\bfy_k,\bfv_k)}{\bfx_k^\star - \bfy_k} + \frac{\mu}{2}\norm{\bfx_k^\star - \bfy_k}
\end{gather*}
Thus, we have
\begin{equation}
    \label{eq:lem10.7}
    \begin{aligned}
        & \frac{1}{1+\gamma}(f(\bfy_k,\bfv_k) - f(\bfx_k,\bfv_k)) + \frac{\lambda}{1+\gamma}(f(\bfy_k,\bfv_k) - f^\star)\\
        & \quad\quad \leq \frac{1}{1+\gamma}\inner{\nabla_1f(\bfy_k,\bfv_k)}{\bfy_k - \bfx_k} + \frac{\lambda}{1 + \gamma}\inner{\nabla_1f(\bfy_k,\bfv_k)}{\bfy_k - \bfx_k^\star}_2^2\\
        &\quad\quad\quad\quad\quad - \frac{\mu\lambda}{2(1+\gamma)}\norm{\bfx_k^\star - \bfy_k}\\
        & \quad\quad = \frac{1}{1 + \gamma}\inner{\nabla_1f(\bfy_k,\bfv_k)}{\bfy_k - \bfx_k + \lambda(\bfy_k - \bfx_k^\star)}- \frac{\mu\lambda}{2(1+\gamma)}\norm{\bfx_k^\star - \bfy_k}_2^2
    \end{aligned}
\end{equation}
Recall that $\bfz_k = \frac{1 - \beta\lambda}{\beta\lambda}(\bfy_k-\bfx_k) + \bfy_k$. This implies that $\bfy_k - \bfx_k = \frac{\beta\lambda}{1 - \beta\lambda}(\bfz_k - \bfy_k)$. Therefore
\begin{align*}
    \bfy_k - \bfx_k + \lambda(\bfy_k - \bfx_k^\star) & = \frac{\beta\lambda}{1 - \beta\lambda}(\bfz_k - \bfy_k) + \lambda(\bfy_k - \bfx_k^\star)\\
    & = \frac{\lambda}{1 - \beta\lambda}\paren{\beta(\bfz_k - \bfy_k) + (1 - \beta\lambda)(\bfy_k - \bfx_k^\star)}\\
    & = \frac{\lambda}{1 - \beta\lambda}\paren{\beta(\bfz_k - \bfx_k^\star) + (1 - \beta\lambda - \beta)(\bfy_k - \bfx_k^\star)}
\end{align*}
Recall that $\Delta\bfz_k = \bfz_k - \bfx_k^\star$ and $\Delta\bfy_k = \bfy_k - \bfx_k^\star$. Thus (\ref{eq:lem10.7}) becomes
\begin{align*}
    & \frac{1}{1+\gamma}(f(\bfy_k,\bfv_k) - f(\bfx_k,\bfv_k)) + \frac{\lambda}{1+\gamma}(f(\bfy_k,\bfv_k) - f^\star)\\
        & \quad\quad \leq\frac{\lambda}{(1 + \gamma)(1 - \beta\lambda)}\inner{\nabla_1f(\bfy_k,\bfv_k)}{\beta\Delta\bfz_k + (1 - \beta\lambda - \beta)\Delta\bfy_k} - \frac{\mu\lambda}{2(1+\gamma)}\norm{\Delta\bfy_k}_2^2
\end{align*}
Combining with (\ref{eq:lem10.6}), the bound on $\Delta_1$ becomes
\begin{equation}
    \label{eq:lem10.7.1}
    \begin{aligned}
        \Delta_1 & \leq \frac{\lambda}{(1 + \gamma)(1 - \beta\lambda)}\inner{\nabla_1f(\bfy_k,\bfv_k)}{\beta\Delta\bfz_k + (1 - \beta\lambda - \beta)\Delta\bfy_k} - \frac{\mu\lambda}{2(1+\gamma)}\norm{\Delta\bfy_k}_2^2\\
        & \quad\quad\quad - \eta(1+\gamma)^2\paren{\frac{3}{4} - \frac{c}{2}}\norm{\nabla_1f(\bfy_k,\bfv_k)}_2^2+ \frac{G_1^2L_2}{2\gamma}\norm{\bfu_k - \bfv_k}_2^2
    \end{aligned}
\end{equation}
\textbf{Bound on $\Delta_2$.} Now we turn to the bound on $\Delta_2$. Recall that $\Delta_2$ is defined as
\[
    \Delta_2 = (1 + \gamma)\norm{\bfz_{k+1} - \bfx_k^\star}_2^2 - \norm{\bfz_k - \bfx_{k-1}^\star}_2^2
\]
To start, we notice that, since $\bfv_k,\bfv_{k-1}\in\mathcal{B}^{(2)}_{R_{\bfu}}$, we can invoke Lemma \ref{lem:minimum_moving} to get that 
\[
    \norm{\bfx_k^{\star} - \bfx_{k-1}^{\star}}_2 \leq \frac{G_2}{\mu}\norm{\bfv_k - \bfv_{k-1}}_2
\]
Therefore, we can lower-bound $\norm{\bfz_k - \bfx_{k-1}^\star}_2^2$ as 
\begin{equation}
    \label{eq:lem10.8}
    \begin{aligned}
        \norm{\bfz_k - \bfx_{k-1}^\star}_2^2 & \geq \norm{\Delta\bfz_k} + 2\inner{\Delta\bfz_k}{\bfx_k - \bfx_{k-1}^{\star}}\\
        & \geq \paren{1 - \frac{\gamma}{2(1+\gamma)}}\norm{\Delta\bfz_k}_2^2 - \frac{2}{\gamma}(1+\gamma)\norm{\bfx_k - \bfx_{k-1}^{\star}}_2^2\\
        & \geq \frac{2 + \gamma}{2 + 2\gamma}\norm{\Delta\bfz_k}_2^2 - \frac{2G_2^2}{\gamma\mu^2}(1+\gamma)\norm{\bfv_k - \bfv_{k-1}}_2^2
    \end{aligned}
\end{equation}
Next, we provide an upper bound for $\norm{\bfz_{k+1} - \bfx_k^\star}_2^2$. Recall that $\bfy_{k+1} = \bfx_{k+1} + \beta(\bfx_{k+1} - \bfx_k)$ and $\bfx_{k+1} = \bfy_k - \eta\nabla_1f(\bfy_k,\bfv_k)$. We can re-write $\bfz_{k+1}$ as
\begin{align*}
    \bfz_{k+1} & = \frac{1 - \beta\lambda}{\beta\lambda}(\bfy_{k+1} - \bfx_{k+1}) + \bfy_{k+1}\\
    & = \frac{1}{\beta\lambda}\bfy_{k+1} - \frac{1 - \beta\lambda}{\beta\lambda}\bfx_{k+1}\\
    & = \frac{1+\beta}{\beta\lambda}\bfx_{k+1} - \frac{1}{\lambda}\bfx_k  - \frac{1 - \beta\lambda}{\beta\lambda}\bfx_{k+1}\\
    & = \frac{1 + \lambda}{\lambda}\bfx_{k+1} - \frac{1}{\lambda}\bfx_k\\
    & = \frac{1 + \lambda}{\lambda}\bfy_k - \frac{1}{\lambda}\bfx_k - \frac{1 + \lambda}{\lambda}\eta\nabla_1f(\bfy_k,\bfv_k)\\
    & = \frac{\beta}{1 - \beta\lambda}\bfz_k + \paren{1 - \frac{\beta}{1 - \beta\lambda}}\bfy_k - \frac{\eta}{\lambda}(1+\lambda)\nabla_1f(\bfy_k,\bfv_k)
\end{align*}
Therefore, we can write $\norm{\bfz_{k+1} - \bfx_k^\star}_2^2$ as
\begin{equation}
    \label{eq:lem10.9}
    \begin{aligned}
        \norm{\bfz_{k+1} - \bfx_k^\star}_2^2 & = \norm{\frac{\beta}{1 - \beta\lambda}\Delta\bfz_k + \paren{1 - \frac{\beta}{1 - \beta\lambda}}\Delta\bfy_k - \frac{\eta}{\lambda}(1+\lambda)\nabla_1f(\bfy_k,\bfv_k)}_2^2\\
        & = \frac{\beta^2}{(1 - \beta\lambda)^2}\norm{\Delta\bfz_k}_2^2 + \frac{(1 - \beta\lambda - \beta)^2}{(1 - \beta\lambda)^2}\norm{\Delta\bfy_k}_2^2 \\
        & \quad\quad\quad - \frac{2\eta(1 + \lambda)}{\lambda(1 - \beta\lambda)}\inner{\nabla_1f(\bfy_k,\bfv_k)}{\beta\Delta\bfz_k + (1 - \beta\lambda - \beta)\Delta\bfy_k}\\
        &\quad\quad\quad + \frac{2\beta(1-\beta\lambda - \beta)}{(1 - \beta\lambda)^2}\inner{\Delta\bfz_k}{\Delta\bfy_k}+\frac{\eta^2}{\lambda^2}(1 + \lambda)^2\norm{\nabla_1f(\bfy_k,\bfv_k)}_2^2\\
        & = \frac{\beta^2}{(1 - \beta\lambda)^2}\norm{\Delta\bfz_k}_2^2 + \frac{(1 - \beta\lambda - \beta)^2}{(1 - \beta\lambda)^2}\norm{\Delta\bfy_k}_2^2 \\
        & \quad\quad\quad - \frac{\lambda}{\mathcal{Q}_1(1+\gamma)^2(1-\beta\lambda)}\inner{\nabla_1f(\bfy_k,\bfv_k)}{\beta\Delta\bfz_k + (1 - \beta\lambda - \beta)\Delta\bfy_k}\\
        &\quad\quad\quad + \frac{2\beta(1-\beta\lambda - \beta)}{(1 - \beta\lambda)^2}\inner{\Delta\bfz_k}{\Delta\bfy_k}+\frac{\eta(1+\gamma)}{2\mathcal{Q}_1}\norm{\nabla_1f(\bfy_k,\bfv_k)}_2^2
    \end{aligned}
\end{equation}
where the last inequality follows from $\mathcal{Q}_1 = \frac{\lambda^2}{2\eta(1+\gamma)^2}$ and $\lambda = (1+\gamma)^3 - 1$. Therefore, combining (\ref{eq:lem10.9}) and (\ref{eq:lem10.8}) gives
\begin{equation}
    \label{eq:lem10.10}
    \begin{aligned}
        \Delta_2 & \leq \paren{\frac{\beta^2(1+\gamma)}{(1-\beta\lambda)^2} - \frac{2+\gamma}{2 + 2\gamma}}\norm{\Delta\bfz_k}_2^2 + (1 + \gamma)\cdot\frac{(1 - \beta\lambda - \beta)^2}{(1 - \beta\lambda)^2}\norm{\Delta\bfy_k}_2^2\\
        & \quad\quad\quad - \frac{\lambda}{\mathcal{Q}_1(1+\gamma)(1-\beta\lambda)}\inner{\nabla_1f(\bfy_k,\bfv_k)}{\beta\Delta\bfz_k + (1 - \beta\lambda - \beta)\Delta\bfy_k}\\
        &\quad\quad\quad + \frac{2(1+\gamma)\beta(1-\beta\lambda - \beta)}{(1 - \beta\lambda)^2}\inner{\Delta\bfz_k}{\Delta\bfy_k}+\frac{\eta(1+\gamma)^2}{2\mathcal{Q}_1}\norm{\nabla_1f(\bfy_k,\bfv_k)}_2^2\\
        &\quad\quad\quad + \frac{2G_2^2}{\gamma\mu^2}(1+\gamma)\norm{\bfv_k - \bfv_{k-1}}_2^2
    \end{aligned}
\end{equation}

\noindent\textbf{Putting things together.} Now, going back to (\ref{eq:lem10.0}). With the help of (\ref{eq:lem10.7.1}) and (\ref{eq:lem10.10}), we have
\begin{equation}
    \label{eq:lem10.11}
    \begin{aligned}
        (1 + \gamma)\phi_{k+1} - \phi_k & = \Delta_1 + \mathcal{Q}_1\Delta_2 + \frac{\eta}{8}(1 + \gamma)\norm{\nabla_1f(\bfy_k,\bfv_k)}_2^2 - \frac{\eta}{8}\norm{\nabla_1f(\bfy_{k-1},\bfv_{k-1})}_2^2\\
        & = \frac{\lambda}{(1 + \gamma)(1 - \beta\lambda)}\inner{\nabla_1f(\bfy_k,\bfv_k)}{\beta\Delta\bfz_k + (1 - \beta\lambda - \beta)\Delta\bfy_k}\\
        &\quad\quad\quad - \frac{\mu\lambda}{2(1+\gamma)}\norm{\Delta\bfy_k}_2^2 - \eta(1+\gamma)^2\paren{\frac{3}{4} - \frac{c}{2}}\norm{\nabla_1f(\bfy_k,\bfv_k)}_2^2\\
        & \quad\quad\quad+ \frac{G_1^2L_2}{2\gamma}\norm{\bfu_k - \bfv_k}_2^2+ \mathcal{Q}_1\paren{\frac{\beta^2(1+\gamma)}{(1-\beta\lambda)^2} - \frac{2+\gamma}{2 + 2\gamma}}\norm{\Delta\bfz_k}_2^2\\
        & \quad\quad\quad  + \mathcal{Q}_1(1 + \gamma)\cdot\frac{(1 - \beta\lambda - \beta)^2}{(1 - \beta\lambda)^2}\norm{\Delta\bfy_k}_2^2\\
        & \quad\quad\quad - \frac{\lambda}{(1+\gamma)(1-\beta\lambda)}\inner{\nabla_1f(\bfy_k,\bfv_k)}{\beta\Delta\bfz_k + (1 - \beta\lambda - \beta)\Delta\bfy_k}\\
        &\quad\quad\quad + \frac{2\mathcal{Q}_1(1+\gamma)\beta(1-\beta\lambda - \beta)}{(1 - \beta\lambda)^2}\inner{\Delta\bfz_k}{\Delta\bfy_k}\\
        &\quad\quad\quad +\frac{\eta}{2}(1+\gamma)^2\norm{\nabla_1f(\bfy_k,\bfv_k)}_2^2+ \frac{2\mathcal{Q}_1G_2^2}{\gamma\mu^2}(1+\gamma)\norm{\bfv_k - \bfv_{k-1}}_2^2\\
        &\quad\quad\quad +  \frac{\eta}{8}(1 + \gamma)\norm{\nabla_1f(\bfy_k,\bfv_k)}_2^2 - \frac{\eta}{8}\norm{\nabla_1f(\bfy_{k-1},\bfv_{k-1})}_2^2\\
        & = -\eta(1+\gamma)^2\paren{\frac{1}{4} - \frac{c}{2}}\norm{\nabla_1f(\bfy_k,\bfv_k)}_2^2  - \frac{\eta}{8}\norm{\nabla_1f(\bfy_{k-1},\bfv_{k-1})}_2^2\\
        & \quad\quad\quad + \mathcal{E}_{1,k} + \mathcal{E}_{2,k}
    \end{aligned}
\end{equation}
where $\mathcal{E}_{1,k}$ and $\mathcal{E}_{2,k}$ are defined as
\begin{equation}
    \label{eq:lem10.12}
    \begin{aligned}
        \mathcal{E}_{1,k} & = \mathcal{Q}_1\paren{\frac{\beta^2(1+\gamma)}{(1-\beta\lambda)^2} - \frac{2+\gamma}{2 + 2\gamma}}\norm{\Delta\bfz_k}_2^2 +  \frac{2\mathcal{Q}_1(1+\gamma)\beta(1-\beta\lambda - \beta)}{(1 - \beta\lambda)^2}\inner{\Delta\bfz_k}{\Delta\bfy_k}\\
        &\quad\quad\quad + \paren{\mathcal{Q}_1(1 + \gamma)\cdot\frac{(1 - \beta\lambda - \beta)^2}{(1 - \beta\lambda)^2} - \frac{\mu\lambda}{2(1+\gamma)}}\norm{\Delta\bfy_k}_2^2\\
        \mathcal{E}_{2,k} & = \frac{G_1^2L_2}{2\gamma}\norm{\bfu_k - \bfv_k}_2^2 + \frac{2\mathcal{Q}_1G_2^2}{\gamma\mu^2}(1+\gamma)\norm{\bfv_k - \bfv_{k-1}}_2^2
    \end{aligned}
\end{equation}
Our first step is to show that $\mathcal{E}_{1,k} \leq 0$. To start, notice that
\[
    \frac{\beta^2(1+\gamma)}{(1-\beta\lambda)^2} - \frac{2+\gamma}{2 + 2\gamma} \leq 0 \Leftrightarrow \beta \leq \frac{\sqrt{1 + \sfrac{\gamma}{2}}}{1 + \gamma + \lambda\sqrt{1 + \sfrac{\gamma}{2}}}
\]
By definition of $\beta$ and using $\lambda = (1 + \gamma)^3 - 1 \geq 3\gamma$, we have
\[ 
    \beta = \frac{4\sqrt{\kappa} - \sqrt{c}}{4\sqrt{\kappa} + 7\sqrt{c}} \leq \frac{4\sqrt{\kappa} - 2c}{4\sqrt{\kappa} + 8c} = \frac{1}{1 + 5\gamma} \leq \frac{\sqrt{1 + \sfrac{\gamma}{2}}}{1 + \gamma + \lambda\sqrt{1 + \sfrac{\gamma}{2}}}
\]
for a small enough constant $c$. Thus, we can guarantee that $\frac{\beta^2(1+\gamma)}{(1-\beta\lambda)^2} - \frac{2+\gamma}{2 + 2\gamma} \leq 0$. Therefore
\begin{align*}
    \frac{2\mathcal{Q}_1(1+\gamma)\beta(1-\beta\lambda - \beta)}{(1 - \beta\lambda)^2}\inner{\Delta\bfz_k}{\Delta\bfy_k} 
 & \leq \mathcal{Q}_1\paren{\frac{2+\gamma}{2 + 2\gamma} -\frac{\beta^2(1+\gamma)}{(1-\beta\lambda)^2}}\norm{\Delta\bfz_k}_2^2\\
 & + \frac{\mathcal{Q}_1(1+\gamma)^2\beta^2(1-\beta\lambda - \beta)^2}{\paren{\frac{2+\gamma}{2 + 2\gamma} -\frac{\beta^2(1+\gamma)}{(1-\beta\lambda)^2}}(1 - \beta\lambda)^4}\norm{\Delta\bfy_k}_2^2
\end{align*}
This implies that
\[
    \mathcal{E}_{1,k} \leq \paren{\frac{\mathcal{Q}_1(1+\gamma)^2\beta^2(1-\beta\lambda - \beta)^2}{\paren{\frac{2+\gamma}{2 + 2\gamma} -\frac{\beta^2(1+\gamma)}{(1-\beta\lambda)^2}}(1 - \beta\lambda)^4} + \mathcal{Q}_1(1 + \gamma)\cdot\frac{(1 - \beta\lambda - \beta)^2}{(1 - \beta\lambda)^2} - \frac{\mu\lambda}{2(1+\gamma)}}\norm{\Delta\bfy_k}_2^2
\]
To show that $\mathcal{E}_{1,k} \leq 0$, it suffice to show that
\[
    \frac{\mathcal{Q}_1(1+\gamma)^2\beta^2(1-\beta\lambda - \beta)^2}{\paren{\frac{2+\gamma}{2 + 2\gamma} -\frac{\beta^2(1+\gamma)}{(1-\beta\lambda)^2}}(1 - \beta\lambda)^4} + \mathcal{Q}_1(1 + \gamma)\cdot\frac{(1 - \beta\lambda - \beta)^2}{(1 - \beta\lambda)^2} \leq \frac{\mu\lambda}{2(1+\gamma)}
\]
Moving $\mathcal{Q}_1(1 + \gamma)$ to the right-hand side gives
\[
    \frac{(1+\gamma)\beta^2(1-\beta\lambda - \beta)^2}{\paren{\frac{2+\gamma}{2 + 2\gamma} -\frac{\beta^2(1+\gamma)}{(1-\beta\lambda)^2}}(1 - \beta\lambda)^4} + \frac{(1 - \beta\lambda - \beta)^2}{(1 - \beta\lambda)^2} \leq \frac{\mu\lambda}{2\mathcal{Q}_1(1+\gamma)^2}
\]
By the definition of $\mathcal{Q}_1$, we have $\frac{\mu\lambda}{2\mathcal{Q}_1(1+\gamma)^2} = \frac{\eta\mu(1+\gamma)^3}{\lambda} \geq \frac{\eta\mu(1 + \gamma)^3}{7\gamma}\geq \frac{c}{7\kappa\gamma}$. Moreover, 
\begin{align*}
    & \frac{(1+\gamma)\beta^2(1-\beta\lambda - \beta)^2}{\paren{\frac{2+\gamma}{2 + 2\gamma} -\frac{\beta^2(1+\gamma)}{(1-\beta\lambda)^2}}(1 - \beta\lambda)^4} + \frac{(1 - \beta\lambda - \beta)^2}{(1 - \beta\lambda)^2}\\
    & \quad\quad \leq \frac{(1+\gamma)\beta^2(1-\beta\lambda - \beta)^2 + (1 - \beta\lambda - \beta)^2(1 - \beta\lambda)^2 - (1+\gamma)\beta^2(1-\beta\lambda - \beta)^2}{\paren{\frac{2+\gamma}{2 + 2\gamma} -\frac{\beta^2(1+\gamma)}{(1-\beta\lambda)^2}}(1 - \beta\lambda)^4}\\
    & \quad\quad = \frac{(1 - \beta\lambda - \beta)^2}{\paren{\frac{2+\gamma}{2 + 2\gamma} -\frac{\beta^2(1+\gamma)}{(1-\beta\lambda)^2}}(1 - \beta\lambda)^2}\\
    & \quad\quad = \frac{\paren{1 - \frac{\beta}{1 - \beta\lambda}}^2}{1 - \frac{\gamma}{2(1+\gamma)} - \paren{\frac{\beta}{1 - \beta\lambda}}^2}
\end{align*}
Therefore, to make $\mathcal{E}_{1,k}\leq 0$, we just need
\[
    \frac{\paren{1 - \frac{\beta}{1 - \beta\lambda}}^2}{1 - \frac{\gamma}{2(1+\gamma)} - \paren{\frac{\beta}{1 - \beta\lambda}}^2} \leq  \frac{c}{7\kappa\gamma}
\]
Recall that $\beta = \frac{4\sqrt{\kappa}-\sqrt{c}}{4\sqrt{\kappa} + \sqrt{c}} = \frac{2\sqrt{c} - (1-2\sqrt{c})\gamma}{2\sqrt{c} + (7 + 2\sqrt{c})\gamma} \geq \frac{2\sqrt{c} - \gamma}{2\sqrt{c} + 8\gamma}$. This implies that
\[  
    \frac{\beta}{1 - \beta\lambda} \geq \frac{\beta}{1 - 7\beta\gamma} \geq \frac{2\sqrt{c} - \gamma}{2\sqrt{c} + 8\gamma}
\]
Therefore
\begin{align*}
    \frac{\paren{1 - \frac{\beta}{1 - \beta\lambda}}^2}{1 - \frac{\gamma}{2(1+\gamma)} - \paren{\frac{\beta}{1 - \beta\lambda}}^2} & \leq \frac{\paren{\frac{9\gamma}{2\sqrt{c} + 8\gamma}}^2}{1 - \frac{\gamma}{2} - \paren{\frac{2\sqrt{c} - \gamma}{2\sqrt{c} + 8\gamma}}^2}\\
    & \leq \frac{81\gamma^2}{\paren{1 - \frac{\gamma}{2}}(2\sqrt{c} + 8\gamma)^2 -(2\sqrt{c} - \gamma)^2}\\
    & \leq \frac{81\gamma}{36\sqrt{c} - 2c} \leq \frac{81\gamma}{35\sqrt{c}}
\end{align*}
Thus, to make $\mathcal{E}_{1,k} \leq 0$, we just need
\[
    \frac{81\gamma}{35\sqrt{c}} \leq \frac{c}{7\kappa\gamma} \Rightarrow \gamma \leq \frac{5c^{\frac{3}{4}}}{81\sqrt{\kappa}}
\]
Since our choice of $\gamma$ is $\gamma = \frac{c}{2\sqrt{\kappa} - c} \leq  \frac{5c^{\frac{3}{4}}}{81\kappa}$ for a smalle enough $c$, we can guarantee that $\mathcal{E}_{1,k}\leq 0$. Thus, (\ref{eq:lem10.11}) becomes 
\begin{equation}
    (1 + \gamma)\phi_{k+1} - \phi_k \leq -\eta(1+\gamma)^2\paren{\frac{1}{8} - \frac{c}{2}}\norm{\nabla_1f(\bfy_k,\bfv_k)}_2^2  - \frac{\eta}{8}\norm{\nabla_1f(\bfy_{k-1},\bfv_{k-1})}_2^2 + \mathcal{E}_{2,k}
\end{equation}
With $c \leq \frac{1}{4}$, it becomes
\begin{equation}
    \label{eq:lem10.13}
    \begin{aligned}
        (1 + \gamma)\phi_{k+1} - \phi_k & \leq \frac{G_1^2L_2}{2\gamma}\norm{\bfu_k - \bfv_k}_2^2 + \frac{2\mathcal{Q}_1G_2^2}{\gamma\mu^2}(1+\gamma)\norm{\bfv_k - \bfv_{k-1}}_2^2\\
        & \quad\quad\quad - \frac{\eta}{8}\norm{\nabla_1f(\bfy_{k-1},\bfv_{k-1})}_2^2
    \end{aligned}
\end{equation}
By the iterates of Nesterov's momentum, we have
\begin{align*}
    \bfu_k - \bfv_k = -\beta(\bfu_k - \bfu_{k-1});\;\bfv_k - \bfv_{k-1} = -\eta\nabla_2f(\bfy_{k-1},\bfv_{k-1}) + \beta(\bfu_k - \bfu_{k-1})
\end{align*}
Therefore, $\norm{\bfu_k - \bfv_k}_2^2 = \beta^2\norm{\bfu_k - \bfu_{k-1}}_2^2$ and
\begin{align*}
    \norm{\bfv_k - \bfv_{k-1}}_2^2 & \leq 2\eta^2\norm{\nabla_2f(\bfy_{k-1},\bfv_{k-1})}_2^2 + 2\beta^2\norm{\bfu_k - \bfu_{k-1}}_2^2\\
    & \leq \frac{2\eta^2G_1^2L_2}{\mu}\norm{\nabla_1f(\bfy_{k-1},\bfv_{k-1})}_2^2 + 2\beta^2\norm{\bfu_k - \bfu_{k-1}}_2^2
\end{align*}
where in the last inequality we invoke Lemma \ref{lem:grad_dominance}.
In this way, (\ref{eq:lem10.13}) becomes
\begin{equation}
    \label{eq:lem10.14}
    \begin{aligned}
        (1 + \gamma)\phi_{k+1} - \phi_k & \leq\beta^2\paren{\frac{G_1^2L_2}{2\gamma} + \frac{4\mathcal{Q}_1G_2^2}{\gamma\mu^2}(1+\gamma)}\norm{\bfu_k - \bfu_{k-1}}_2^2 + \\
        & \quad\quad\quad + \eta\paren{\frac{4\eta \mathcal{Q}_1G_1^2G_2^2L_2}{\gamma\mu^3}(1 + \gamma) - \frac{1}{8}}\norm{\nabla_1f(\bfy_{k-1},\bfv_{k-1})}_2^2\\
    \end{aligned}
\end{equation}
Plugging in the requirement $G_1^2G_2^2 \leq \frac{C_2\mu^3}{L_2(L_2+1)\sqrt{\kappa}}\paren{\frac{1-\beta}{1+\beta}}^2$ gives that $\frac{4\eta \mathcal{Q}_1G_1^2G_2^2L_2}{\gamma\mu^3}(1 + \gamma) \leq \frac{1}{8}$. Therefore, (\ref{eq:lem10.14}) becomes
\[
    (1 + \gamma)\phi_{k+1} - \phi_k \leq\beta^2\paren{\frac{G_1^2L_2}{2\gamma} + \frac{4\mathcal{Q}_1G_2^2}{\gamma\mu^2}(1+\gamma)}\norm{\bfu_k - \bfu_{k-1}}_2^2
\]
which completes the proof.
\section{Proofs for Section \ref{sec:addition_model}}
\label{sec:proof_additition_model}
\subsection{Proof of Lemma \ref{lem:additive_sat_asump}}
To start, recall that our objective function is defined as
\begin{equation}
    \label{eq:lem7.1}
    f(\bfx,\bfu) = \frac{1}{2}\norm{\bfA_1\bfx + \sigma\paren{\bfA_2\bfu} - \bfb}_2^2
\end{equation}
We first compute its gradient and its Hessian
\begin{equation}
    \label{eq:lem7.2}
    \begin{gathered}
        \nabla_1f(\bfx,\bfu) = \bfA_1^\top\paren{\bfA_1\bfx + \sigma\paren{\bfA_2\bfu} - \bfb}\\
        \nabla_{11}f(\bfx,\bfu) = \bfA_1^\top\bfA_1
    \end{gathered}
\end{equation}
Thus, $\bfa^\top\nabla_{11}f(\bfx,\bfu)\bfa = \norm{\bfA_1\bfa}_2^2$. This implies that $\lambda_{\max}\paren{\nabla_{11}f(\bfx,\bfu)} \leq \sigma_{\max}\paren{\bfA_1}^2$. Moreover, since $\bfA_1\in\R^{m\times m}$, we can also know that $\lambda_{\min}\paren{\nabla_{11}f(\bfx,\bfu)}\geq \sigma_{\min}\paren{\bfA_1}^2$. Thus, Assumption \ref{asump:strong_cvx}, \ref{asump:f_smooth} holds with $\mu = \sigma_{\min}\paren{\bfA_1}^2$ and $L_1 = \sigma_{\max}\paren{\bfA_1}^2$. Since $g$ is defined as $g(\bfs) = \frac{1}{2}\norm{\bfs - \bfb}_2^2$, it must be $1$-smooth. Moreover, its minimum values is $0$. Going back to $f$, we notice that since $\sigma_{\min}(\bfA_1)> 0$, choosing $\bfx^{\star}(\bfu) = \paren{\bfA_1^\top\bfA_1}^{-1}\bfA_1^\top\paren{\bfb -\sigma\paren{\bfA_2\bfu}}$ gives $f(\bfx,\bfu) = 0$. This shows that Assumption \ref{asump:g_smooth},\ref{asump:universal_opt} hold with $L_2 = 1$. For Assumption \ref{asump:h_lip}, we can compute that
\begin{align*}
    \norm{h(\bfx,\bfu) - h(\bfx,\bfv)}_2 & = \norm{\sigma\paren{\bfA_2\bfu} - \sigma\paren{\bfA_2\bfv}}_2\\
    & \leq B\norm{\bfA_2\bfu - \bfA_2\bfv}_2\\
    & \leq B\sigma_{\max}\paren{\bfA_2}\norm{\bfu - \bfv}_2
\end{align*}
where in the first inequality we use the $B$-Lipschitzness of $\sigma$. This shows that Assumption \ref{asump:h_lip} holds with $G_1 = B\sigma_{\max}\paren{\bfA_2}$. Lastly, for Assumption \ref{asump:grad_lip}, we can compute that
\begin{align*}
    \norm{\nabla_1f(\bfx,\bfu) - \nabla_1f(\bfx,\bfv)}_2 & = \norm{\bfA_1^\top\paren{\sigma\paren{\bfA_2\bfu} - \sigma\paren{\bfA_2\bfv}}}_2\\
    & \leq \sigma_{\max}\paren{\bfA_1}\norm{\sigma\paren{\bfA_2\bfu} - \sigma\paren{\bfA_2\bfv}}_2\\
    & \leq B\sigma_{\max}\paren{\bfA_1}\sigma_{\max}\paren{\bfA_2}\norm{\bfu - \bfv}_2
\end{align*}
Therefore, Assumption \ref{asump:grad_lip} holds with $G_2 = B\sigma_{\max}\paren{\bfA_1}\sigma_{\max}\paren{\bfA_2}$.
\subsection{Proof of Theorem \ref{theo:addit_conv}}
We want to invoke Theorem \ref{theo:nesterov_conv} to prove Theorem \ref{theo:addit_conv}. Thus, it suffices to check the requirements in (\ref{eq:G1_G4_req}) for the coefficients in Lemma \ref{lem:additive_sat_asump}:
\begin{equation}
    \label{eq:theo4.1}
    \begin{aligned}
       R_{\bfx} = & R_{\bfu} = \infty;\; \mu = \sigma_{\min}\paren{\bfA_1}^2;\; L_1 = \sigma_{\max}\paren{\bfA_1}^2;\; L_2 = 1\\
        G_1 & = B\sigma_{\max}\paren{\bfA_2};\; G_2 = B\sigma_{\max}\paren{\bfA_1}\sigma_{\max}\paren{\bfA_2}.
    \end{aligned}
\end{equation}
Since $\beta = \frac{4\sqrt{\kappa} - \sqrt{c}}{4\sqrt{\kappa}+7\sqrt{c}}$, we have
\[
    \frac{1-\beta}{1+\beta} = \frac{8\sqrt{c}}{6\sqrt{\kappa} + 6\sqrt{c}} \geq \frac{c'}{\sqrt{\kappa}}
\]
for some small enough constant $c'$. Treating $L_2 = 1$ as a constant, it suffices to guarantee that 
\begin{equation}
    G_1^4 \leq \frac{\tilde{C}_1\mu^2}{\kappa^{\frac{3}{2}}};\quad G_1^2G_2^2 \leq \frac{\tilde{C}_2\mu^3}{\kappa^{\frac{3}{2}}}
\end{equation}
for some small enough constants $\tilde{C}_1$ and $\tilde{C}_2$. Plugging in the coefficients in (\ref{eq:theo4.1})yields
\begin{equation}
    B^4\sigma_{\max}\paren{\bfA_2}^4 \leq \frac{\tilde{C}_1\sigma_{\min}\paren{\bfA_1}^4}{\kappa^{\frac{3}{2}}};\quad B^4\sigma_{\max}\paren{\bfA_2}^4\sigma_{\max}\paren{\bfA_1}^2 \leq \frac{\tilde{C}_2\sigma_{\min}\paren{\bfA_1}^6}{\kappa^{\frac{3}{2}}}
\end{equation}
The second condition can also be written as
\[
    B^4\sigma_{\max}\paren{\bfA_2}^4 \leq \frac{\tilde{C}_2\sigma_{\min}\paren{\bfA_1}^4}{\kappa^{\frac{5}{2}}}
\]
Therefore, we can guarantee the requirements in (\ref{eq:G1_G4_req}) as long as 
\[
    \sigma_{\min}\paren{\bfA_1}\geq \tilde{C}\sigma_{\max}\paren{\bfA_2}B\kappa^{0.75}
\]
for some large enough constant $\tilde{C}$. In this way, we can invoke Theorem \ref{theo:nesterov_conv} and notice that $f^\star = 0$ to get that
\[
    f(\bfx_k,\bfu_k) \leq 2\paren{1 - \frac{c}{4\sqrt{\kappa}}}^kf(\bfx_0,\bfu_0)
\]
\section{Proofs for Section \ref{sec:dnn}}
\label{sec:proof_dnn}
\subsection{Proof of Lemma \ref{lem:nn_sat_asump}}
\textbf{Orthogonal Transformation of the Parameters and Equivalence of Nesterov's Momentum.}\\
To obtain a parameter partition that achieves partial strong convexity, we cannot directly partition the parameters in the standard basis. Instead, we first apply an orthogonal transformation to all the parameters and then partition the transformed parameters. In particular, let $\bfO\in\R^{d\times d}$ be an orthogonal matrix, and for any objective $\hat{f}:\R^d\rightarrow\R$, we consider a new function $\tilde{f}(\bfw) = \hat{f}\paren{\bfO\bfw}$ for $\bfx\in\R^d$. Intuitively, $\tilde{f}$ is the equivalence of $\hat{f}$ on the orthogonally transformed parameter defined by $\bfO$. Ideally, using Nesterov's momentum to minimize $\tilde{f}$ executes
\[
    \bfw_{k+1} = \bar{\bfw}_k - \eta\nabla \tilde{f}\paren{\bar{\bfw}_k};\quad \bar{\bfw}_{k+1} = \bfw_{k+1} + \beta\paren{\bfw_{k+1} - \bfw_k}
\]
Notice that, by the chain rule, $\nabla\tilde{f}\paren{\bar{\bfw}_k} = \bfO^\top\nabla\hat{f}\paren{\bfO\bfw}$. If we multiply both sides of the two equations in the updates of Nesterov's momentum by $\bfO$, then we get
\[
    \bfO\bfw_{k+1} = \bfO\bar{\bfw}_k - \eta\nabla f\paren{\bfO\bar{\bfw}_k};\quad\bfO\bar{\bfw}_{k+1} = \bfO\bfw_{k+1} + \beta\paren{\bfO\bfw_{k+1} - \bfO\bfw_{k}}
\]
which is precisely the update rule of Nesterov's momentum for minimizing $f\paren{\bfO\bfw}$. Therefore, we can conclude that orthogonal transformation preserves the property of the algorithm of interest.

\noindent\textbf{Computation of the Coefficients.}
We let $\bfw = \paren{\texttt{V}\paren{\bfW_1},\dots,\texttt{V}\paren{\bfW_{\Lambda}}}\in\R^{\sum_{\ell=1}^{\Lambda}d_{\ell}d_{\ell-1}}$. Let $\bm{\theta}_0$ be the initialized parameter, and consider the SVD of $\bfF_{\Lambda-1}\paren{\bm{\theta}_0}$ as $\bfF_{\Lambda-1}\paren{\bm{\theta}_0} = \bfU\bm{\Sigma}_0\hat{\bfV}$ with $\bfU\in\R^{n\times n}$ and $\hat{\bfV}\in\R^{d_{\Lambda-1}\times d_{\Lambda-1}}$. Let $\hat{\bfV}_1\in\R^{n\times d_{\Lambda-1}}$ be the top-$n$ rows of $\hat{\bfV}$ and $\hat{\bfV}_2$ be the rest rows. We define $\bfV_1\in\R^{d_{\Lambda}n \times \sum_{\ell=1}^{\Lambda}d_{\ell}d_{\ell-1}},\bfV_2\in\R^{\paren{\sum_{\ell=1}^{\Lambda}d_{\ell}d_{\ell-1} - d_{\Lambda}n}\times \sum_{\ell=1}^{\Lambda}d_{\ell}d_{\ell-1}}$ in the following sense ($\otimes$ denotes the Kronecker product):
\begin{align*}
    \bfV_1 & = \begin{bmatrix}
        \bm{0}_{d_{\Lambda}n\times \sum_{\ell=1}^{\Lambda-1}d_{\ell}d_{\ell-1}}  & \bfI_{d_{\Lambda}}\otimes\hat{\bfV}_1
    \end{bmatrix}\\
    \bfV_2 & = \begin{bmatrix}
        \bm{0}_{d_{\Lambda}\paren{d_{\Lambda-1}-n}\times \sum_{\ell=1}^{\Lambda-1}d_{\ell}d_{\ell-1}} & \bfI_{d_{\Lambda}}\otimes\hat{\bfV}_2\\
        \bfI_{\sum_{\ell=1}^{\Lambda-1}d_{\ell}d_{\ell-1}} & \bm{0}_{\sum_{\ell=1}^{\Lambda-1}d_{\ell}d_{\ell-1}\times d_{\Lambda-1}d_{\Lambda}}
    \end{bmatrix}
\end{align*}
Together, $\begin{bmatrix}
    \bfV_1\\
    \bfV_2
\end{bmatrix}$ is an orthogonal matrix. Under this orthogonal transformation, we partition the aggregation of all parameters $\bfw$
into $\bfx = \bfV_1\bfw$ and $\bfu = \bfV_2\bfw$. Moreover, we let $\hat{\bfW}_{\Lambda,1} = \hat{\bfV}_1\bfW_{\Lambda}$ and $\hat{\bfW}_{\Lambda,2} = \hat{\bfV}_2\bfW_{\Lambda}$, and observe that
\[
    \bfx = \texttt{V}\paren{\hat{\bfW}_{\Lambda,1}}; \bfu = \paren{\texttt{V}\paren{\hat{\bfW}_{\Lambda,2}}, \texttt{V}\paren{\bfW_1},\dots, \texttt{V}\paren{\bfW_{\Lambda-1}}}
\]
Notice that $\bfF_{\Lambda}\paren{\bm{\theta}}$ can be written as
\[
    \bfF_{\Lambda}\paren{\bm{\theta}} = \bfF_{\Lambda-1}\paren{\bm{\theta}}\paren{\hat{\bfV}_1^\top\hat{\bfW}_{\Lambda,1} + \hat{\bfV}_2^\top\hat{\bfW}_{\Lambda,2}}
\]
Therefore, since $\bfx = \texttt{V}\paren{\hat{\bfW}_{\Lambda,1}}$, we have
\[
    \nabla_{11}f(\bfx,\bfu) = \bfI_{d_\Lambda}\otimes \hat{\bfV}_1\bfF_{\Lambda-1}\paren{\bm{\theta}_{\bfx,\bfu}}^\top\bfF_{\Lambda-1}\paren{\bm{\theta}_{\bfx,\bfu}}\hat{\bfV}_1^\top
\]
namely, $\nabla_{11}f\paren{\bfx,\bfu}$ is a block-diagonal matrix. Therefore, its eigenvalues are given by
\begin{align*}
    \lambda_{\max}\paren{\nabla_{11}f(\bfx,\bfu)} & = \lambda_{\max}\paren{\hat{\bfV}_1\bfF_{\Lambda-1}\paren{\bm{\theta}_{\bfx,\bfu}}^\top\bfF_{\Lambda-1}\paren{\bm{\theta}_{\bfx,\bfu}}\hat{\bfV}_1^\top} =\sigma_1\paren{\bfF_{\Lambda-1}\paren{\bm{\theta}_{\bfx,\bfu}}\hat{\bfV}_1^\top}^2\\
    \lambda_{\min}\paren{\nabla_{11}f(\bfx,\bfu)} & = \lambda_{\min}\paren{\hat{\bfV}_1\bfF_{\Lambda-1}\paren{\bm{\theta}_{\bfx,\bfu}}^\top\bfF_{\Lambda-1}\paren{\bm{\theta}_{\bfx,\bfu}}\hat{\bfV}_1^\top} = \sigma_n\paren{\bfF_{\Lambda-1}\paren{\bm{\theta}_{\bfx,\bfu}}\hat{\bfV}_1^\top}^2
\end{align*}
Based on our assumption, for all $\bfx\in\mathcal{B}^{(1)}_{R_{\bfx}}$ and $\bfu\in\mathcal{B}^{(2)}_{R_{\bfu}}$, we must have that
\begin{equation}
    \label{eq:lem8.1}
    \begin{gathered}
        \norm{\bfW_{\Lambda}(\bfx) - \bfW_{\Lambda}(\bfx_0)}_2 \leq \norm{\bfx - \bfx_0}_2 \leq R_{\bfx};\;
        \norm{\bfW_{\ell}(\bfu) - \bfW_{\ell}(\bfu_0)}_2 \leq \norm{\bfu - \bfu_0}_2 \leq R_{\bfu};
    \end{gathered}
\end{equation}
Moreover
\begin{equation}
    \label{eq:lem8.2}
    \sum_{\ell=1}^{\Lambda-1}\norm{\bfW_{\ell}(\bfu) - \bfW_{\ell}(\bfu_0)}_2 \leq \sqrt{\Lambda}\norm{\bfu - \bfu_0}_2 \leq \sqrt{\Lambda}R_{\bfu}
\end{equation}
Therefore, by (\ref{eq:lem8.1}), we have
\begin{gather*}
    \norm{\bfW_{\Lambda}(\bfx)}_2 \leq \norm{\bfW_{\Lambda}(\bfx_0)}_2 + \norm{\bfW_{\Lambda}(\bfx) - \bfW_{\Lambda}(\bfx_0)}_2 \leq \frac{\lambda_{\Lambda}}{2} + R_{\bfx} \leq \lambda_{\Lambda}\\
    \norm{\bfW_{\ell}(\bfu)}_2 \leq \norm{\bfW_{\ell}(\bfu_0)}_2 + \norm{\bfW_{\ell}(\bfu) - \bfW_{\ell}(\bfu_0)}_2 \leq \frac{\lambda_{\ell}}{2} + R_{\bfu} \leq \lambda_{\ell}
\end{gather*}
by the initialization property. This shows that requiring $R_{\bfu}\leq \frac{1}{2}\min_{\ell\in[\Lambda-1]}\lambda_{\ell}$ and $R_{\bfx}\leq \frac{\lambda_{\Lambda}}{2}$ suffice for making the definition of $\lambda_{\ell}$'s valid. By Lemma 2.1 in \citep{nguyen2021ontheproof}, we have
\begin{align*}
    \norm{\bfF_{\Lambda-1}(\bm{\theta}_{\bfx,\bfu}) - \bfF_{\Lambda-1}(\bm{\theta}(0))}_F & \leq \norm{\bfX}_F\lambda_{1\rightarrow \Lambda-1}\sum_{\ell=1}^{\Lambda-1}\lambda_{\ell}^{-1}\norm{\bfW_{\ell}(\bfx,\bfu) - \bfW_{\ell}(\bfx_0,\bfu_0)}_2\\
    & \leq \sqrt{\Lambda}\norm{\bfX}_F\lambda_{1\rightarrow\Lambda-1}R_{\bfu}\paren{\min_{\ell\in[\Lambda-1]}\lambda_{\ell}}^{-1}\\
    & \leq \frac{\alpha_0}{4}
\end{align*}
where the second-to-last inequality follows from (\ref{eq:lem8.2}), and the last inequality follows from the upper bound on $R_{\bfu}$. Therefore, we have
\begin{align*}
    \sigma_n\paren{\bfF_{\Lambda-1}\paren{\bm{\theta}_{\bfx,\bfu}}\hat{\bfV}_1^\top} & \leq \sigma_n\paren{\bfF_{\Lambda-1}\paren{\bm{\theta}(0)}\hat{\bfV}_1^\top} - \norm{\paren{\bfF_{\Lambda-1}(\bm{\theta}_{\bfx,\bfu}) - \bfF_{\Lambda-1}(\bm{\theta}(0))}\hat{\bfV}_1^\top}_2\\
    & \leq \sigma_n\paren{\bfF_{\Lambda-1}\paren{\bm{\theta}(0)}} - \norm{\bfF_{\Lambda-1}(\bm{\theta}_{\bfx,\bfu}) - \bfF_{\Lambda-1}(\bm{\theta}(0))}_F\\
    & \leq \alpha_0 - \frac{\alpha_0}{4}\\
    & = \frac{3}{4}\alpha_0
\end{align*}
where the second inequality follows from the fact that $\norm{\hat{\bfV}_1}_2\leq 1$. This implies that for all $\bfx\in R_{\bfx}$ and $\bfu\in R_{\bfu}$, we have
\begin{equation}
    \label{eq:lem14_mu}\lambda_{\min}\paren{\nabla_{11}f\paren{\bfx,\bfu}}\geq \paren{\frac{3}{4}\alpha_0}^2\geq \frac{\alpha_0^2}{2} =: \mu
\end{equation}
This shows the partial strong convexity. To prove the partial smoothness, we have
\begin{equation}
    \label{eq:lem14_L1}
    \sigma_1\paren{\bfF_{\Lambda-1}\paren{\bm{\theta}_{\bfx,\bfu}}\hat{\bfV}_1^\top} \leq \norm{\bfF_{\Lambda-1}\paren{\bm{\theta}_{\bfx,\bfu}}}_2\leq \norm{\bfX}_F\lambda_{1\rightarrow\Lambda-1}
\end{equation}
where the first inequality follows from the fact that $\norm{\hat{\bfV}_1}_2\leq 1$. Therefore,
\[
    \lambda_{\max}\paren{\nabla_{11}f\paren{\bfx,\bfu}} \leq \norm{\bfX}_F^2\lambda_{1\rightarrow\Lambda-1}^2 =: L_1
\]
Now, we proceed to compute $G_1$by bounding $\norm{h(\bfx,\bfu) - h(\bfx,\bfv)}_2$. We notice that
\begin{align*}
    h(\bfx,\bfu) - h(\bfx,\bfv) & = \bfF_{\Lambda-1}\paren{\bm{\theta}_{\bfx,\bfu}}\paren{\hat{\bfV}_1^\top\hat{\bfW}_{\Lambda,1}\paren{\bfx} + \hat{\bfV}_2^\top\hat{\bfW}_{\Lambda,2}\paren{\bfu}}\\
    & \quad\quad\quad- \bfF_{\Lambda-1}\paren{\bm{\theta}_{\bfx,\bfv}}\paren{\hat{\bfV}_1^\top\hat{\bfW}_{\Lambda,1}\paren{\bfx} + \hat{\bfV}_2^\top\hat{\bfW}_{\Lambda,2}\paren{\bfv}}\\
    & = \paren{\bfF_{\Lambda-1}\paren{\bm{\theta}_{\bfx,\bfu}} -  \bfF_{\Lambda-1}\paren{\bm{\theta}_{\bfx,\bfv}}}\paren{\hat{\bfV}_1^\top\hat{\bfW}_{\Lambda,1}\paren{\bfx} + \hat{\bfV}_2^\top\hat{\bfW}_{\Lambda,2}\paren{\bfu}}\\
    & \quad\quad\quad + \bfF_{\Lambda-1}\paren{\bm{\theta}_{\bfx,\bfv}}\hat{\bfV}_2^\top\paren{\hat{\bfW}_{\Lambda,2}(\bfu) - \hat{\bfW}_{\Lambda,2}(\bfv)}
\end{align*}
Now, for the first term, we have
\begin{align*}
    \norm{ \paren{\bfF_{\Lambda-1}\paren{\bm{\theta}_{\bfx,\bfu}} -  \bfF_{\Lambda-1}\paren{\bm{\theta}_{\bfx,\bfv}}}\hat{\bfV}_1^\top\hat{\bfW}_{\Lambda,1}}_F & \leq \lambda_{\Lambda}\norm{\bfF_{\Lambda-1}\paren{\bm{\theta}_{\bfx,\bfu}} -  \bfF_{\Lambda-1}\paren{\bm{\theta}_{\bfx,\bfv}}}_F\\
    & \leq \norm{\bfX}_F\lambda_{1\rightarrow \Lambda-1}\sum_{\ell=1}^{\Lambda}\lambda_{\ell}^{-1}\norm{\bfW_{\ell}(\bfu) - \bfW_{\ell}(\bfv)}_2\\
    & \leq \sqrt{\Lambda}\norm{\bfX}_F\lambda_{1\rightarrow \Lambda}\paren{\min_{\ell\in[\Lambda-1]}\lambda_{\ell}}^{-1}\norm{\bfu- \bfv}_2
\end{align*}
For the second term, we have
\begin{align*}
    & \norm{\bfF_{\Lambda-1}\paren{\bm{\theta}_{\bfx,\bfv}}\hat{\bfV}_2^\top\paren{\hat{\bfW}_{\Lambda,2}(\bfu) - \hat{\bfW}_{\Lambda,2}(\bfv)}}_F\\
    & \quad\quad\quad\leq  \norm{\bfF_{\Lambda-1}\paren{\bm{\theta}_{\bfx,\bfv}}\hat{\bfV}_2^\top}_2\norm{\hat{\bfW}_{\Lambda,2}(\bfu) - \hat{\bfW}_{\Lambda,2}(\bfv)}\\
    & \quad\quad\quad\leq \norm{\bfF_{\Lambda-1}\paren{\bm{\theta}_{\bfx,\bfv}} - \bfF_{\Lambda-1}\paren{\bm{\theta}(0)}}_F\norm{\bfu - \bfv}_2\\
    & \quad\quad\quad\leq \sqrt{\Lambda}\norm{\bfX}_F\lambda_{1\rightarrow\Lambda-1}R_{\bfu}\paren{\min_{\ell\in[\Lambda-1]}\lambda_{\ell}}^{-1}\norm{\bfu - \bfv}_2
\end{align*}
where the second inequality follows from
\begin{align*}
     \norm{\bfF_{\Lambda-1}\paren{\bm{\theta}_{\bfx,\bfv}}\hat{\bfV}_2^\top}_2 & \leq  \norm{\bfF_{\Lambda-1}\paren{\bm{\theta}(0)}\hat{\bfV}_2^\top}_2 + \norm{\bfF_{\Lambda-1}\paren{\bm{\theta}_{\bfx,\bfv}} - \bfF_{\Lambda-1}\paren{\bm{\theta}(0)}}_2\\
     & = \norm{\bfF_{\Lambda-1}\paren{\bm{\theta}_{\bfx,\bfv}} - \bfF_{\Lambda-1}\paren{\bm{\theta}(0)}}_2
\end{align*}
by noticing that $\norm{\bfF_{\Lambda-1}\paren{\bm{\theta}(0)}\hat{\bfV}_2^\top}_2 = 0$ by the definition of $\hat{\bfV}_2$.
Combining the two, we have
\begin{align*}
    \norm{h(\bfx,\bfu) - h(\bfx,\bfv)}_2 \leq \paren{\lambda_{\Lambda} + R_{\bfu}}\sqrt{\Lambda}\norm{\bfX}_F\lambda_{1\rightarrow\Lambda-1}\paren{\min_{\ell\in[\Lambda-1]}\lambda_{\ell}}^{-1}\norm{\bfu - \bfv}_2
\end{align*}
This implies that
\[
    G_1 = \paren{\lambda_{\Lambda} + R_{\bfu}}\sqrt{\Lambda}\norm{\bfX}_F\lambda_{1\rightarrow\Lambda-1}\paren{\min_{\ell\in[\Lambda-1]}\lambda_{\ell}}^{-1}
\]
Next, we proceed to compute $G_1$ by bounding $\norm{\nabla_1f\paren{\bfx,\bfu} - \nabla_1f\paren{\bfx,\bfv}}_2$. 
Computing the gradient, we have
\[
    \nabla_{\hat{\bfW}_{\Lambda, 1}}\mathcal{L}\paren{\bm{\theta}} = \hat{\bfV}_1\bfF_{\Lambda - 1}\paren{\bm{\theta}}^\top\paren{\bfF_{\Lambda - 1}\paren{\bm{\theta}}\paren{\hat{\bfV}_1^\top\hat{\bfW}_{\Lambda, 1} + \hat{\bfV}_2^\top\hat{\bfW}_{\Lambda, 2}} - \bfY}
\]
Therefore
\begin{align*}
    & \nabla_1f(\bfx,\bfu) - \nabla_1f\paren{\bfx,\bfv} \\
    &\quad\quad\quad = \underbrace{\hat{\bfV}_1\paren{\bfF_{\Lambda - 1}\paren{\bm{\theta}_{\bfx,\bfu}}^\top\bfF_{\Lambda - 1}\paren{\bm{\theta}_{\bfx,\bfu}} - \bfF_{\Lambda - 1}\paren{\bm{\theta}_{\bfx,\bfv}}^\top\bfF_{\Lambda - 1}\paren{\bm{\theta}_{\bfx,\bfv}}}\bfW_{\Lambda}(\bfx, \bfu)}_{\bm{\delta}_1}\\
    &\quad\quad\quad -\hat{\bfV}_1\paren{\bfF_{\Lambda - 1}\paren{\bm{\theta}_{\bfx,\bfu}} - \bfF_{\Lambda - 1}\paren{\bm{\theta}_{\bfx,\bfv}}}\bfY \\
    &\quad\quad\quad + \underbrace{\hat{\bfV}_1^\top\bfF_{\Lambda - 1}\paren{\bm{\theta}_{\bfx,\bfv}}^\top\bfF_{\Lambda - 1}\paren{\bm{\theta}_{\bfx,\bfv}}\hat{\bfV}_2\paren{\hat{\bfW}_{\Lambda, 2}(\bfu) - \hat{\bfW}_{\Lambda, 2}(\bfu)}}_{\bm{\delta}_2}
\end{align*}
We bound the magnitude of $\bm{\delta}_1$ and $\bm{\delta}_2$ separately. For $\bm{\delta}_1$, we have
\begin{align*}
    \norm{\bm{\delta}_1}_F & = \norm{\bfF_{\Lambda - 1}\paren{\bm{\theta}_{\bfx,\bfu}}^\top\bfF_{\Lambda - 1}\paren{\bm{\theta}_{\bfx,\bfu}} - \bfF_{\Lambda - 1}\paren{\bm{\theta}_{\bfx,\bfv}}^\top\bfF_{\Lambda - 1}\paren{\bm{\theta}_{\bfx,\bfv}}}_F\\
    & \leq \paren{\norm{\bfF_{\Lambda - 1}\paren{\bm{\theta}_{\bfx,\bfu}}}_F + \norm{\bfF_{\Lambda - 1}\paren{\bm{\theta}_{\bfx,\bfv}}}_F}\norm{\bfF_{\Lambda - 1}\paren{\bm{\theta}_{\bfx,\bfu}} - \bfF_{\Lambda - 1}\paren{\bm{\theta}_{\bfx,\bfv}}}_F\norm{\bfW_{\Lambda}(\bfx,\bfu)}_2\\
    & \leq 2\norm{\bfX}_F\lambda_{1\rightarrow\Lambda - 1}\cdot \sqrt{\Lambda}\norm{\bfX}_F\lambda_{1\rightarrow \Lambda-1}\paren{\min_{\ell\in[\Lambda-1]}\lambda_{\ell}}^{-1}\norm{\bfu- \bfv}_2\cdot \lambda_{\Lambda}\\
    & = 2\sqrt{\Lambda}\norm{\bfX}_F^2\lambda_{\Lambda}\lambda_{1\rightarrow\Lambda-1}^2\paren{\min_{\ell\in[\Lambda-1]}\lambda_{\ell}}^{-1}\norm{\bfu- \bfv}_2
\end{align*}
For the second term, we have
\begin{align*}
    \norm{\hat{\bfV}_1\paren{\bfF_{\Lambda - 1}\paren{\bm{\theta}_{\bfx,\bfu}} - \bfF_{\Lambda - 1}\paren{\bm{\theta}_{\bfx,\bfv}}}\bfY } & \leq \norm{\bfF_{\Lambda - 1}\paren{\bm{\theta}_{\bfx,\bfu}} - \bfF_{\Lambda - 1}\paren{\bm{\theta}_{\bfx,\bfv}}}_F\norm{\bfY}_F\\
    & \leq \sqrt{\Lambda}\norm{\bfX}_F\norm{\bfY}_F\lambda_{1\rightarrow \Lambda-1}\paren{\min_{\ell\in[\Lambda-1]}\lambda_{\ell}}^{-1}\norm{\bfu- \bfv}_2
\end{align*}
Lastly, for $\bm{\delta}_2$, we have
\begin{align*}
    \norm{\bm{\delta}_2}_F & \leq \norm{\bfF_{\Lambda-1}\paren{\bm{\theta}_{\bfx,\bfv}}}_2\norm{\bfF_{\Lambda-1}\paren{\bm{\theta}_{\bfx,\bfv}} - \bfF_{\Lambda-1}\paren{\bm{\theta}(0)}}_2\norm{\hat{\bfW}_{\Lambda, 2}(\bfu) - \hat{\bfW}_{\Lambda, 2}(\bfu)}_F\\
    & \leq \sqrt{\Lambda}\norm{\bfX}_F^2\lambda_{1\rightarrow \Lambda-1}^2\paren{\min_{\ell\in[\Lambda-1]}\lambda_{\ell}}^{-1}R_{\bfu}\norm{\bfu- \bfv}_2
\end{align*}
Putting things together gives
\[
    G_2 = \paren{\paren{2\lambda_{\Lambda} + R_{\bfu}}\norm{\bfX}_F\lambda_{1\rightarrow\Lambda - 1} + \norm{\bfY}_F}\sqrt{\Lambda}\norm{\bfX}_F\lambda_{1\rightarrow\Lambda - 1}\paren{\min_{\ell\in[\Lambda-1]}\lambda_{\ell}}^{-1}
\]
Now that we have shown that Assumption 1,2,4,5 holds, we proceed to prove Assumption 3 and Assumption 6. 
Simple decomposition gives
\[
    \bfF_{\Lambda}\paren{\bm{\theta}} = \bfF_{\Lambda-1}\paren{\bm{\theta}}\paren{\hat{\bfV}_1^\top\hat{\bfW}_{\Lambda, 1} + \hat{\bfV}_2^\top\hat{\bfW}_{\Lambda,2}}
\]
Therefore, to set $\bfF_{\Lambda}\paren{\bm{\theta}} = \bfY$, we can simply let $\hat{\bfW}_{\Lambda,1}$ to be
\[
    \hat{\bfW}_{\Lambda,1} = \paren{\hat{\bfV}_1\bfF_{\Lambda-1}\paren{\bm{\theta}}^\top\bfF_{\Lambda-1}\paren{\bm{\theta}}\hat{\bfV}_1^\top}^{-1}\hat{\bfV}_1\bfF_{\Lambda-1}\paren{\bm{\theta}}^\top\paren{\bfY - \hat{\bfV}_2^\top\hat{\bfW}_{\Lambda,2}}
\]
since $\hat{\bfV}_1\bfF_{\Lambda-1}\paren{\bm{\theta}}^\top\bfF_{\Lambda-1}\paren{\bm{\theta}}\hat{\bfV}_1^\top\in\R^{n\times n}$ has full rank. This shows that $\min_{\bfx}f\paren{\bfx,\bfu} = 0$ for any $\bfu$. Since $f\paren{\bfx,\bfu}\geq0$ by the property of the MSE, we can conclude that Assumption 6 holds. Moreover, Assumption 3 also holds with $L_2 =1$ since MSE is by itself $1$-smooth.

\subsection{Proof of Theorem \ref{theo:nn_nesterov_conv}}
We want to invoke Theorem \ref{theo:nesterov_conv} to prove Theorem \ref{theo:nn_nesterov_conv}. Thus it suffices to check the requirements in (\ref{eq:G1_G4_req}) (which we restate below):
\begin{align*}
    & G_1^4 \leq \frac{C_1\mu^2}{L_2(L_2 + 1)^2}\paren{\frac{1-\beta}{1+\beta}}^3;\quad G_1^2G_2^2\leq \frac{C_2\mu^3}{L_2(L_2+1)\sqrt{\kappa}}\paren{\frac{1-\beta}{1+\beta}}^2; \\
    & R_{\bfx} \geq \frac{36}{c}\sqrt{\kappa}\paren{\frac{\eta(L_2 + 1)}{1 - \beta}}^{\frac{1}{2}}(f(\bfx_0, \bfu_0) - f^\star)^{\frac{1}{2}}; \\
    & R_{\bfu} \geq \frac{36}{c}\sqrt{\kappa}\paren{\frac{\eta G_1^2L_2(L_2 + 1)(1+\beta)^3}{\mu\beta(1-\beta)^3}}^{\frac{1}{2}}(f(\bfx_0, \bfu_0) - f^\star)^{\frac{1}{2}},
\end{align*}
for the coefficients in Lemma \ref{lem:nn_sat_asump} (which we restate below as well):
\begin{gather*}
    R_{\bfx} = \frac{\lambda_{\Lambda}}{2};\quad R_{\bfu} = \frac{1}{2}\paren{\min_{\ell\in[\Lambda - 1]}\lambda_{\ell}}\min\left\{1, \frac{\alpha_0}{2\sqrt{\Lambda}\norm{\bfX}_F\lambda_{1\rightarrow\Lambda - 1}}\right\}^2;\quad\mu = \frac{\alpha_0^2}{2}\\  L_1 = \norm{\bfX}_F^2\lambda_{1\rightarrow\Lambda - 1}^2;\quad L_2 = 1;\quad G_1 = \paren{\lambda_{\Lambda} + R_{\bfu}}\sqrt{\Lambda}\norm{\bfX}_F\lambda_{1\rightarrow\Lambda-1}\paren{\min_{\ell\in[\Lambda-1]}\lambda_{\ell}}^{-1}\\
    G_2 = \paren{\paren{2\lambda_{\Lambda} + R_{\bfu}}\norm{\bfX}_F\lambda_{1\rightarrow\Lambda - 1} + \norm{\bfY}_F}\sqrt{\Lambda}\norm{\bfX}_F\lambda_{1\rightarrow\Lambda - 1}\paren{\min_{\ell\in[\Lambda-1]}\lambda_{\ell}}^{-1}
\end{gather*}
Recall that $1 - \beta = O\paren{\sfrac{1}{\sqrt{\kappa}}} = O\paren{\sfrac{\sqrt{\mu}}{\sqrt{L_1}}}$ and $1 + \beta = O(1)$. Since $L_2 = 1$, we treat it as a constant. Moreover, we also treat $\Lambda$ as a constant. We have shown in Lemma \ref{lem:nn_sat_asump} that $f^\star = 0$. Thus $f(\bfx_0, \bfu_0) - f^\star = \calL(\bm{\theta}(0))$. With $\eta = \frac{c}{L_1}$, the requirement in (\ref{eq:G1_G4_req}) can be simplified to
\begin{equation}
    \label{eq:theo5_req_simp}
    \underbrace{G_1^4 \leq \frac{\hat{C}_1\mu^{\frac{7}{2}}}{L_1^{\frac{3}{2}}}}_{\mathcal{R}_1};\quad \underbrace{G_1^2G_2^2\leq \frac{\hat{C}_2\mu^{\frac{9}{2}}}{L_1^{\frac{3}{2}}}}_{\mathcal{R}_2};\quad \underbrace{R_{\bfx}\geq \frac{\hat{C}_3L_1^{\frac{1}{4}}}{\mu^{\frac{3}{4}}}\calL(\bm{\theta}(0))^{\frac{1}{2}}}_{\mathcal{R}_3};\quad \underbrace{R_{\bfu}\geq \frac{\hat{C}_4G_1L_1^{\frac{3}{4}}}{\mu^{\frac{7}{4}}}\calL(\bm{\theta}(0))^{\frac{1}{2}}}_{\mathcal{R}_4}
\end{equation}
In the following parts of the proof, we will use $\gtrsim$ and $\lesssim$ to denote the inequality hiding constants. We will analyze each requirement separately.

\noindent\textbf{Calculation for $\mathcal{R}_1$.} Notice that 
\[
    G_1^4 \lesssim \paren{\lambda_{\Lambda}^4 + R_{\bfu}^4}\norm{\bfX}_F^4\lambda_{1\rightarrow\Lambda - 1}^4\paren{\min_{\ell\in[\Lambda-1]}\lambda_{\ell}}^{-4} = \paren{\lambda_{\Lambda}^4 + R_{\bfu}^4}\paren{\min_{\ell\in[\Lambda-1]}\lambda_{\ell}}^{-4}L_1^2
\]
It suffices to show that 
\[
    \max\left\{\lambda_{\Lambda}^4,R_{\bfu}^4\right\} \lesssim \min_{\ell\in[\Lambda-1]}\lambda_{\ell}^4\cdot \frac{\mu^\frac{7}{2}}{L_1^\frac{7}{2}}
\]
Notice that, by definition,
\[
    R_{\bfu} \leq \frac{1}{2}\min_{\ell\in[\Lambda - 1]}\lambda_{\ell} \cdot \frac{\mu}{L_1}\Rightarrow R_{\bfu}^4 \leq \paren{\frac{1}{2}}^4\min_{\ell\in[\Lambda - 1]}\lambda_{\ell}^4 \cdot \frac{\mu^4}{L_1^4} \leq \paren{\frac{1}{2}}^4\min_{\ell\in[\Lambda-1]}\lambda_{\ell}^4\cdot \frac{\mu^\frac{7}{2}}{L_1^\frac{7}{2}}
\]
where the last inequality follows from $\mu \leq L_1$. Therefore the condition on $R_{\bfu}$ is satisfied automatically. It suffice to consider the condition on $\lambda_{\Lambda}^4$, which boils down to
\[
    \lambda_{\Lambda}^4 \lesssim \min_{\ell\in[\Lambda-1]}\lambda_{\ell}^4\cdot \frac{\alpha^7}{\norm{\bfX}_F^7\lambda_{1\rightarrow \Lambda-1}^7}
\]
Rearranging gives
\begin{equation}
    \label{eq:r1_simp}
    \alpha^7\gtrsim \frac{\norm{\bfX}_F^7\lambda_{1\rightarrow \Lambda}^7}{\lambda_{\Lambda}^3\min_{\ell\in[\Lambda-1]}\lambda_{\ell}^4}
\end{equation}
\textbf{Calculation for $\mathcal{R}_2$.}
With the condition that $\mathcal{R}_1$ is satisfied, it suffice to show that 
\[
    G_2^4 \lesssim \frac{\mu^{\frac{11}{2}}}{L_1^{\frac{3}{2}}}
\]
For $G_2$, we have
\[
    G_2^4 \lesssim \max\left\{\mathcal{T}_1,\mathcal{T}_2,\mathcal{T}_3\right\}
\]
with
\begin{align*}
    \mathcal{T}_1 & = \norm{\bfX}_F^8\lambda_{\Lambda}^4\lambda_{1\rightarrow\Lambda-1}^8\paren{\min_{\ell\in[\Lambda-1]}\lambda_{\ell}}^{-4}\\
    \mathcal{T}_2 & = R_{\bfu}^4\norm{\bfX}_F^8\lambda_{1\rightarrow\Lambda-1}^8\paren{\min_{\ell\in[\Lambda-1]}\lambda_{\ell}}^{-4}\\
    \mathcal{T}_3 & = \norm{\bfY}_F^4\norm{\bfX}_F^4\lambda_{1\rightarrow\Lambda-1}^4\paren{\min_{\ell\in[\Lambda-1]}\lambda_{\ell}}^{-4}
\end{align*}
Notice that since
\[
    R_{\bfu}^4\leq \paren{\frac{C^\perp}{2}}^4\min_{\ell\in[\Lambda-1]}\lambda_{\ell}^4\cdot\frac{\mu^4}{L_1^4}
\]
we must have
\[
    \mathcal{T}_2 \leq \paren{\frac{C^\perp}{2}}^4\min_{\ell\in[\Lambda-1]}\lambda_{\ell}^4\cdot\frac{\mu^4}{L_1^4}\cdot L_1^4\paren{\min_{\ell\in[\Lambda-1]}\lambda_{\ell}}^{-4} \leq \paren{\frac{C^\perp}{2}}^4\mu^4\lesssim \frac{\mu^{\frac{11}{2}}}{L_1^{\frac{3}{2}}}
\]
since $\mu\leq L_1$. Thus, we only need to consider $\mathcal{T}_1$ and $\mathcal{T}_3$. Combining the two conditions, $\mathcal{R}_2$ boils down to
\begin{equation}
    \label{eq:r2_simp}
    \alpha_0^{11}\gtrsim \frac{\norm{\bfX}_F^7\lambda_{1\rightarrow\Lambda-1}^7}{\min_{\ell\in[\Lambda-1]}\lambda_{\ell}^4}\paren{\norm{\bfX}_F^4\lambda_{1\rightarrow\Lambda}^4 + \norm{\bfY}_F^4}
\end{equation}
\textbf{Calculation for $\mathcal{R}_3$.}
We first notice that since $\mu \leq L_1$, $\mathcal{R}_3$ can be restricted to
\[
    R_{\bfx}\geq \frac{\hat{C}_3L_1^{\frac{1}{2}}}{\mu}\calL(\bm{\theta}(0))^{\frac{1}{2}}
\]
Plugging in $R_{\bfx} = \frac{\lambda_{\Lambda}}{2}$ and $\mu, L_1$ gives
\[
    \lambda_{\Lambda}\gtrsim \frac{\norm{\bfX}_F\lambda_{1\rightarrow\Lambda-1}}{\alpha_0^2}\calL(\bm{\theta}(0))^{\frac{1}{2}}
\]
Rearranging the terms gives
\begin{equation}
    \label{eq:r3_simp}
    \alpha_0^2\gtrsim \frac{\norm{\bfX}_F\lambda_{1\rightarrow L}}{\lambda_{\Lambda}^2}\calL(\bm{\theta}(0))^{\frac{1}{2}}
\end{equation}
\textbf{Calculation for $\mathcal{R}_4$}
Notice that $\alpha_0 = \sqrt{2\mu}\leq \sqrt{2L_1} \leq 2\sqrt{\Lambda}\norm{\bfX}_F\lambda_{1\rightarrow\Lambda-1}$. Therefore
\[
    R_{\bfu} \lesssim \min_{\ell\in[\Lambda-1]}\lambda_{\ell}\frac{\alpha_0^2}{\norm{\bfX}_F^2\lambda_{1\rightarrow\Lambda-1}^2}
\]
To satisfy $\mathcal{R}_4$, we need
\begin{align*}
    R_{\bfu} & \gtrsim R_{\bfu}\norm{\bfX}_F\lambda_{1\rightarrow\Lambda-1}\calL\paren{\bm{\theta}(0)}^{\frac{1}{2}}\frac{L_1^{\frac{3}{4}}}{\mu^{\frac{7}{4}}}\paren{\min_{\ell\in[\Lambda-1]}\lambda_{\ell}}^{-1}\\
    R_{\bfu} & \gtrsim \norm{\bfX}_F\lambda_{1\rightarrow\Lambda}\calL\paren{\bm{\theta}(0)}^{\frac{1}{2}}\frac{L_1^{\frac{3}{4}}}{\mu^{\frac{7}{4}}}\paren{\min_{\ell\in[\Lambda-1]}\lambda_{\ell}}^{-1}
\end{align*}
To analyze the first, we simply remove $R_{\bfu}$ from both sides to get that
\begin{equation}
    \label{eq:r4_1_simp}
    \alpha_0^7\gtrsim \frac{\norm{\bfX}_F^5\lambda_{1\rightarrow\Lambda-1}^5\calL\paren{\bm{\theta}(0)}^2}{\min_{\ell\in[\Lambda-1]}\lambda_{\ell}^4}
\end{equation}
For the second, we plug in the upper bound on $R_{\bfu}$ to have
\begin{equation}
    \label{eq:r4_2_simp}
    \alpha_0^{11}\gtrsim \frac{\norm{\bfX}_F^9\lambda_{1\rightarrow\Lambda}^9}{\lambda_{\Lambda}^7\min_{\ell\in[\Lambda-1]}\lambda_{\ell}^4}\calL\paren{\bm{\theta}(0)}
\end{equation}
\textbf{Initialization Scheme.} Recall our initialization scheme
\begin{gather*}
    d_{\ell} = \Theta\paren{m}\quad \forall\ell\in[\Lambda-1];\quad d_{L-1} = \Theta\paren{n^{4.5}\max{n,d_0^2}}\\
    \left[\bfW_{\ell}(0)\right]_{ij}\sim\mathcal{N}\paren{0,d_{\ell-1}^{-1}}\quad \forall \ell\in[\Lambda-1];\quad \left[\bfW_{\Lambda}(0)\right]_{ij} \sim\mathcal{N}\paren{0,d_{\Lambda-1}^{-\frac{3}{2}}}
\end{gather*}
We will show that this initialization scheme satisfies (\ref{eq:r1_simp})-(\ref{eq:r4_2_simp}), which we restate below
\begin{align*}
    & \alpha^7\gtrsim \frac{\norm{\bfX}_F^7\lambda_{1\rightarrow \Lambda}^7}{\lambda_{\Lambda}^3\min_{\ell\in[\Lambda-1]}\lambda_{\ell}^4}\\
    & \alpha_0^{11}\gtrsim \frac{\norm{\bfX}_F^7\lambda_{1\rightarrow\Lambda-1}^7}{\min_{\ell\in[\Lambda-1]}\lambda_{\ell}^4}\paren{\norm{\bfX}_F^4\lambda_{1\rightarrow\Lambda}^4 + \norm{\bfY}_F^4}\\
    & \alpha_0^2\gtrsim \frac{\norm{\bfX}_F\lambda_{1\rightarrow \Lambda}}{\lambda_{\Lambda}^2}\calL(\bm{\theta}(0))^{\frac{1}{2}}\\
    & \alpha_0^7\gtrsim \frac{\norm{\bfX}_F^5\lambda_{1\rightarrow\Lambda-1}^5\calL\paren{\bm{\theta}(0)}^2}{\min_{\ell\in[\Lambda-1]}\lambda_{\ell}^4}\\
    &  \alpha_0^{11}\gtrsim \frac{\norm{\bfX}_F^9\lambda_{1\rightarrow\Lambda}^9}{\lambda_{\Lambda}^7\min_{\ell\in[\Lambda-1]}\lambda_{\ell}^4}\calL\paren{\bm{\theta}(0)}
\end{align*}
To start, we first compute $\lambda_{\ell}$'s. Recall that we required initializing $\norm{\bfW_{\ell}(0)}_2 = \frac{\lambda_{\ell}}{2}$. This implies that $\lambda_{\ell}\leq 2\norm{\bfW_{\ell}(0)}_2$ for all $\ell\in[\Lambda]$. By Theorem 4.4.5 in \citep{vershynin_2018}, we have
\[
    \norm{\bfW_{\ell}(0)}_2 = \begin{cases}
    O\paren{1 + \frac{\sqrt{m}}{\sqrt{d_0}}} & \text{ if }\ell = 1\\
    O\paren{1} & \text{ if } \ell =2,\dots \Lambda-2\\
    O\paren{1 + \frac{\sqrt{d_{\Lambda-1}}}{\sqrt{m}}} & \text{ if }\ell = \Lambda-1\\
    O\paren{d_{\Lambda-1}^{-\frac{1}{4}} + \frac{\sqrt{d_{\Lambda}}}{d_{\Lambda-1}^{\sfrac{3}{4}}}} & \text{ if }\ell = \Lambda
    \end{cases}
\]
Since $\lambda_{\ell}$ satisfies the same scaling, plugging in the width, and notice that $d_{\Lambda-1}\geq m \geq \max\{d_{\Lambda}, d_0\}$ gives
\[
    \lambda_{\ell} = \begin{cases}
        O\paren{\frac{\sqrt{m}}{\sqrt{d_0}}} & \text{ if } \ell = 1\\
        O\paren{1} & \text{ if }\ell = 2,\dots,\Lambda-2\\
        O\paren{\frac{n^{\sfrac{9}{4}}}{\sqrt{m}}\max\{\sqrt{n},d_0\}} & \text{ if }\ell = \Lambda-1\\
        O\paren{\frac{1}{n^{\sfrac{9}{8}}\max\left\{n^{\sfrac{1}{4}},\sqrt{d_0}\right\}}} & \text{ if }\ell = \Lambda
    \end{cases}
\]
Therefore
\[
    \min_{\ell\in[\Lambda-1]}\lambda_{\ell} = O\paren{1};\quad \lambda_{1\rightarrow\Lambda-1} = O\paren{\frac{n^{\sfrac{9}{4}}}{\sqrt{d_0}}\max\{\sqrt{n},d_0\}}
\]
Moreover, by Assumption 3.1 in \citep{nguyen2021ontheproof}, we have $\norm{\bfX}_F = O\paren{\sqrt{nd_0}}$ and $\norm{\bfY}_F = O\paren{\sqrt{n}}$. By Lemma 3.3 in \citep{nguyen2021ontheproof}, we have that $\alpha_0 = \Omega\paren{d_{\Lambda-1}^{\frac{1}{2}}} = \Omega\paren{n^{\sfrac{9}{4}}\max\{\sqrt{n},d_0}\}$. Lastly, by Lemma C.1 in \citep{nguyen2020global} and \citep{nguyen2021ontheproof}, we have $\calL(\bm{\theta}(0))^{\frac{1}{2}} = O\paren{\sqrt{nd_0}}$. With these preparations, let's check each requirement. For (\ref{eq:r1_simp}), we have
\[
    \alpha_0^7 = \Omega\paren{\max\left\{n^{\sfrac{77}{4}},n^{\sfrac{63}{4}}d_0^7\right\}};\quad \frac{\norm{\bfX}_F^7\lambda_{1\rightarrow\Lambda}^7}{\lambda_{\Lambda}^3\min_{\ell\in[\Lambda-1]}\lambda_{\ell}^4} = O\paren{\max\left\{n^{\sfrac{69}{4}},n^{\sfrac{59}{4}}d_0^5\right\}}
\]
Therefore, we have that (\ref{eq:r1_simp}) is satisfied. For (\ref{eq:r2_simp}), we have
\begin{gather*}
    \alpha_0^{11} = \Omega\paren{\max\left\{n^{\sfrac{121}{4}},n^{\sfrac{99}{4}}d_0^9\right\}}\\
    \frac{\norm{\bfX}_F^7\lambda_{1\rightarrow\Lambda-1}^7}{\min_{\ell\in[\Lambda-1]}\lambda_{\ell}^4}\paren{\norm{\bfX}_F^4\lambda_{1\rightarrow\Lambda}^4 + \norm{\bfY}_F^4} = O\paren{\max\left\{n^{\sfrac{165}{8}},n^{\sfrac{143}{8}}d_0^{\sfrac{11}{2}}\right\}}
\end{gather*}
Therefore, we have that (\ref{eq:r2_simp}) is satisfied. For (\ref{eq:r3_simp}), we have
\[
    \alpha_0^2 = \Omega\paren{\max\left\{n^\frac{11}{2},n^{\frac{9}{2}}d_0^2\right\}};\quad \frac{\norm{\bfX}_F\lambda_{1\rightarrow \Lambda}}{\lambda_{\Lambda}^2}\calL(\bm{\theta}(0))^{\frac{1}{2}} = O\paren{\max\left\{n^\frac{33}{8},n^{\frac{31}{8}}d_0\right\}}
\]
Therefore, we have that (\ref{eq:r3_simp}) is satisfied. For (\ref{eq:r4_1_simp}), we have
\[
    \alpha_0^7 = \Omega\paren{\max\left\{n^{\sfrac{77}{4}},n^{\sfrac{63}{4}}d_0^7\right\}};\quad \frac{\norm{\bfX}_F^5\lambda_{1\rightarrow\Lambda-1}^5\calL\paren{\bm{\theta}(0)}^2}{\min_{\ell\in[\Lambda-1]}\lambda_{\ell}^4} = O\paren{\max\left\{n^{\sfrac{73}{4}}d_0^2,n^{\sfrac{63}{4}}d_0^{7}\right\}}
\]
Notice that $n^{\sfrac{73}{4}}d_0^2\geq n^{\sfrac{63}{4}}d_0^{7}$ only when $d_0\leq \sqrt{n}$. In this case, we must have that $n^{\frac{77}{4}}\geq n^{\sfrac{73}{4}}d_0^2$. Therefore, we have that (\ref{eq:r4_1_simp}) is satisfied. For (\ref{eq:r4_2_simp}), we have
\[
    \alpha_0^{11} = \Omega\paren{\max\left\{n^{\sfrac{121}{4}},n^{\sfrac{99}{4}}d_0^9\right\}};\quad \frac{\norm{\bfX}_F^9\lambda_{1\rightarrow\Lambda}^9}{\lambda_{\Lambda}^7\min_{\ell\in[\Lambda-1]}\lambda_{\ell}^4}\calL\paren{\bm{\theta}(0)} = O\paren{\max\left\{n^{\sfrac{51}{4}}d_0,n^{10}d_0^{\sfrac{13}{2}}\right\}}
\]
Similarly, when $n^{\sfrac{51}{4}}d_0\geq n^{10}d_0^{\sfrac{13}{2}}$, we must have $d_0\leq \sqrt{n}$. This implies that $n^{\frac{121}{4}}\geq n^{\sfrac{51}{4}}d_0$. Therefore, we have that (\ref{eq:r4_2_simp}) is also satisfied. Now, all requirements in Theorem \ref{theo:nesterov_conv} can be satisfied by the initialization scheme with our over-parameterization. Thus, we can invoke Theorem \ref{theo:nesterov_conv} to get that
\[
    f(\bfx_k,\bfu_k) - f^\star \leq 2\paren{1 - \frac{c}{4\sqrt{\kappa}}}(f(\bfx_0,\bfu_0) - f^\star)
\]
Noting that $f^\star = 0$ and $f(\bfx_k,\bfu_k) = \calL(\bm{\theta}(k))$, we have
\[
    \calL(\bm{\theta}(k)) \leq 2\paren{1 - \frac{c}{4\sqrt{\kappa}}}\calL(\bm{\theta}(0))
\]
which completes the proof.
\section{Auxiliary Lemma}
\begin{lemma}
    \label{lem:smooth_grad_bound}
    Suppose that Assumption \ref{asump:f_smooth} holds. Then for all $\bfx\in\R^{d_1}$ and $\bfu\in\mathcal{B}^{(2)}_{R_{\bfu}}$ we have
    \[
        \norm{\nabla_1f(\bfx,\bfu)}_2^2 \leq 2L_1\paren{f(\bfx,\bfu) - f^\star}
    \]
\end{lemma}
\begin{proof}
Assumption \ref{asump:f_smooth} implies that, for all $\bfx\in\R^{d_1}$ and $\bfu\in\mathcal{B}^{(2)}_{R_{\bfu}}$
\begin{align*}
    f^\star \leq f\paren{\bfx - \frac{1}{L_1}\nabla_1f(\bfx,\bfu),\bfu} & \leq f(\bfx,\bfu) - \frac{1}{L_1}\norm{\nabla_1f(\bfx,\bfu)}_2^2 + \frac{1}{2L_1}\norm{\nabla_1f(\bfx,\bfu)}_2^2\\
    & = f(\bfx,\bfu) - \frac{1}{2L_1}\norm{\nabla_1f(\bfx,\bfu)}_2^2
\end{align*}
which implies that, for all $\bfx\in\R^{d_1}$ and $\bfu\in\mathcal{B}^{(2)}_{R_{\bfu}}$
\[
    \norm{\nabla_1f(\bfx,\bfu)}_2^2 \leq 2L_1\paren{f(\bfx,\bfu) - f^\star}
\]
\end{proof}

\begin{lemma}
    \label{lem:g_grad_bound}
    Suppose that Assumption \ref{asump:g_smooth} holds. Then for all $\bfs\in\R^{\hat{d}}$ we have
    \begin{align*}
        \norm{\nabla g(\bfs)}_2^2 \leq 2L_2\paren{g(\bfs) - f^\star}
    \end{align*}
\end{lemma}
\begin{proof}
    By Assumption \ref{asump:g_smooth}, for all $\bfs\in\R^{\hat{d}}$ we have
    \[
        f^\star = g^{\star} \leq g\paren{\bfs - \frac{1}{L_2}\nabla g(\bfs)} = g(\bfs) - \frac{1}{L_2}\norm{\nabla g(\bfs)}_2^2 + \frac{1}{2L_2}\norm{\nabla g(\bfs)}_2^2 = g(\bfs) - \frac{1}{2L_2}\norm{\nabla g(\bfs)}_2^2
    \]
    Therefore, for all $\bfs\in\R^{\hat{d}}$, it holds that
    \begin{align*}
        \norm{\nabla g(\bfs)}_2^2 \leq 2L_2\paren{g(\bfs) - f^\star}
    \end{align*}
\end{proof}

\begin{lemma}
    \label{lem:du_grad_bound}
    Suppose that Assumption \ref{asump:g_smooth}, \ref{asump:h_lip} holds. Then for all $\bfx\in\mathcal{B}^{(1)}_{R_{\bfx}}$ and $\bfu\in\mathcal{B}^{(2)}_{R_{\bfu}}$ we have
    \begin{align*}
        \norm{\nabla_2f(\bfx,\bfu)}_2^2  \leq 2G_1L_2\paren{f(\bfx,\bfu)-f^\star}
    \end{align*}
\end{lemma}
\begin{proof}
    Since Assumption \ref{asump:g_smooth} holds, we can invoke Lemma \ref{lem:g_grad_bound} to get that for all $\bfx\in\mathcal{B}^{(1)}_{R_{\bfx}}$ and $\bfu\in\mathcal{B}^{(2)}_{R_{\bfu}}$, we have
    \[
        \norm{\nabla g(h(\bfx,\bfu))}_2^2 \leq 2L_2\paren{f(\bfx,\bfu) - f^\star}
    \]
    By Assumption \ref{asump:h_lip}, we must have that $\norm{\nabla_2h(\bfx,\bfu)}_2 \leq G_1$. Therefore, using the chain rule, we have
    \[
        \norm{\nabla_2f(\bfx,\bfu)}_2^2 \leq \norm{\nabla_2h(\bfx,\bfu)}_2\norm{\nabla g(h(\bfx,\bfu))}_2^2 \leq 2G_1^2L_2\paren{f(\bfx,\bfu) - f^\star}
    \]
\end{proof}

\begin{lemma}
    \label{lem:phi0_upper_bound}
    Let $\phi_0$ be defined in (\ref{eq:lyapnov_exact}). Suppose that Assumption holds. Then we have that $\phi_0 \leq 2(f(\bfx_0,\bfu_0) - f^\star)$.
\end{lemma}
\begin{proof}
    To start, for all $c \leq 1$, we must have that
    \[
        \gamma = \frac{c}{2\sqrt{\kappa} - c} \leq \frac{c}{\sqrt{\kappa}} \leq 1
    \]
    Thus, for $\lambda$, we have $\lambda = (1 + \gamma)^3 - 1 \leq 7\gamma$. This implies that
    \[
        \mathcal{Q}_1 = \frac{\lambda^2}{2\eta(1 + \gamma)^5} \leq \frac{25\gamma^2}{\eta} = \frac{25cL_1}{\kappa} = 25c\mu
    \]
    When $k=0$, we have $\bfz_0 = \bfy_0 = \bfx_0$. Moreover, $\bfx_{-1}^{\star} = \argmin_{\bfx\in\R^{d_1}}f(\bfx,\bfu_0)$. At $\bfu_0$, Assumption \ref{asump:strong_cvx} must hold, which implies that
    \[
        f(\bfx_0,\bfu_0) \geq f(\bfx_{-1}^\star,\bfu_0) + \frac{\mu}{2}\norm{\bfx_0 - \bfx_{-1}^{\star}}_2^2
    \]
    This implies that $\norm{\bfx_0 - \bfx_{-1}^{\star}}_2^2 \leq \frac{2}{\mu}\paren{f(\bfx_0,\bfu_0) - f^\star}$. Thus
    \[
        \mathcal{Q}_1\norm{\bfz_0 - \bfx_{-1}^\star}_2^2\leq 50c\paren{f(\bfx_0,\bfu_0) - f^\star}
    \]
    Moreover, by Assumption \ref{asump:f_smooth}, we have
    \[
        \frac{\eta}{8}\norm{\nabla_1f(\bfy_{-1},\bfv_{-1}}_2 = \frac{\eta}{8}\norm{\nabla_1f(\bfx_0,\bfu_0)}_2^2 \leq \frac{\eta L_1}{4}\paren{f(\bfx_0,\bfu_0) - f^\star} = \frac{c}{4}\paren{f(\bfx_0,\bfu_0) - f^\star}
    \]
    Thus, putting things together, we have
    \[
        \phi_0 \leq \paren{1 + 50.25c}\paren{f(\bfx_0,\bfu_0) - f^\star} \leq 2\paren{f(\bfx_0,\bfu_0) - f^\star}
    \]
    as long as $c\leq \frac{1}{51}$.
\end{proof}

\end{document}